\documentclass{article}

\usepackage{arxiv}

\usepackage[utf8]{inputenc} 
\usepackage[T1]{fontenc}    
\usepackage{hyperref}       
\usepackage{url}            
\usepackage{booktabs}       
\usepackage{amsfonts}       
\usepackage{nicefrac}       
\usepackage{microtype}      
\usepackage{lipsum}
\usepackage{graphicx}
\usepackage{natbib}
\usepackage{amsmath}
\graphicspath{ {./images/} }

\usepackage{listings}
\usepackage{amsmath}
\usepackage{amssymb}
\usepackage{physics}
\usepackage{xcolor}
\usepackage{geometry}
\usepackage{hyperref}
\usepackage{algorithm}
\usepackage{algpseudocode}
\usepackage{booktabs}
\usepackage{graphicx}
\usepackage{array}
\usepackage{multirow}
\usepackage{amsthm}

\usepackage{listings}
\usepackage{amsmath}
\usepackage{amssymb}
\usepackage{physics}
\usepackage{xcolor}
\usepackage{geometry}
\usepackage{hyperref}
\usepackage{algorithm}
\usepackage{algpseudocode}
\usepackage{booktabs}
\usepackage{graphicx}
\usepackage{array}
\usepackage{multirow}
\usepackage{colortbl}
\usepackage{tabularx}
\usepackage{float}
\usepackage{placeins}
\usepackage{caption}
\usepackage{tikz}
\usepackage{mathtools}
\usetikzlibrary{arrows.meta, positioning, calc, decorations.pathreplacing, backgrounds, fit, shapes.geometric}

\definecolor{figTeal}{HTML}{F3F1EE}
\definecolor{figTealBorder}{HTML}{8C857B}
\definecolor{figBlue}{HTML}{ECEAE6}
\definecolor{figBlueBorder}{HTML}{6B6560}
\definecolor{figSlate}{HTML}{3E3A36}
\definecolor{figBar}{HTML}{D5CFC8}
\definecolor{figBarBorder}{HTML}{9B9488}
\definecolor{figAccent}{HTML}{5C5650}
\definecolor{figDivider}{HTML}{B5AFA8}
\definecolor{figState}{HTML}{F7F5F2}
\definecolor{figStateBorder}{HTML}{A09890}
\definecolor{figRowAlt}{HTML}{F5F3F0}

\usepackage{amsthm}
\newtheorem{theorem}{Theorem}[section]
\newtheorem{lemma}[theorem]{Lemma}
\newtheorem{proposition}[theorem]{Proposition}
\newtheorem{definition}[theorem]{Definition}

\theoremstyle{remark}
\newtheorem{remark}[theorem]{Remark}

\title{Deep Sequence Modeling with Quantum Dynamics:\\
Language as a Wave Function}
\author{
  Ahmed Nebli \\
  \texttt{ahmed.nebli@cai-technology.ai} \\
  cAI Technology GmbH
  \and
  \textbf{Hadi Saadatdoorabi} \\
  \texttt{hadi.saadatdoorabi@cai-technology.ai} \\
  cAI Technology GmbH
    \and
  \textbf{Kevin Yam} \thanks{Corresponding author}  \\
  \texttt{kevin.yam@coeo-group.ai} \\
  cAI Technology GmbH 
}\date{}

\begin{document}
\maketitle

\begin{abstract}
We introduce a sequence modeling framework in which the latent state is a
complex-valued wave function evolving on a finite-dimensional Hilbert space
under a learned, time-dependent Hamiltonian.  Unlike standard recurrent
architectures that rely on gating mechanisms to suppress competing
hypotheses, our framework utilizes quantum interference: the Hamiltonian
steers the phases of complex amplitudes so that conflicting interpretations
cancel while compatible ones reinforce.  The dynamics are strictly unitary,
ensuring that the state norm is preserved exactly at every time step via a
Cayley (Crank--Nicolson) discretization.  Token probabilities are extracted
using the Born rule, a quadratic measurement operator that couples
magnitudes and relative phases.  Our primary theoretical contribution is a
separation theorem characterizing the representational advantage of this
readout: we define a family of disambiguation tasks that a complex unitary
model of dimension $N$ solves exactly, but which requires a state dimension
of $\Omega(N^2)$ for any real-valued orthogonal model equipped with a
standard affine-softmax readout.  This quadratic gap arises because the
Born rule implicitly lifts the $N$-dimensional state into the space of
rank-one Hermitian matrices, accessing pairwise phase correlations that are
inaccessible to linear projections.  Finally, we derive a continuity
equation for the latent probability mass, yielding conserved pairwise
currents that serve as a built-in diagnostic for tracing information flow
between dimensions.
\end{abstract}

\newpage

\section{Introduction}
\label{sec:introduction}

Sequence modeling, the task of predicting the next element given a
preceding context, underlies modern language models and drives applications
from text generation to code synthesis.  The dominant architectures for
this task, including Transformers~\cite{vaswani2017attention}, recurrent
neural networks~\cite{hochreiter1997long}, and state-space
models~\cite{gu2023mamba}, differ substantially in their computational
mechanisms but share a fundamental representational choice: the latent
state at each time step is a vector of real numbers.  In Transformers,
context is stored in a cache of real-valued key-value pairs; in recurrent
networks, it is a hidden vector updated multiplicatively; in state-space
models, it is a vector governed by a discretized linear dynamical system.
This paper examines the algebraic implications of this choice.  In a real
vector space, the superposition of two vectors is strictly additive.  While
real-valued neural networks can suppress incorrect hypotheses using learned
nonlinear gating mechanisms (such as the sigmoid gates in LSTMs or the
projection layers in Transformers), they lack the intrinsic geometric
property of phase.  In a complex vector space, superposition allows for
interference: depending on the relative phase, adding two amplitudes can
result in constructive reinforcement or destructive cancellation.

Complex-valued representations provide a distinct mechanism for hypothesis
interaction.  A complex number carries both a magnitude and a phase.  This
structure allows competing interpretations of an ambiguous context to
interact algebraically rather than requiring a dedicated gating module to
suppress one against the other.  Consider the prefix ``The bank was
\ldots'' where a model must maintain latent probability mass for both the
`financial institution' and `river edge' interpretations.  If the
subsequent token is ``steep,'' it acts as a filter that should reinforce
the river interpretation and suppress the financial one.  In the framework
proposed here, this suppression arises from interference: the evolution of
the state rotates the phases of the latent components such that the
amplitudes associated with the financial interpretation interfere
destructively with the new input, while the river interpretation interferes
constructively.  This architectural choice is inspired by, though distinct
from, the field of quantum cognition.  Busemeyer and
Bruza~\cite{busemeyer2012quantum} and Pothos and
Busemeyer~\cite{pothos2013quantum} have documented that human probability
judgments often violate classical probability axioms in ways formally
described by quantum probability theory.  While our goal is to model the
statistical distribution of text rather than human cognitive processes,
these findings suggest that the mathematical structure of complex Hilbert
spaces may offer a parsimonious inductive bias for sequence disambiguation.

There is also growing interest in formulating neural network dynamics in
continuous time.  Chen et al.~\cite{chen2018neural} formalized residual
networks as discretizations of ordinary differential equations (Neural
ODEs).  Greydanus et al.~\cite{greydanus2019hamiltonian} introduced
Hamiltonian Neural Networks to enforce conservation laws in physical
simulations. \cite{haber2017stable} demonstrated that
constraining dynamics to be norm-preserving (orthogonal or unitary)
provides strong stability guarantees, mitigating the vanishing and
exploding gradient problems.  Modern state-space
models~\cite{gu2023mamba,gu2021efficiently} leverage continuous-time linear
systems for efficient sequence processing.  However, existing approaches
typically operate in real-valued spaces or rely on unconstrained dynamics
that may lack unitary stability.  None combine complex-valued states with
structured Hamiltonian evolution and a quadratic, measurement-based output
rule.

This paper introduces a framework integrating these concepts.  The latent
state is a unit-norm vector in a finite-dimensional complex Hilbert space.
Its evolution follows the time-dependent Schr\"{o}dinger equation
$i\frac{d}{dt}\ket{\psi(t)} = H(t)\ket{\psi(t)}$, where $H(t)$ is a
Hermitian operator.  The Hermiticity of $H(t)$ guarantees that the
time-evolution operator is strictly unitary, ensuring the state vector
rotates on the unit sphere without changing its norm.  We decompose
$H(t) = H_0 + H_{\mathrm{int}}(t)$, separating the dynamics into a
learnable diagonal term $H_0$ that sets baseline oscillation frequencies,
and an input-dependent interaction Hamiltonian $H_{\mathrm{int}}(t)$
generated by a neural network.  This decomposition allows us to move into
the interaction picture, a coordinate change that factors out the free
oscillations, leaving the numerical integrator to resolve only the
input-driven dynamics.

To implement this system, we discretize the evolution using the Cayley
transform, which corresponds to the Crank--Nicolson (implicit midpoint)
scheme of numerical analysis~\cite{hairer2006geometric,crank1947practical}.
Unlike standard explicit integrators (e.g., Runge-Kutta), which destroy the
unitary property and introduce norm drift, the Cayley transform produces an
exactly unitary update for any step size.  This ensures the model preserves
the state norm regardless of sequence length, although we note that
discretization error in the \emph{phase} trajectory (integration accuracy)
remains bounded by the step size.  Token prediction is performed via the
Born rule: the probability of token $k$ is the squared magnitude of the
inner product between the state and a learned measurement vector.  This
introduces a quadratic nonlinearity at the readout layer, differing
fundamentally from the linear-projection-plus-softmax used in standard
models.  As we show, this quadratic structure allows the model to access
pairwise phase correlations that are inaccessible to a linear readout.

The central theoretical contribution of this paper is a separation theorem
characterizing the representational capacity of this architecture.  We
construct a family of disambiguation tasks in which the correct output
depends on the phase relationship between two context tokens.  We prove
that a complex-valued unitary model of dimension $N$ can represent these
tasks exactly, whereas any real-valued model restricted to orthogonal
dynamics and an affine-softmax readout requires a state dimension of
$\Omega(N^2)$.  The gap arises because the Born rule implicitly lifts the
state into the space of Hermitian matrices, accessing $O(N^2)$ degrees of
freedom (including phase cross-terms) from an $N$-dimensional complex
vector.  Real-valued models with linear readouts must explicitly encode
these pairwise interactions in their state dimensions.
The lower bound (Theorem~\ref{thm:separation}) is proved conditionally,
assuming the target log-probability matrix has full row rank (a
genericity-type condition on the task parameters, discussed in
Section~\ref{sec:separation}), and is established for the state-independent
model class defined in Section~\ref{sec:separation}; the extension to the
full state-dependent architecture of Section~\ref{sec:model} is left for
future work.
Finally, we derive a continuity equation for the probability mass in the
latent dimensions.  Because the dynamics are unitary, the change in
occupation probability of any dimension is exactly accounted for by
antisymmetric probability currents flowing between dimensions.  We propose
these currents as a built-in diagnostic tool for tracing the internal
redistribution of semantic content.

\begin{table}[h!]
\centering
\renewcommand{\arraystretch}{1.15}
\scriptsize
\begin{tabularx}{\textwidth}{@{}lXl@{}}
\toprule
\textbf{Quantum Formalism} & \textbf{Architectural Role} & \textbf{ML Interpretation} \\
\midrule
Hilbert space $\mathbb{C}^N$ & Complex latent space with magnitude and phase per coordinate & Hidden state space \\[2pt]
Wave function $\ket{\psi(t)}$ & Unit-norm complex vector; superposition of interpretations & Recurrent hidden state \\[2pt]
Hermitian $H(t){=}H^\dagger$ & Constrains dynamics to be norm-preserving by construction & Learned transition rule \\[2pt]
Schr\"odinger equation & $H^\dagger{=}H \Rightarrow$ unitary $\Rightarrow$ $\|\psi\|{\equiv}1$ & Recurrence equation \\[2pt]
Cayley transform & Structure-preserving discretization; unitarity exact for any step size & Discrete state update \\[2pt]
Born rule $|\!\braket{m_k}{\psi}\!|^2$ & Quadratic readout: $N^2$ features from $N$ complex dims & Token probabilities \\[2pt]
Interference (phase) & Competing interpretations cancel or reinforce via relative phase & Disambiguation mechanism \\[2pt]
Prob.\ current $J_{j\leftarrow k}$ & Antisymmetric, conserved pairwise flow driven by each token & Built-in interpretability \\
\bottomrule
\end{tabularx}
\caption{Correspondence between quantum-mechanical formalism and the sequence model.  Each structure plays a specific functional role: Hermiticity guarantees unitarity, unitarity guarantees norm conservation, and norm conservation ensures the Born rule produces valid probabilities.  The architecture uses this algebraic chain as an engineering constraint, not as a physical simulation.}
\label{tab:concept_bridge}
\end{table}

\section{Related Work}
\label{sec:related}

Our framework combines three ideas: (i) complex-valued latent states, (ii)
Hamiltonian continuous-time dynamics, and (iii) measurement-based decoding,
into a single architecture for next-token prediction.  Each of these ideas
has antecedents in the literature, but no prior work combines all three,
and several structural limitations of existing approaches motivate the
specific design choices we make.  We organize this section around the seven
research programs most directly relevant to our proposal.

\subsection{Unitary and Orthogonal Recurrent Networks}

The observation that recurrent neural networks suffer from vanishing and
exploding gradients during backpropagation through time motivated a line of
work that constrains the hidden-to-hidden transition matrix to be
orthogonal or unitary.  Arjovsky et al.~\cite{arjovsky2016unitary}
introduced the Unitary Evolution RNN (uRNN), parameterizing the transition
matrix as a product of diagonal phase matrices, Householder reflections,
discrete Fourier transforms, and permutations.  This composition is always
unitary, which guarantees that the gradient norm neither grows nor shrinks
across time steps.  The authors demonstrated the benefit on synthetic
benchmarks requiring long-term memory, such as the \textit{Copying Task}
and the \textit{Adding Problem}, where the uRNN maintained stable gradients
over hundreds of time steps where standard LSTMs
degraded~\cite{hochreiter1997long}.

However, the fixed factorization used by Arjovsky et al.\ covers only a
strict subset of the unitary group $U(N)$.  Wisdom et
al.~\cite{wisdom2016full} addressed this limitation by parameterizing the
full unitary group via a gradient-based optimization on the Stiefel
manifold~\cite{absil2008optimization}, using the Cayley transform to map
skew-Hermitian matrices to unitary matrices.  Their Full-Capacity Unitary
RNN achieved lower error on the \textit{Adding Problem} and the
\textit{Copying Task} at sequence lengths up to $T = 2000$.  Mhammedi et
al.~\cite{mhammedi2017efficient} proposed parameterizing orthogonal
matrices via a product of Householder reflections, obtaining a strictly
expressive family with full coverage of $O(N)$ at cost $O(N^2)$ per step
while avoiding the manifold retraction required by gradient-based Stiefel
methods. \cite{lezcano2019cheap}
subsequently offered a unified extrinsic parameterization of orthogonal and
unitary matrices through the matrix exponential of skew-symmetric or
skew-Hermitian matrices, equipped with an efficient Cayley-based retraction.

Vorontsov et al.~\cite{vorontsov2017orthogonality} further studied
orthogonality constraints in recurrent networks and their effect on
gradient propagation.  Our framework shares the use of unitary evolution
and the Cayley transform with these models, but differs in three structural
respects that are consequential for language modeling.  First, both uRNNs
and Full-Capacity Unitary RNNs typically define the recurrence as a
discrete algebraic update $h_{t+1} = W h_t + V x_t$, where $W$ is a
unitary matrix and $V x_t$ is an additive input injection.  In this
formulation, the input token shifts the state but does not alter the
dynamics that govern state evolution.  In our framework, the input token
at time $t$ parameterizes the interaction Hamiltonian
$H_{\mathrm{int}}(t)$, which is the generator of the evolution itself.
This means the input controls the \emph{rotation axis and angular velocity}
of the state trajectory on the unit sphere, allowing the input to
restructure the interference pattern among active interpretations rather
than merely displacing the state vector.

Second, prior unitary RNNs typically decode the hidden state through a
standard real-valued affine projection followed by softmax.  This discards
the phase information in the complex hidden state, collapsing it to
magnitudes before the output layer.  Our Born-rule decoding computes $p(k)
= |\langle m_k | \psi(t) \rangle|^2$, which is a quadratic function of
the complex amplitudes.  As we discuss in Section~\ref{sec:born_rule}, this
quadratic structure allows the output probabilities to depend on cross-term
interference between components of $\ket{\psi(t)}$, a mechanism that a
linear projection followed by softmax cannot replicate without increasing
the latent dimension.

Third, while prior work focuses on gradient stability, it does not
establish a formal representational separation between complex unitary and
real orthogonal models.  Our separation theorem
(Section~\ref{sec:separation}) constructs a family of disambiguation tasks
and proves that a complex unitary model of dimension $N$ computes these
tasks exactly while any real orthogonal model \emph{with an affine-softmax
readout} requires dimension $\Omega(N^2)$.  This identifies a specific
readout-dependent bottleneck in standard real-valued architectures.

\subsection{State-Space Models}

Structured State Space Models (SSMs) form the continuous time framework
closest to our own.  The S4 model introduced by Gu et
al.~\cite{gu2021efficiently} defines the latent dynamics as a linear
system $x'(t) = Ax(t) + Bu(t)$, $y(t) = Cx(t) + Du(t)$, where $A$ is
initialized with the HiPPO matrix, a specific structure designed to
optimally compress continuous signals.  S4 discretizes this system and
computes the resulting convolution via FFT, achieving strong results on the
\textit{Long Range Arena} benchmark.

Mamba~\cite{gu2023mamba} extended SSMs by making the matrices $B$, $C$,
and the discretization step size $\Delta$ functions of the input token,
introducing input-dependent selection.  This broke the time-invariance that
enabled S4's convolutional mode but allowed the model to selectively
propagate or forget information, achieving language modeling perplexity
competitive with Transformers.  Smith et al.~\cite{smith2023simplified}
subsequently simplified and extended the S4 lineage with the S5 model,
showing further efficiency gains through parallel scans.  Poli et
al.~\cite{poli2023hyena} introduced the Hyena operator, a long-convolution
approach that replaces attention with implicit parameterized convolutions,
achieving competitive performance on language modeling while scaling
sub-quadratically with sequence length; the analysis of Hyena underscores
that the choice of state-update operator---linear convolution versus
Hamiltonian flow---has direct consequences for the model's expressive power
and inductive bias.

Our framework relates to SSMs in the following precise ways.  When the
interaction Hamiltonian is set to zero ($H_{\mathrm{int}}(t) = 0$) and the
free Hamiltonian $H_0$ is diagonal, the Schr\"{o}dinger equation
$i\frac{d}{dt}\ket{\psi} = H_0\ket{\psi}$ reduces to $N$ decoupled
oscillators $\frac{d}{dt}c_j = -i\omega_j c_j$.  This is a diagonal
linear system identical in structure to the diagonalized form of S4,
except that our state variables are natively complex.  In our framework,
the complex structure is semantically load-bearing: phases encode relational
information between interpretations, and the Born rule converts phase
relationships into output probabilities through interference.

Critically, the linearity of SSMs imposes a constraint: the state update
at each step is a linear map applied to the current state.  Our Hamiltonian
framework introduces nonlinearity through the interaction term
$H_{\mathrm{int}}(t)$: because $H_{\mathrm{int}}(t)$ is generated by a
neural network applied to the token embedding and the current state (as
detailed in Section~\ref{sec:hamiltonian}), the map from input sequence to
state trajectory is nonlinear, even though the instantaneous evolution
under any fixed $H(t)$ is linear.

\subsection{Complex-Valued Neural Networks}

Trabelsi et al.~\cite{trabelsi2018deep} introduced a systematic framework
for deep complex-valued networks, defining complex convolutions and batch
normalization.  They demonstrated improvements over real-valued baselines
on tasks like \textit{MusicNet}, attributing the gains to the ability of
complex convolutions to learn phase-sensitive filters.  Hirose and
Yoshida~\cite{hirose2012generalization} provide an earlier theoretical
treatment of complex-valued neural networks and their generalization
properties.  Virtue et al.~\cite{virtue2017better} showed that
complex-valued networks can outperform their real-valued counterparts in
MRI fingerprinting tasks, where the underlying physics produces
complex-valued measurement signals and phase structure carries diagnostic
information; this domain-specific result provides additional evidence that
complex representations are most beneficial when the problem's native
structure is phase-bearing.  However, the empirical literature is not
uniformly positive; other works have found that real-valued models can
match complex-valued performance if the parameter count is equalized by
increasing the real network's width.  This suggests that complex values
are not a panacea, but rather a specific inductive bias suitable for
specific problems.

Our separation theorem provides a formalization of \emph{when} this bias
is useful in sequence modeling.  It identifies a specific
mechanism---destructive interference between complex amplitudes accessed via
a quadratic readout---and proves that it enables a dimension reduction by a
quadratic factor relative to real-valued models with linear readouts.  The
proof constructs disambiguation functions in which the correct output
depends on the relative phase between two context tokens.  A complex
unitary model encodes the interaction in the relative phases of the state
vector, while a real orthogonal model must represent pairwise interactions
as separate state dimensions.

\subsection{Neural ODEs and Hamiltonian Networks}

Chen et al.~\cite{chen2018neural} showed that residual networks can be
interpreted as discretizations of an ODE $\frac{dh}{dt} = f(h, \theta,
t)$.  While powerful, unconstrained Neural ODEs face topological
limitations; Dupont et al.~\cite{dupont2019augmented} showed that Neural
ODEs in $\mathbb{R}^n$ cannot represent functions that require trajectories
to cross.  They resolved this by augmenting the state space.  Kidger et
al.~\cite{kidger2021neural} further developed neural controlled
differential equations for sequence modeling, demonstrating that
continuous-time recurrences can match or exceed discrete recurrent models
on irregular time-series.  Our framework resolves trajectory crossings
through phase: two states with identical magnitudes $|c_j|$ but different
phases occupy different points on the complex unit sphere.  Phase provides
an extra degree of freedom per complex coordinate, and the Hamiltonian
dynamics naturally exploit this by rotating phases at rates determined by
the learned frequencies.

Greydanus et al.~\cite{greydanus2019hamiltonian} introduced Hamiltonian
Neural Networks (HNNs) to enforce energy conservation in physical
simulations.  However, HNNs operate in real-valued phase spaces and
conserve a scalar energy function.  In contrast, our framework conserves
the $L^2$ norm of a complex state vector.  This conservation of total
probability yields vector-valued probability currents (derived in
Section~\ref{sec:currents}) rather than a scalar energy value, serving as
a diagnostic for information flow rather than physical energy.  Stoudenmire
and Schwab~\cite{stoudenmire2016supervised} demonstrated that tensor
networks, which share the unitary-constraint philosophy of our approach,
can be competitive classifiers, further motivating norm-preserving
parameterizations in machine learning.

\subsection{Liquid Neural Networks and Continuous-Time Recurrences}

Liquid Neural Networks (LNNs), introduced by Hasani et
al.~\cite{hasani2021liquid}, define recurrent dynamics through a system of
ordinary differential equations whose time constants are modulated jointly
by the input signal and the current state.  The core governing equation
for the hidden state $h(t) \in \mathbb{R}^d$ is
\begin{equation}
\dot{h}(t) = -\frac{h(t)}{\tau(h(t), x(t), \theta)} + f_\theta(h(t),
x(t)),
\end{equation}
where $\tau > 0$ is a state- and input-dependent time constant determining
how quickly each neuron decays toward the input-driven target, and
$f_\theta$ is a learned nonlinearity.  The resulting dynamics are
``liquid'' in the sense that each neuron's effective temporal receptive
field adapts to the current input, compressing or dilating its timescale
in response to input statistics.  Lechner et al.~\cite{lechner2020neural}
demonstrated that wiring topologies inspired by the nematode \textit{C.
elegans} connectome produce remarkably compact and interpretable
controllers for continuous control tasks, achieving state-of-the-art
performance with far fewer parameters than generic Neural ODEs.  The
Closed-form Continuous-time (CfC) variant~\cite{hasani2022closed} derives
an approximate closed-form solution to these ODEs, eliminating the need
for an online numerical integrator during inference and recovering the
computational efficiency of discrete-step recurrence while retaining the
expressive inductive bias of continuous-time dynamics.

LNNs share with our framework the fundamental architectural choice of
grounding the state update in a differential equation rather than in a
hand-designed recurrence rule, and both frameworks allow the input signal
to modulate the transition, rather than contributing only an additive
displacement.  These structural similarities make LNNs the most direct
continuous-time comparator to the architecture proposed here.

Several differences distinguish the two frameworks, however, and they are
consequential for the representational and stability properties relevant to
language modeling.  First, LNN dynamics operate in real-valued state spaces
with unconstrained weight matrices; no architectural invariant controls the
spectral properties of the Jacobian of $f_\theta$, and stability is ensured
through the sign of the decay term $-h/\tau$ rather than through an exact
structural constraint.  Our framework uses a Hermitian Hamiltonian, which
structurally guarantees that the continuous-time flow is norm-preserving
(unitary) and that, via the Crank--Nicolson (Cayley) discretization
(Section~\ref{sec:cayley}), this property holds exactly at every discrete
step.  Second, the CfC approximation replaces the ODE solution with an
exponential-decay closed form; the quality of this approximation degrades
when the nonlinear term in $f_\theta$ is large relative to the linear decay,
creating a tradeoff between efficiency and representational fidelity that
does not arise in our framework, where the Cayley update is exact with
respect to norm preservation regardless of Hamiltonian magnitude.  Third,
LNNs decode through an affine-softmax layer, discarding any phase-like
structure that the ODE trajectory may have produced; our Born-rule decoding
is quadratic in the state, providing access to $O(N^2)$ pairwise features
(Lemma~\ref{lem:quadratic_lifting}) that the affine-softmax readout
cannot access without a corresponding linear-dimension increase
(Lemma~\ref{lem:softmax_rank}).

Taken together, the two frameworks represent complementary continuous-time
strategies: LNNs prioritize adaptive timescales and biological plausibility
in real-valued spaces, while our framework prioritizes algebraic structure
preservation---norm conservation and phase-sensitive readout---in
complex-valued spaces.  Whether the additional structural constraints of
our framework translate to improved performance on natural language tasks
remains an empirical question addressed by the experimental protocols of
Section~\ref{sec:predictions}.

\subsection{Quantum Cognition and Quantum-Inspired NLP}
\label{sec:quantum_cognition}

The hypothesis that human judgment exhibits interference effects has been
developed by Busemeyer and Bruza~\cite{busemeyer2012quantum}.  Empirical
evidence, such as order effects and the conjunction fallacy, suggests that
human probability judgments sometimes violate classical axioms in ways
described by quantum probability.  For example, the probability of ``A and
B'' can be rated higher than the probability of ``A'' alone, which is
naturally explained by constructive interference in a projector-based
measurement model.

We cite this literature not to claim that language models \emph{must}
simulate human cognitive processes, but to motivate the architecture as a
potentially useful inductive bias.  If human language production leaves
traces of these interference-like structures, a model equipped with complex
amplitudes and Born-rule decoding may represent them more efficiently than
a classical probability model.  The Born-rule measurement postulate has
begun to attract direct empirical attention in NLP: recent work on text
classification with Born's rule~\cite{guidotti2022text} provides
evidence that squared-amplitude measurement operators can serve as
effective output mechanisms for natural language tasks, contextualizing our
architectural choice within an emerging body of quantum-inspired NLP
research.  Similarly, quantum-inspired feature maps in reproducing kernel
Hilbert spaces~\cite{schuld2019quantum} provide theoretical grounding for
using complex inner products as similarity measures in high-dimensional
feature spaces.  Existing quantum cognition and quantum-inspired NLP models
are largely either descriptive (fitting data with hand-selected operators)
or restricted to small-scale demonstrations.  Our framework is a trainable
realization of this mathematical structure, scaling it to vocabulary sizes
of tens of thousands and learning the operators from data.

\subsection{Quantum Machine Learning and Quantum Neural Networks}

The term \emph{quantum machine learning} (QML) encompasses two related but
distinct paradigms that are important to distinguish from our approach.
The first, and currently dominant, paradigm investigates how machine
learning computations can be accelerated or qualitatively enriched by
running on actual quantum processors, so-called near-term noisy
intermediate-scale quantum (NISQ) devices~\cite{preskill2018quantum}.  The
second paradigm develops quantum-informed mathematical structures as
inductive biases for models that run entirely on classical hardware, of
which our framework is an instance.  We discuss both, as each informs the
design and interpretation of the present work.

\paragraph{Parameterized quantum circuits and quantum neural networks.}
In the NISQ paradigm, \emph{quantum neural networks} (QNNs) are typically
implemented as parameterized quantum circuits
(PQCs)~\cite{cerezo2021variational, mitarai2018quantum}: a sequence of
parameterized unitary gates drawn from a hardware-native gate set is
applied to qubits initialized in a standard state, and a measurement in the
computational basis extracts classical outputs.  The measurement step is
governed, as in our framework, by Born's rule: a projective measurement on
a state $\ket{\psi}$ produces outcome $k$ with probability
$|\langle k | \psi \rangle|^2$.  This shared measurement postulate makes
the mathematics of PQC-based classifiers directly analogous to our
Born-rule decoding layer.  The key distinction is the substrate and the
dimension: PQCs operate on quantum amplitude vectors that physically exist
in a $2^n$-dimensional Hilbert space for $n$ qubits (with quantum
interference arising from genuine quantum superposition), whereas our
framework maintains a classically stored, explicitly computed state vector
in $\mathbb{C}^N$.

The expressivity and trainability of QNNs are active research areas.
\cite{biamonte2017quantum} provided an influential early
survey situating quantum machine learning within the broader landscape of
quantum algorithms.  Havl\'{i}\v{c}ek et al.~\cite{havlicek2019supervised}
demonstrated that quantum feature maps can produce kernel methods with
potential classification advantages on specific structured datasets.  Abbas
et al.~\cite{abbas2021power} characterized the effective dimension and
trainability of QNNs, showing that they can exhibit high capacity per
parameter in certain regimes.  However, the ``barren plateau''
phenomenon~\cite{mcclean2018barren} poses a significant trainability
obstacle analogous to (though more severe than) the vanishing-gradient
problem in classical deep networks; recent work by Cerezo et
al.~\cite{cerezo2021cost} has analyzed the conditions under which local
cost functions resist barren plateaus in shallow circuits.

\paragraph{Relation to the present framework.}
Our framework relates to the QML literature in three ways.  First, we
import the Born-rule measurement postulate and the Hermitian-Hamiltonian
evolution law as algebraic constraints on a classical sequence model, not
as descriptions of any physical quantum process.  The representational
advantages we establish ($O(N^2)$ effective features from a Born-rule
readout (Lemma~\ref{lem:quadratic_lifting}), exact norm-conservation from
Hermitian dynamics) are purely algebraic consequences that apply on
classical hardware.  Second, the gradient stability analysis of
Section~\ref{sec:init_training} is structurally analogous to the
barren-plateau analysis of PQCs: in both cases the central question is
whether useful gradient signals propagate through a composition of unitary
transformations.  Our Crank--Nicolson discretization guarantees that the
gradient of the loss with respect to the recurrent state propagates without
change in norm along the state pathway, providing an architectural
resolution to the state-gradient form of this problem, although
gradients through the parameter pathway (inside $g_\theta$) are not covered
by this guarantee, as discussed in Section~\ref{sec:init_training}.  Third,
the data-encoding and measurement structure of our Born-rule layer resembles
the data-reuploading classifiers of P\'{e}rez-Salinas et
al.~\cite{perezsalinas2020data}, in which classical data is injected at
multiple circuit layers to improve expressivity; in our model, the
analogous role is played by the state-dependent Hamiltonian $H_{\mathrm{int}}(t)$,
which re-encodes the current token at every step in a manner whose depth is
effectively determined by the sequence length.

\paragraph{Quantum reservoir computing and quantum-inspired classical
algorithms.}
A third connection arises through quantum reservoir
computing~\cite{fujii2017harnessing, ghosh2019quantum}, which uses the
dynamics of a (fixed, untrained) quantum system as a nonlinear feature map
whose output is read by a simple trained linear layer.  Quantum reservoir
computers are the quantum analogue of echo state
networks~\cite{jaeger2004harnessing}; only the output weights are trained,
making the approach computationally lightweight but limiting the model's
capacity to that of a kernel method applied to the reservoir's fixed
representational space.  In contrast, our framework trains the full
dynamics, the Hamiltonian, measurement, and initial state, end-to-end.  The
broader quantum-inspired classical algorithms
literature~\cite{tang2019quantum} explores problems where the mathematical
structure of quantum algorithms (amplitude amplification, phase estimation)
can be simulated efficiently on classical hardware to yield formal
algorithmic speedups; our work belongs to this tradition in motivation but
targets representational rather than computational advantages.

We emphasize that the present architecture is entirely classical in its
implementation.  The state $\ket{\psi(t)}$ is stored as a vector of
floating-point complex numbers; the Cayley update is a deterministic linear
algebra operation; the Born-rule output is the squared modulus of an inner
product, all computable on commodity hardware without access to quantum
processors.  The physical terminology (Hamiltonian, wave function, Born
rule) names mathematical structures whose algebraic properties provide the
guarantees we prove.  This terminological choice is conventional in the
quantum-inspired ML literature and should not be read as claiming a
computational quantum advantage.

\subsection{Transformers and Attention Mechanisms}

Transformers~\cite{vaswani2017attention} dominate large-scale language
modeling.  The attention mechanism provides direct access to context,
bypassing the recurrent bottleneck.  Our framework differs in its
computational structure.  First, the Transformer's memory (KV cache) grows
as $O(T d L)$ with sequence length $T$, hidden dimension $d$, and number
of layers $L$.  More importantly, the standard scaled dot-product
self-attention operation requires $O(T^2 d)$ time and $O(T^2 + T d)$
memory per layer due to the all-pairs attention matrix, a quadratic
dependence on sequence length that becomes prohibitive for long documents.
Efficient attention variants, including linear
attention~\cite{katharopoulos2020transformers}, sparse attention such as
BigBird~\cite{zaheer2020big}, and hardware-optimized exact
attention~\cite{dao2022flashattention}, reduce this to $O(T d)$ or $O(T
\log T \, d)$, at the cost of approximating or restructuring the full
attention pattern.  Our recurrent state is fixed-size $O(N)$, independent
of sequence length.

Second, the softmax attention mechanism computes a convex combination of
value vectors.  While the Multi-Layer Perceptron (MLP) layers following
attention can implement subtraction or cancellation, the attention operation
itself is an averaging process.  In our framework, cancellation is an
intrinsic property of the state accumulation: two latent components can
cancel through destructive interference during the summation of state
updates.  This suggests a different mechanism for suppressing irrelevant
information: rather than driving an attention weight to zero via a
saturating softmax, the model can align phases to produce destructive
interference.

\section{The Quantum Sequence Model}
\label{sec:model}

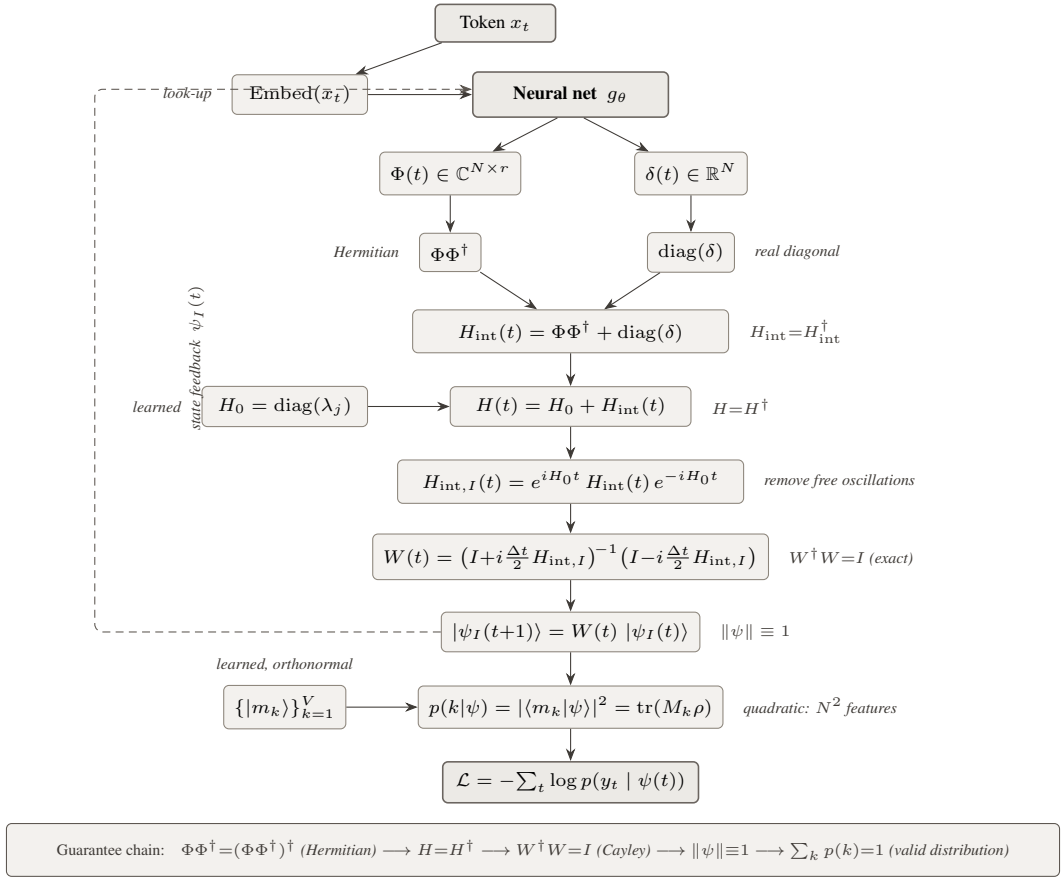
\begin{figure}[h!]
\centering
\begin{tikzpicture}[>=Stealth, font=\scriptsize,
  box/.style={draw, rectangle, rounded corners=2pt, minimum height=0.5cm,
              inner sep=4pt, align=center},
  op/.style={box, fill=figTeal, draw=figTealBorder},
  nn/.style={box, fill=figBlue, draw=figBlueBorder, line width=0.6pt},
  lbl/.style={font=\tiny, text=figSlate},
  flow/.style={->, >=Stealth, line width=0.4pt, figSlate},
  gflow/.style={->, >=Stealth, densely dashed, line width=0.4pt, figAccent},
]
\def\VS{0.95}

\node[font=\footnotesize\bfseries, anchor=north west] at (-5.6,1.6)
  {Single time step --- internal mechanism};

\node[nn, minimum width=1.6cm] (xt) at (0,0.65) {Token $x_t$};

\node[op, minimum width=1.8cm] (emb) at (-2.6,-0.3) {$\mathrm{Embed}(x_t)$};
\node[lbl, anchor=east] at ([xshift=-4pt]emb.west) {\emph{look-up}};
\draw[flow] (xt) -- (emb);

\node[nn, minimum width=2.6cm, minimum height=0.6cm] (gtheta) at (1.0,-0.3)
  {\textbf{Neural net}\; $g_\theta$};
\draw[flow] (emb) -- (gtheta);

\node[op] (phi) at (-0.6,-1.35) {$\Phi(t) \in \mathbb{C}^{N\times r}$};
\node[op] (delta) at (2.6,-1.35) {$\delta(t) \in \mathbb{R}^N$};
\draw[flow] ([xshift=-5pt]gtheta.south) -- (phi);
\draw[flow] ([xshift=5pt]gtheta.south) -- (delta);

\node[op] (PhiPhi) at (-0.6,-2.4) {$\Phi\Phi^\dagger$};
\node[op] (diagd) at (2.6,-2.4) {$\operatorname{diag}(\delta)$};
\draw[flow] (phi) -- (PhiPhi);
\draw[flow] (delta) -- (diagd);
\node[lbl, anchor=east] at ([xshift=-4pt]PhiPhi.west) {\emph{Hermitian}};
\node[lbl, anchor=west] at ([xshift=4pt]diagd.east) {\emph{real diagonal}};

\node[op, minimum width=4.2cm] (Hint) at (1.0,-3.45)
  {$H_{\mathrm{int}}(t) = \Phi\Phi^\dagger + \operatorname{diag}(\delta)$};
\draw[flow] (PhiPhi) -- (Hint);
\draw[flow] (diagd) -- (Hint);
\node[lbl, anchor=west] at ([xshift=4pt]Hint.east) {$H_{\mathrm{int}}{=}H_{\mathrm{int}}^\dagger$};

\node[op, minimum width=2.2cm] (H0) at (-2.8,-4.45) {$H_0 = \operatorname{diag}(\lambda_j)$};
\node[lbl, anchor=east] at ([xshift=-4pt]H0.west) {\emph{learned}};

\node[op, minimum width=3.2cm] (Ht) at (1.0,-4.45) {$H(t) = H_0 + H_{\mathrm{int}}(t)$};
\draw[flow] (Hint) -- (Ht);
\draw[flow] (H0) -- (Ht);
\node[lbl, anchor=west] at ([xshift=4pt]Ht.east) {$H{=}H^\dagger$};

\node[op, minimum width=4.6cm] (HintI) at (1.0,-5.45)
  {$H_{\mathrm{int},I}(t) = e^{iH_0 t}\, H_{\mathrm{int}}(t)\, e^{-iH_0 t}$};
\draw[flow] (Ht) -- (HintI);
\node[lbl, anchor=west] at ([xshift=4pt]HintI.east) {\emph{remove free oscillations}};

\node[op, minimum width=5.0cm] (cayley) at (1.0,-6.45)
  {$W(t) = \bigl(I{+}i\tfrac{\Delta t}{2}H_{\mathrm{int},I}\bigr)^{\!-1}\bigl(I{-}i\tfrac{\Delta t}{2}H_{\mathrm{int},I}\bigr)$};
\draw[flow] (HintI) -- (cayley);
\node[lbl, anchor=west] at ([xshift=4pt]cayley.east) {$W^\dagger W {=} I$ \emph{(exact)}};

\node[op, minimum width=3.2cm] (psi) at (1.0,-7.45)
  {$\ket{\psi_I(t{+}1)} = W(t)\,\ket{\psi_I(t)}$};
\draw[flow] (cayley) -- (psi);
\node[lbl, anchor=west] at ([xshift=4pt]psi.east) {$\|\psi\| \equiv 1$};

\node[op, minimum width=3.6cm] (born) at (1.0,-8.45)
  {$p(k|\psi) = |\!\braket{m_k}{\psi}\!|^2 = \operatorname{tr}(M_k\rho)$};
\draw[flow] (psi) -- (born);
\node[lbl, anchor=west] at ([xshift=4pt]born.east) {\emph{quadratic: $N^2$ features}};

\node[op] (mk) at (-2.8,-8.45) {$\{|m_k\rangle\}_{k=1}^V$};
\draw[flow] (mk) -- (born);
\node[lbl, anchor=south] at (-2.8,-8.1) {\emph{learned, orthonormal}};

\node[nn, minimum width=3.4cm] (loss) at (1.0,-9.45)
  {$\mathcal{L} = -\!\sum_t \log p(y_t \mid \psi(t))$};
\draw[flow] (born) -- (loss);

\draw[gflow, rounded corners=3pt]
  (psi.west) -- ++(-4.6,0) |- ([yshift=2pt]gtheta.west);
\node[lbl, rotate=90, anchor=south] at (-3.75,-3.8)
  {\emph{state feedback}\; $\psi_I(t)$};

\node[draw=figAccent, rounded corners=2pt, inner sep=6pt,
      font=\tiny, text=figSlate, fill=figTeal, minimum width=14cm, align=center]
  at (0.5,-10.35)
  {Guarantee chain:\quad
   $\Phi\Phi^\dagger {=} (\Phi\Phi^\dagger)^\dagger$\;\emph{(Hermitian)}%
   \;$\longrightarrow$\;%
   $H{=}H^\dagger$%
   \;$\longrightarrow$\;%
   $W^\dagger W{=}I$\;\emph{(Cayley)}%
   \;$\longrightarrow$\;%
   $\|\psi\|{\equiv}1$%
   \;$\longrightarrow$\;%
   $\sum_k p(k){=}1$\;\emph{(valid distribution)}};

\end{tikzpicture}
\caption{Detailed architecture of a single time step in the quantum sequence model.  The neural network $g_\theta$ (the \emph{sole} unconstrained learned component, bold outline) receives the token embedding and the current interaction-picture state, and outputs the complex matrix $\Phi(t)$ and real vector $\delta(t)$.  The outer product $\Phi\Phi^\dagger$ is Hermitian by construction; adding the learned diagonal $H_0$ yields the full Hamiltonian $H(t)$.  The interaction picture removes the known free oscillations, and the Cayley transform discretizes the remaining evolution into an exactly unitary update $W(t)$.  Unitarity preserves $\|\psi\|=1$ at every step, which ensures the Born rule produces a valid probability distribution over the vocabulary.  Left-side annotations mark learned components; right-side annotations trace the algebraic guarantees.  The dashed line indicates recurrent state feedback, which makes the overall dynamics nonlinear despite each individual step being a linear (unitary) map.}
\label{fig:architecture}
\end{figure}

The Introduction argued that complex-valued representations support
interference, a mechanism by which competing interpretations can suppress
one another through phase opposition rather than through explicit gating.
Section~\ref{sec:related} surveyed four families of models that each
capture part of this idea: unitary RNNs use complex-valued states but
discard phase at the output layer; state-space models use continuous-time
dynamics but operate in real-valued spaces; complex-valued networks employ
complex parameters but impose no structural constraints on the dynamics;
and Hamiltonian neural networks preserve conservation laws but operate in
real phase spaces designed for physical simulation.  No existing
architecture combines complex-valued states, Hamiltonian evolution, and
phase-sensitive decoding into a single trainable sequence model.

This section constructs such an architecture in six steps.  We first define
the state space (Section~\ref{sec:state_space}), explaining what structure
a complex unit-norm vector provides and why that structure is relevant to
the interference mechanism.  We then specify the evolution law
(Section~\ref{sec:hamiltonian}), decomposing it into a free component that
establishes a spectrum of oscillation timescales and an input-driven
component that couples latent dimensions in response to each observed
token.  A change of coordinates (Section~\ref{sec:interaction_picture})
factors out the known free oscillations, isolating the input-driven
dynamics in a form suited to numerical integration.  A
structure-preserving discretization (Section~\ref{sec:cayley}) converts
the continuous evolution into a discrete update that maintains unitarity
regardless of step size.  A measurement rule
(Section~\ref{sec:born_rule}) extracts token probabilities from the complex
state through a mechanism whose algebraic properties differ from the
standard softmax readout.  Finally, the initialization and training
objective (Section~\ref{sec:init_training}) complete the model
specification.

\begin{table}[tp]
\centering
\caption{Summary of all mathematical notation used throughout the paper.
All complex vectors are column vectors; $\|\cdot\|$ denotes the Euclidean
($L^2$) norm; $(\cdot)^\dagger$ denotes the conjugate transpose.}
\label{tab:notation}
\small
\setlength{\tabcolsep}{6pt}
\begin{tabular}{@{}llp{8.5cm}@{}}
\toprule
\textbf{Symbol} & \textbf{Space} & \textbf{Definition} \\
\midrule
\multicolumn{3}{l}{\textit{Scalars and index sets}} \\[2pt]
$N$ & $\mathbb{Z}_{>0}$ & Latent (complex Hilbert-space) dimension \\
$V$ & $\mathbb{Z}_{>0}$ & Vocabulary size \\
$T$ & $\mathbb{Z}_{>0}$ & Sequence length \\
$d$ & $\mathbb{Z}_{>0}$ & Token embedding dimension \\
$r$ & $\mathbb{Z}_{>0}$ & Rank of $\Phi(t)$; controls the interaction Hamiltonian's
                          expressivity-efficiency tradeoff ($r \ll N$) \\
$\Delta t$ & $\mathbb{R}_{>0}$ & Discretization step size (typically $\Delta t = 1$: one step per token) \\
$\lambda_j$ & $\mathbb{R}$ & $j$-th learnable oscillation frequency (eigenvalue of $H_0$) \\[6pt]
\multicolumn{3}{l}{\textit{State vectors}} \\[2pt]
$\ket{\psi(t)}$ & $\mathbb{C}^N$ & Schr\"odinger-picture latent state; $\|\psi(t)\|=1$ \\
$\ket{\psi_I(t)}$ & $\mathbb{C}^N$ & Interaction-picture latent state;
                                     $\ket{\psi(t)} = e^{-iH_0 t}\ket{\psi_I(t)}$ \\
$c_j(t)$ & $\mathbb{C}$ & $j$-th amplitude: $c_j(t)=\braket{j}{\psi(t)}$ \\
$p_j(t)$ & $[0,1]$ & Occupation probability of dimension $j$: $p_j(t)=|c_j(t)|^2$ \\[6pt]
\multicolumn{3}{l}{\textit{Hamiltonians and evolution}} \\[2pt]
$H(t)$ & $\mathbb{C}^{N\times N}$ & Full Hamiltonian: $H(t)=H_0+H_{\mathrm{int}}(t)$; Hermitian \\
$H_0$ & $\mathbb{C}^{N\times N}$ & Free Hamiltonian: $\mathrm{diag}(\lambda_0,\ldots,\lambda_{N-1})$ \\
$H_{\mathrm{int}}(t)$ & $\mathbb{C}^{N\times N}$ & Interaction (input-driven) Hamiltonian; Hermitian \\
$H_{\mathrm{int},I}(t)$ & $\mathbb{C}^{N\times N}$ & Interaction Hamiltonian in the interaction picture:
                           $e^{iH_0 t}H_{\mathrm{int}}(t)e^{-iH_0 t}$ \\
$\mathcal{U}(t_2,t_1)$ & $U(N)$ & Unitary time-evolution operator mapping $\ket{\psi(t_1)}$
                                   to $\ket{\psi(t_2)}$ \\
$W(t)$ & $U(N)$ & Cayley (Crank--Nicolson) discrete update operator \\[6pt]
\multicolumn{3}{l}{\textit{Interaction Hamiltonian parameterization}} \\[2pt]
$g_\theta$ & --- & Neural network mapping $(\mathrm{Embed}(x_t),\,
                   [\mathrm{Re}(c_I(t)), \mathrm{Im}(c_I(t))])\to(\Phi(t),\delta(t))$ \\
$\Phi(t)$ & $\mathbb{C}^{N\times r}$ & Low-rank factor of $H_{\mathrm{int}}(t)$; output of $g_\theta$ \\
$\delta(t)$ & $\mathbb{R}^N$ & Diagonal-shift vector; output of $g_\theta$ \\
$\tilde{\Phi}(t)$ & $\mathbb{C}^{N\times r}$ & Interaction-picture version:
                   $[\tilde{\Phi}]_{ja}=e^{i\lambda_j t}\Phi_{ja}(t)$ \\
$\mathrm{Embed}(x)$ & $\mathbb{R}^d$ & Learned token embedding for vocabulary item $x$ \\[6pt]
\multicolumn{3}{l}{\textit{Measurement and output}} \\[2pt]
$\ket{m_k}$ & $\mathbb{C}^N$ & Measurement vector for vocabulary item $k$ \\
$M$ & $\mathbb{C}^{N\times V}$ & Measurement matrix; columns are $\ket{m_k}$;
                                  satisfies $MM^\dagger=I_N$ \\
$M_k$ & $\mathbb{C}^{N\times N}$ & Rank-one projector $\ket{m_k}\bra{m_k}$ \\
$\rho$ & $\mathbb{C}^{N\times N}$ & Density matrix $\ket{\psi}\bra{\psi}$; rank-one Hermitian \\
$p(k\mid\psi(t))$ & $[0,1]$ & Born-rule token probability: $|\braket{m_k}{\psi(t)}|^2$ \\[6pt]
\multicolumn{3}{l}{\textit{Probability currents}} \\[2pt]
$J_{j\leftarrow k}(t)$ & $\mathbb{R}$ & Continuous-time probability current from
                           dimension $k$ to dimension $j$ \\
$J_{j\leftarrow k}^{\mathrm{mid}}(t)$ & $\mathbb{R}$ & Midpoint (exact discrete) probability current \\[6pt]
\multicolumn{3}{l}{\textit{Separation theorem}} \\[2pt]
$\mathcal{D}_N$ & --- & Disambiguation task family parameterized by $N$ \\
CUSM & --- & Complex Unitary Sequence Model (Definition~\ref{def:cusm}) \\
ROSM & --- & Real Orthogonal Sequence Model (Definition~\ref{def:rosm}) \\
$\rho_{ij}$ & $\mathbb{C}^{N\times N}$ & Task density matrix $W_j\ket{\psi_i}\bra{\psi_i}W_j^\dagger$ \\
\bottomrule
\end{tabular}
\end{table}

\subsection{State Space}
\label{sec:state_space}

The interference mechanism described in the Introduction requires that the
latent representation carry both magnitudes and phases: two contributions
to a prediction cancel only if their phases oppose one another.  To realize
this, we define the latent state at discrete time step $t$ as a unit-norm
vector in a finite-dimensional complex Hilbert space:
\begin{equation}
\label{eq:state}
\ket{\psi(t)} \in \mathbb{C}^N, \qquad \langle \psi(t) | \psi(t) \rangle = 1,
\end{equation}
where $N$ is a hyperparameter setting the model's capacity and $\langle
\cdot | \cdot \rangle$ is the standard Hermitian inner product.  Expanding
in the computational basis $\{\ket{0}, \ket{1}, \ldots, \ket{N-1}\}$, the
state takes the form
\begin{equation}
\ket{\psi(t)} = \sum_{j=0}^{N-1} c_j(t) \ket{j},
\end{equation}
where each coefficient $c_j(t) \in \mathbb{C}$ is called an amplitude.
The unit-norm constraint requires $\sum_j |c_j(t)|^2 = 1$, so the squared
magnitudes form a probability distribution over the $N$ latent dimensions
at every time step.

Each amplitude carries two degrees of freedom.  In polar form, $c_j(t) =
r_j(t)\, e^{i\theta_j(t)}$: the magnitude $r_j(t) = |c_j(t)|$ determines
how much weight the $j$-th dimension carries, while the phase $\theta_j(t)
= \arg c_j(t)$ encodes relational information that has no counterpart in
real-valued representations under a linear readout.  To see concretely why
phases matter, consider two states that share the same magnitude
distribution (e.g., say $|c_1| = |c_2| = 1/\sqrt{2}$ with all other
amplitudes zero) but differ in relative phase: in one state $c_2 = c_1$,
and in the other $c_2 = -c_1$.  Under a linear readout $w^\top h$ applied
to the squared magnitudes, these two states produce identical outputs.  But
under the inner product with a measurement vector $\ket{m}$, the overlap
$\langle m | \psi \rangle = m_1^* c_1 + m_2^* c_2$ differs between the
two cases.  In the first, the contributions from dimensions 1 and 2
reinforce; in the second, they cancel.  This is interference, and it occurs
because complex addition is phase-sensitive.  Standard real-valued
architectures achieve analogous suppression through dedicated learned
mechanisms such as gating (in LSTMs~\cite{hochreiter1997long} and GRUs) or
attention weight modulation (in Transformers~\cite{vaswani2017attention}).
These mechanisms require their own parameters and computational overhead.
The complex state space provides suppression through the geometry of
complex addition itself, without requiring additional parameters, though
only when paired with a phase-sensitive readout such as the Born rule
defined in Section~\ref{sec:born_rule}.

The connection to language modeling is as follows.  When the model
processes an ambiguous prefix, the state $\ket{\psi(t)}$ distributes
amplitude across latent dimensions associated (through training) with
different interpretations.  The relative phases between these amplitudes
record how the interpretations relate to one another: whether they are
compatible, conflicting, or independent.  As subsequent tokens arrive, the
evolution law (Section~\ref{sec:hamiltonian}) rotates these phases, and
the output rule (Section~\ref{sec:born_rule}) converts the resulting
interference pattern into token probabilities.  The unit-norm constraint
ensures that strengthening one interpretation necessarily weakens others:
probability is conserved and redistributed, not created or destroyed.  The
conserved probability currents derived in Section~\ref{sec:currents} make
this redistribution process visible.

The basis vectors $\ket{j}$ do not correspond to pre-assigned linguistic
categories.  They are abstract directions in $\mathbb{C}^N$ that acquire
semantic content through training, in the same way that hidden units in a
feedforward network develop interpretable roles through optimization.  What
distinguishes this state space from a standard real-valued hidden vector is
the geometry: the state lives on the complex unit sphere $S^{2N-1} \subset
\mathbb{C}^N$, which is a $(2N{-}1)$-dimensional manifold.  Each complex
coordinate contributes two real degrees of freedom (magnitude and phase),
so two points on $S^{2N-1}$ can share identical coordinate magnitudes
while differing in their phases, and both the dynamics and the output rule
are sensitive to these differences.  The separation theorem in
Section~\ref{sec:separation} formalizes this by showing that, for a
specific family of disambiguation tasks, the phase degrees of freedom
enable a complex state of dimension $N$ with Born-rule output to encode
relationships that require dimension $\Omega(N^2)$ in a real-valued model
with affine-softmax readout.

\subsection{Hamiltonian Decomposition}
\label{sec:hamiltonian}

Having defined the state space, we need a rule for how the state evolves
as the model processes tokens.  This rule must satisfy three properties
identified in the preceding sections.  First, it must preserve the unit
norm of the state at every step, since both the probabilistic
interpretation of the amplitudes and the Born-rule output
(Section~\ref{sec:born_rule}) depend on $\|\psi(t)\| = 1$.  Second, it
must allow the input token at each step to shape the dynamics, not merely
shift the state by an additive bias, as in the input injection $Wh_t +
Vx_t$ used by unitary RNNs (Section~\ref{sec:related}), but control the
axis and rate of the state's rotation on the unit sphere.  Third, it should
accommodate multiple timescales of variation, so that different latent
dimensions can be allocated to features that change at different rates.

We govern the state evolution by the time-dependent Schr\"{o}dinger
equation,
\begin{equation}
\label{eq:schrodinger}
i \frac{d}{dt} \ket{\psi(t)} = H(t) \ket{\psi(t)},
\end{equation}
where $H(t)$ is a Hermitian operator satisfying $H(t) = H(t)^\dagger$.
The Hermitian property is the single structural constraint on the dynamics,
and it is sufficient to guarantee norm preservation.  To verify this,
differentiate the norm:
\begin{equation}
\frac{d}{dt}\langle \psi | \psi \rangle = \left\langle \frac{d\psi}{dt}
\bigg| \psi \right\rangle + \left\langle \psi \bigg| \frac{d\psi}{dt}
\right\rangle = \langle \psi | (iH)^\dagger | \psi \rangle + \langle \psi
| (-iH) | \psi \rangle = i\langle \psi | H | \psi \rangle - i\langle \psi
| H | \psi \rangle = 0.
\end{equation}
The time-evolution operator $\mathcal{U}(t_2, t_1)$ mapping $\ket{\psi(t_1)}$
to $\ket{\psi(t_2)}$ is therefore unitary ($\mathcal{U}^\dagger \mathcal{U}
= I$): the state rotates on the unit sphere but never drifts off it.  This
addresses a stability concern raised in Section~\ref{sec:related}: Haber
and Ruthotto~\cite{haber2017stable} showed that norm-preserving dynamics
provide the strongest stability guarantees for deep networks, but their
antisymmetric-weight prescription only approximately preserves the norm
under Euler discretization.  Here, the Hermitian constraint guarantees norm
preservation at the continuous level, and the Cayley discretization in
Section~\ref{sec:cayley} will maintain it at the discrete level as well.
The practical consequence is that no gradient clipping, weight
normalization, or spectral regularization is needed to keep the state norm
stable.

We decompose the Hamiltonian into two terms with distinct roles:
\begin{equation}
\label{eq:hamiltonian_decomposition}
H(t) = H_0 + H_{\mathrm{int}}(t).
\end{equation}

\paragraph{The free Hamiltonian $H_0$.}
The first term is a diagonal matrix with learnable real entries:
\begin{equation}
H_0 = \operatorname{diag}(\lambda_0, \lambda_1, \ldots, \lambda_{N-1}), \qquad \lambda_j \in \mathbb{R}.
\end{equation}
When the interaction is absent ($H_{\mathrm{int}}(t) = 0$), each amplitude
evolves independently as $c_j(t) = c_j(0)\, e^{-i\lambda_j t}$, which is
a rotation in the complex plane at angular frequency $\lambda_j$.  The
parameters $\{\lambda_j\}$ define a learnable spectrum of timescales:
dimensions with large $|\lambda_j|$ oscillate rapidly, while dimensions
with small $|\lambda_j|$ change slowly.  Because $H_0$ is diagonal, it
induces no coupling between dimensions; its role is to establish the
baseline temporal structure that the interaction term will modulate.

The motivation for separating out $H_0$ is both conceptual and
computational.  Conceptually, the spectrum $\{\lambda_j\}$ provides the
model with a bank of oscillators at different frequencies, analogous to a
Fourier basis for temporal variation.  When the model is trained on
language data, the optimization can assign fast frequencies to dimensions
that track rapidly varying features (such as local syntactic constraints
that change with every token) and slow frequencies to dimensions that track
gradually varying features (such as topic or discourse structure that
persists over many sentences).  Whether the trained model actually
organizes its frequencies this way is an empirical question we address in
Section~\ref{sec:predictions}; the architecture provides the capacity for
this organization but does not enforce it.  Computationally, factoring out
$H_0$ enables the interaction-picture transformation
(Section~\ref{sec:interaction_picture}), which removes the known free
oscillations from the evolution operator and improves the numerical accuracy
of the Cayley discretization by reducing the effective operator norm that
governs the integration error.

This separation also connects to state-space models.  As noted in
Section~\ref{sec:related}, when $H_{\mathrm{int}}(t) = 0$ the
Schr\"{o}dinger equation reduces to $N$ decoupled oscillators
$\frac{d}{dt}c_j = -i\lambda_j c_j$, each evolving at a learned frequency.
This is structurally identical to the diagonalized form of
S4~\cite{gu2021efficiently}, except that in our framework, the complex
state variables are semantically load-bearing (i.e., their phases encode
relational information that the output rule will convert into token
probabilities) rather than an implementation artifact of a diagonalization.

\paragraph{The interaction Hamiltonian $H_{\mathrm{int}}(t)$.}
The second term depends on the input token observed at time $t$ and is
responsible for all coupling between latent dimensions and all input-driven
changes to the state trajectory.  We construct $H_{\mathrm{int}}(t)$ to be
Hermitian by design, regardless of the values produced by the underlying
neural network.  A network $g_\theta$ receives as input the concatenation
of the token embedding $\operatorname{Embed}(x_t) \in \mathbb{R}^d$ and
the current interaction-picture state represented as real and imaginary
parts $[\operatorname{Re}(c_I(t)),\, \operatorname{Im}(c_I(t))] \in
\mathbb{R}^{2N}$ (Section~\ref{sec:interaction_picture} defines this
representation; the interaction-picture amplitudes $c_I(t)$ differ from
the Schr\"{o}dinger-picture amplitudes $c(t)$ by the known phase factors
$e^{i\lambda_j t}$, which $g_\theta$ can absorb into its learned
parameters). The network $g_\theta$ is a standard feedforward neural
network, distinct from, and not to be confused with, the attention weight
matrices of Transformer models, that outputs two objects: a complex matrix
$\Phi(t) \in \mathbb{C}^{N \times r}$ and a real vector $\delta(t) \in
\mathbb{R}^N$.  The interaction Hamiltonian is then
\begin{equation}
\label{eq:interaction}
H_{\mathrm{int}}(t) = \Phi(t)\,\Phi(t)^\dagger + \operatorname{diag}(\delta(t)).
\end{equation}
This construction guarantees Hermiticity through two separate mechanisms.
The matrix $\Phi(t)\Phi(t)^\dagger$ is Hermitian because
$(\Phi\Phi^\dagger)^\dagger = \Phi\Phi^\dagger$ for any complex matrix
$\Phi$, and it is positive semidefinite with rank at most $r$, where $r
\ll N$ is a hyperparameter.  The diagonal matrix
$\operatorname{diag}(\delta(t))$ is Hermitian because it is real and
diagonal.  Their sum is therefore Hermitian for any output of $g_\theta$,
with no additional projection or correction step.

The rank parameter $r$ controls a tradeoff between expressivity and
efficiency.  An unrestricted $N \times N$ Hermitian matrix has $N^2$ real
degrees of freedom, which would be expensive for $g_\theta$ to produce at
each time step.  The low-rank factorization reduces this to $2Nr$ real
degrees of freedom (since $\Phi$ has $Nr$ complex entries, each with two
real components).  Each column $\phi_a(t) \in \mathbb{C}^N$ of $\Phi(t)$
specifies a coupling pattern across all $N$ dimensions: the rank-one matrix
$\phi_a \phi_a^\dagger$ has entry $[\phi_a]_j\, [\phi_a]_k^*$ at position
$(j,k)$, coupling dimensions $j$ and $k$ with a strength and phase
determined by the amplitudes and phases of $\phi_a$ at those positions.
The full off-diagonal interaction $\sum_{a=1}^{r} \phi_a \phi_a^\dagger$
is a superposition of $r$ such coupling patterns.  The diagonal term
$\operatorname{diag}(\delta(t))$ supplements this by allowing the network
to shift the effective oscillation frequency of individual dimensions in
response to the current token, without introducing additional off-diagonal
coupling.  Together, the two terms give $H_{\mathrm{int}}(t)$ the structure
needed to both redirect amplitude between dimensions (via off-diagonal
coupling) and modulate their oscillation rates (via diagonal shifts).

\paragraph{State-dependent Hamiltonians and nonlinearity.}
The fact that $g_\theta$ takes the current state as input, in addition to
the token embedding, is a deliberate design choice that determines the
model's expressivity class.  If $g_\theta$ depended only on the token
embedding, the Hamiltonian at each step would be a fixed function of the
input token, independent of the state.  The evolution over a full sequence
would then be a composition of token-determined unitary operators applied
to the initial state, a sequence-to-state map that is linear in the initial
state.  This would place the model in the same expressivity class as linear
state-space models such as S4 and Mamba~\cite{gu2023mamba}, whose
fundamental limitation (discussed in Section~\ref{sec:related}) is that
two tokens can only interact through accumulated products of linear
transition matrices.

By conditioning $g_\theta$ on the current state, the Hamiltonian at step
$t$ depends on the outcome of all previous steps, and the overall map from
input sequence to state trajectory becomes nonlinear.  The nonlinearity
resides entirely in the map from (token, state) pairs to Hamiltonians; the
evolution under any single fixed $H(t)$ remains a linear unitary map.
This separation is the mechanism by which the architecture achieves
nonlinear input-output behavior while preserving the norm-conservation
guarantee.  The situation is analogous to a time-varying linear system in
control theory: at each instant the dynamics are linear, but because the
transition operator changes at each step as a function of both the input
and the state, the composite behavior across steps is nonlinear.

This arrangement also contrasts with the input mechanism used by unitary
RNNs.  As discussed in Section~\ref{sec:related}, both Arjovsky et
al.~\cite{arjovsky2016unitary} and Wisdom et al.~\cite{wisdom2016full}
inject the input as an additive bias: $h_{t+1} = Wh_t + Vx_t$, where $W$
is a fixed unitary matrix.  The unitary part of the dynamics is independent
of the input, and the input only displaces the state after the rotation.
In our framework, the input token parameterizes the Hamiltonian
$H_{\mathrm{int}}(t)$, which is the generator of the rotation itself.  The
input thereby controls which dimensions couple and at what rate, rather
than adding a post-rotation offset.  For language modeling, this distinction
matters: when a disambiguating token like ``steep'' follows the prefix
``The bank was,'' it should not merely shift the state toward one
interpretation but should restructure the interference pattern among all
active interpretations.  An input-dependent Hamiltonian achieves this
because the relative phases accumulated under $H_{\mathrm{int}}(t)$
determine which components interfere constructively and which cancel in the
subsequent Born-rule readout.

\subsection{The Interaction Picture}
\label{sec:interaction_picture}

The free Hamiltonian $H_0$ produces phase rotations at rates $\lambda_j$
that persist regardless of the input.  These oscillations are known
analytically (i.e., the solution under $H_0$ alone is $c_j(t) =
c_j(0)\,e^{-i\lambda_j t}$), and they contribute to the operator norm of
the full Hamiltonian $H(t) = H_0 + H_{\mathrm{int}}(t)$ even when they
carry no information about the input-driven dynamics.  This matters for
numerical accuracy: the Cayley discretization introduced in
Section~\ref{sec:cayley} is a second-order integrator whose local
truncation error scales with the cube of the operator norm of the
Hamiltonian being integrated~\cite{hairer2006geometric}.  When some free
frequencies $\lambda_j$ are large, the norm $\|H(t)\|$ is dominated by
$H_0$ rather than by the input-driven interaction $H_{\mathrm{int}}(t)$,
and the integration error is correspondingly inflated.

The interaction picture is a change of variables, standard in quantum
mechanics~\cite{sakurai1994modern}, that factors out the free oscillations
analytically and leaves an evolution equation driven only by the
interaction.  By removing $H_0$ from the Hamiltonian that the integrator
must approximate, the interaction picture reduces the effective operator
norm to $\|H_{\mathrm{int}}(t)\|$, which depends only on the input-driven
coupling strength and is independent of the free frequencies.  This
improves the accuracy of the Cayley discretization without affecting its
unitarity guarantee, which holds unconditionally.

Define the interaction-picture state by
\begin{equation}
\label{eq:ip_state}
\ket{\psi(t)} = e^{-iH_0 t}\ket{\psi_I(t)},
\end{equation}
so that $\ket{\psi_I(t)} = e^{iH_0 t}\ket{\psi(t)}$ is the original state
with the free rotations undone.  Since $H_0$ is diagonal, this
transformation acts componentwise: the $j$-th amplitude in the interaction
picture is $[c_I(t)]_j = c_j(t)\, e^{i\lambda_j t}$, which removes the
free-oscillation factor from each component.  The forward pass of the
model operates on $c_I(t)$ throughout; the Schr\"{o}dinger-picture state
$c(t)$ can be recovered at any time via~\eqref{eq:ip_state} if needed, but
it is not maintained explicitly during inference.  As described in
Section~\ref{sec:hamiltonian}, the network $g_\theta$ receives the
interaction-picture amplitudes $c_I(t)$ as input; since $g_\theta$ is a
universal function approximator, it can learn to account for the known
phase relationship $c_j(t) = e^{-i\lambda_j t}\,[c_I(t)]_j$ between the
two representations.

To derive the evolution equation for $\ket{\psi_I(t)}$, substitute the
ansatz~\eqref{eq:ip_state} into the Schr\"{o}dinger
equation~\eqref{eq:schrodinger}.  The left-hand side becomes
\begin{equation}
i\frac{d}{dt}\left(e^{-iH_0 t}\ket{\psi_I}\right) = H_0 e^{-iH_0
t}\ket{\psi_I} + i\, e^{-iH_0 t}\frac{d}{dt}\ket{\psi_I},
\end{equation}
and the right-hand side is $(H_0 + H_{\mathrm{int}}(t))\,e^{-iH_0
t}\ket{\psi_I}$.  The $H_0$ terms cancel on both sides.  Left-multiplying
the remainder by $e^{iH_0 t}$ yields
\begin{equation}
\label{eq:ip_evolution}
i\frac{d}{dt}\ket{\psi_I(t)} = H_{\mathrm{int},I}(t)\ket{\psi_I(t)},
\end{equation}
where the interaction-picture Hamiltonian is
\begin{equation}
\label{eq:ip_hamiltonian}
H_{\mathrm{int},I}(t) = e^{iH_0 t}\, H_{\mathrm{int}}(t)\, e^{-iH_0 t}.
\end{equation}
Since $H_0$ is diagonal, this conjugation acts on individual matrix
entries as
\begin{equation}
\label{eq:ip_elements}
[H_{\mathrm{int},I}(t)]_{jk} = [H_{\mathrm{int}}(t)]_{jk}\, e^{i(\lambda_j - \lambda_k)t}.
\end{equation}
The diagonal entries ($j = k$) are unchanged, while each off-diagonal
entry acquires a time-dependent phase oscillating at the frequency
difference $\lambda_j - \lambda_k$ between the two dimensions it couples.

This frequency-dependent modulation has consequences for how latent
dimensions interact over multiple time steps.  Consider two dimensions $j$
and $k$ with similar natural frequencies ($\lambda_j \approx \lambda_k$).
The factor $e^{i(\lambda_j - \lambda_k)t}$ oscillates slowly in discrete
time, so the coupling $[H_{\mathrm{int}}(t)]_{jk}$ passes through to the
interaction picture with little modification across consecutive tokens.
Now consider two dimensions with very different frequencies ($|\lambda_j -
\lambda_k|$ large relative to $1/\Delta t$).  The rapidly oscillating
phase means that the effective coupling alternates in sign from one token
to the next, making it difficult for the optimization to learn consistent
amplitude transfer between these dimensions.  The result is an inductive
bias toward frequency-selective coupling: the optimization naturally favors
amplitude exchange between dimensions on similar timescales and penalizes
exchange between dimensions on very different timescales.  This is
analogous to the resonance phenomenon in coupled oscillators, where energy
transfers most efficiently between oscillators whose natural frequencies
are close.  We emphasize that this is a property of the architecture's
inductive bias over training, not a per-step dynamical effect: at any
single time step, the Cayley update applies the full
interaction-picture Hamiltonian $H_{\mathrm{int},I}(t)$ regardless of the
magnitude of the frequency differences.

Two properties of this transformation are important for the rest of the
construction.  First, it is an exact change of variables, not an
approximation.  The state $\ket{\psi_I(t)}$ contains the same information
as $\ket{\psi(t)}$ and can be recovered via~\eqref{eq:ip_state} at any
time.  Second, $\ket{\psi_I(t)}$ is constant whenever
$H_{\mathrm{int}}(t) = 0$: if no input arrives, the interaction-picture
state does not change.  This property is what makes the separation
computationally useful: the Cayley integrator in Section~\ref{sec:cayley}
operates on equation~\eqref{eq:ip_evolution}, whose right-hand side has
operator norm $\|H_{\mathrm{int},I}(t)\| = \|H_{\mathrm{int}}(t)\|$
(since unitary conjugation preserves the operator norm).  The integration
error therefore depends on the strength of the input-driven interaction,
not on the potentially much larger free frequencies $\{\lambda_j\}$.

\subsection{Cayley Discretization}
\label{sec:cayley}

The continuous evolution equation~\eqref{eq:ip_evolution} must be
converted to a discrete update for implementation on digital hardware.  The
choice of discretization is consequential because not all numerical
integrators preserve unitarity, and any norm drift that occurs at a single
step accumulates over the length of the sequence.

To illustrate the problem, consider forward Euler: replace
$i\frac{d}{dt}\ket{\psi_I} = H_{\mathrm{int},I}\ket{\psi_I}$ with
$\ket{\psi_{t+1}} = (I - i\Delta t\, H_{\mathrm{int},I}(t))\ket{\psi_t}$.
The update matrix $I - i\Delta t\, H_{\mathrm{int},I}$ is not unitary for
any nonzero $H_{\mathrm{int},I}$ and $\Delta t$.  To see this, note that
for an eigenvector of $H_{\mathrm{int},I}$ with eigenvalue $\mu$, the
update multiplies its norm by $|1 - i\Delta t\, \mu| = \sqrt{1 + \Delta
t^2 \mu^2} > 1$, so the state norm grows at each step.  Over a sequence of
$T$ tokens, this growth compounds, and for a language model processing
sequences of thousands of tokens the accumulated drift is large enough to
invalidate both the probabilistic interpretation of the state and the
Born-rule output.  Higher-order explicit methods (Runge--Kutta,
Adams--Bashforth) suffer from the same structural deficiency: they
approximate the unitary evolution operator with a polynomial in $\Delta t\,
H_{\mathrm{int},I}$, and no polynomial of finite degree other than the
trivial identity is unitary~\cite{hairer2006geometric}.

We instead discretize using the \emph{Cayley transform}, which arises from
the implicit midpoint rule, also known as the \textbf{Crank--Nicolson
scheme}~\cite{crank1947practical,hairer2006geometric,iserles2000lie}.  This
is the standard structure-preserving discretization for systems whose
continuous-time flow is unitary (or more generally symplectic), and it is
the canonical choice within the field of geometric numerical
integration~\cite{leimkuhler2004simulating}.  Evaluate the right-hand side
of~\eqref{eq:ip_evolution} at the midpoint $\frac{1}{2}(\ket{\psi_{t+1}} +
\ket{\psi_t})$ rather than at $\ket{\psi_t}$:
\begin{equation}
\label{eq:cayley_derivation}
\frac{\ket{\psi_{t+1}} - \ket{\psi_t}}{\Delta t} = -i\, H_{\mathrm{int},I}(t)\, \frac{\ket{\psi_{t+1}} + \ket{\psi_t}}{2}.
\end{equation}
Rearranging to collect $\ket{\psi_{t+1}}$ on the left:
\begin{equation}
\label{eq:cayley_update}
\left(I + \frac{i\Delta t}{2}\, H_{\mathrm{int},I}(t)\right)\ket{\psi_{t+1}} = \left(I - \frac{i\Delta t}{2}\, H_{\mathrm{int},I}(t)\right)\ket{\psi_t}.
\end{equation}
The discrete update operator is therefore
\begin{equation}
W(t) = \left(I + \frac{i\Delta t}{2}\, H_{\mathrm{int},I}(t)\right)^{-1}\left(I - \frac{i\Delta t}{2}\, H_{\mathrm{int},I}(t)\right),
\end{equation}
which is the Cayley transform of the skew-Hermitian matrix $K =
\frac{i\Delta t}{2}\, H_{\mathrm{int},I}(t)$.

\begin{proposition}[Unitarity of the Cayley update]
\label{prop:cayley_unitary}
If $H_{\mathrm{int},I}(t)$ is Hermitian, then $W(t)$ is unitary for every $\Delta t > 0$.
\end{proposition}

\begin{proof}
Define $K = \frac{i\Delta t}{2}\, H_{\mathrm{int},I}(t)$.  Because
$H_{\mathrm{int},I}(t)$ is Hermitian, $K^\dagger = -\frac{i\Delta t}{2}\,
H_{\mathrm{int},I}(t) = -K$, so $K$ is skew-Hermitian.  The update
operator is $W = (I + K)^{-1}(I - K)$.  For the adjoint, $W^\dagger = (I
- K^\dagger)[(I + K^\dagger)]^{-1} = (I + K)(I - K)^{-1}$.  Now observe
that $(I + K)(I - K) = I - K^2 = (I - K)(I + K)$, so $(I + K)$ and $(I -
K)$ commute and therefore so do their inverses.  Computing the product:
\begin{align}
W^\dagger W &= (I + K)(I - K)^{-1}(I + K)^{-1}(I - K) \nonumber \\
&= (I + K)(I + K)^{-1}(I - K)^{-1}(I - K) = I. \qedhere
\end{align}
\end{proof}

The unitarity of $W(t)$ holds for every step size $\Delta t$ and every
spectrum of $H_{\mathrm{int},I}(t)$.  There is no stability condition
requiring $\Delta t$ to be small relative to the eigenvalues of
$H_{\mathrm{int},I}$ for the purpose of norm preservation: the state norm
is exactly $1$ after every step, regardless of the Hamiltonian's magnitude.
After $T$ steps, $\|\psi_T\| = \|\psi_0\| = 1$ up to floating-point
rounding errors, which we correct with a single renormalization as a
numerical safeguard rather than an algorithmic necessity.

\paragraph{Unitarity versus accuracy.}
It is important to distinguish the unconditional norm-preservation guarantee
from the question of integration accuracy.  The Cayley transform is a
second-order integrator: its local truncation error (the difference between
the Cayley update and the exact solution of~\eqref{eq:ip_evolution} over
one step) is $O(\Delta t^3 \|H_{\mathrm{int},I}\|^3)$~\cite{hairer2006geometric}.
The norm of the error is controlled by $\|H_{\mathrm{int},I}\|$, which is
why the interaction picture (Section~\ref{sec:interaction_picture}) is
beneficial: it removes the free-frequency contribution $\|H_0\|$ from the
operator whose norm governs the error.  However, the error is nonzero for
any $\Delta t > 0$ and any nonzero $H_{\mathrm{int},I}$.  Over $T$ steps,
phase errors accumulate: the discrete trajectory
$\{\ket{\psi_I(t)}\}_{t=0}^T$ deviates from the exact continuous-time
trajectory by an amount that grows with $T$.  This phase drift does not
violate any property used by the model.  The Born-rule output at each step
depends on the current discrete state, not on its fidelity to the
continuous-time solution; the training objective
(Section~\ref{sec:init_training}) optimizes the parameters of the discrete
model directly.  The continuous-time Schr\"{o}dinger
equation~\eqref{eq:schrodinger} serves as the architectural motivation for
the update rule, not as a target that the discrete model must track.  The
guarantees that carry over exactly from continuous to discrete time are norm
preservation (Proposition~\ref{prop:cayley_unitary}), the validity of the
Born-rule output, and the conservation properties of the probability
currents (Section~\ref{sec:currents}).  Trajectory accuracy is a separate
property that depends on $\Delta t$ and $\|H_{\mathrm{int},I}\|$.

It is worth noting that Wisdom et al.~\cite{wisdom2016full} also use the
Cayley transform in their Full-Capacity Unitary RNN, but in a different
role: they use it to parameterize a single fixed unitary transition matrix
from a learnable skew-Hermitian matrix.  In our framework, the Cayley
transform is applied at each time step to a different interaction-picture
Hamiltonian $H_{\mathrm{int},I}(t)$, producing a time-varying sequence of
unitary updates rather than a single fixed unitary matrix.  Both uses
exploit the same algebraic fact, namely that the Cayley transform maps
skew-Hermitian matrices to unitary matrices, but ours applies it as a
discretization of continuous dynamics (in the spirit of the Crank--Nicolson
method~\cite{crank1947practical}) rather than as a static parameterization.

In implementation, we solve the linear system~\eqref{eq:cayley_update} at
each step using a batched direct solver rather than explicitly forming the
inverse $(I + \frac{i\Delta t}{2}\, H_{\mathrm{int},I})^{-1}$.  The
low-rank-plus-diagonal structure of $H_{\mathrm{int},I}$ (inherited from
the construction in equation~\eqref{eq:interaction}) can be exploited via
the Woodbury matrix identity to reduce the cost from $O(N^3)$ to $O(Nr^2)$
per step.  The Woodbury factorization involves inverting an $r \times r$
matrix $G = I_r + \frac{i\Delta t}{2}\,\tilde{\Phi}^\dagger
D^{-1}\tilde{\Phi}$, where $D$ is the diagonal part of the coefficient
matrix (Section~\ref{sec:complexity}).  The conditioning of $G$ depends on
the spectrum of $\tilde{\Phi}^\dagger D^{-1}\tilde{\Phi}$, which in turn
depends on the current interaction Hamiltonian.  For typical Hamiltonian
magnitudes encountered during training, $G$ is well conditioned because
$\|H_{\mathrm{int},I}\|$ is moderate and $D$ is close to the identity.
If training drives $\|H_{\mathrm{int},I}\|$ to large values, the
conditioning of $G$ can degrade; monitoring the condition number of $G$
during training provides a diagnostic for this scenario.

\subsection{Born-Rule Decoding}
\label{sec:born_rule}

At each time step, the model must produce a probability distribution over a
vocabulary of $V$ tokens.  The standard approach in neural language models
is to compute a linear projection of the hidden state followed by softmax
normalization.  Section~\ref{sec:related} discussed a limitation of this
approach identified by Yang et al.~\cite{yang2018breaking}: when the hidden
dimension $N$ is smaller than the vocabulary size $V$, the matrix of
log-probabilities across all contexts has rank at most $N$, creating a
bottleneck that prevents the model from representing certain distributions
regardless of its parameter values.  We now define an alternative decoding
mechanism whose algebraic structure differs from the softmax readout in a
way that is directly relevant to this bottleneck and to the interference
mechanism described throughout the paper.

We associate with each vocabulary item $k \in \{0, 1, \ldots, V-1\}$ a
learnable measurement vector $\ket{m_k} \in \mathbb{C}^N$ and define the
probability of token $k$ as
\begin{equation}
\label{eq:born_rule}
p(k \mid \psi(t)) = \left|\langle m_k | \psi(t) \rangle\right|^2.
\end{equation}
This is a quadratic output rule: the probability depends on the squared
modulus of a complex inner product between the state and a learned vector.
The same mathematical form appears in quantum mechanics as the Born rule,
and we adopt that name to make the connection to the quantum cognition
literature (Section~\ref{sec:quantum_cognition}) explicit, while
emphasizing that the justification for using this output rule in our
framework is algebraic (quadratic access to phase information), not
physical.  The inner product $\langle m_k | \psi(t) \rangle =
\sum_{j=0}^{N-1} [m_k]_j^*\, c_j(t)$ is a complex number whose squared
modulus gives the output probability.  For these probabilities to form a
valid distribution over the vocabulary, they must sum to one.  Collecting
the measurement vectors as columns of a matrix $M = [\ket{m_0}, \ldots,
\ket{m_{V-1}}] \in \mathbb{C}^{N \times V}$, the condition $\sum_k
p(k|\psi) = 1$ for all unit-norm $\ket{\psi}$ is equivalent to
\begin{equation}
\label{eq:resolution}
\sum_{k=0}^{V-1} \ket{m_k}\bra{m_k} = MM^\dagger = I_N.
\end{equation}
This is the resolution-of-identity condition: the outer products of the
measurement vectors must sum to the identity on $\mathbb{C}^N$.  When $V
\geq N$ (which is always the case in language modeling, where typical
vocabularies contain tens of thousands of tokens while latent dimensions
are in the hundreds), this requires $M$ to have orthonormal rows.

\paragraph{Enforcing the resolution of identity.}
In implementation, we store $M$ as a learnable complex matrix and enforce
the row-orthonormality constraint $MM^\dagger = I_N$ via QR decomposition
at each forward pass: we compute the thin QR factorization of $M^\dagger
\in \mathbb{C}^{V \times N}$ and replace $M^\dagger$ with the $Q$ factor,
whose columns are orthonormal.  This is a projection onto the Stiefel
manifold $\mathrm{St}(N, V) = \{X \in \mathbb{C}^{V \times N} : X^\dagger
X = I_N\}$~\cite{absil2008optimization}.  Gradients of the loss with
respect to the pre-projection parameter matrix flow through the QR
decomposition via the chain rule; modern automatic differentiation
frameworks support differentiation through QR~\cite{townsend2016pymanopt}.
The resulting optimization trajectory is not identical to Riemannian
gradient descent on the Stiefel manifold~\cite{absil2008optimization}
(which would require a retraction and a metric correction), but it is a
well-studied approximation that has been used successfully in prior work on
orthogonal and unitary parameterizations~\cite{wisdom2016full}.  An
alternative approach would be to parameterize $M$ directly via a product
of Householder reflections or Givens rotations, which would eliminate the
need for the QR projection; we use QR for simplicity and defer the
comparison of manifold optimization strategies to future work.

To verify that the Born rule produces valid probabilities under this
constraint, compute
\begin{equation}
\sum_{k} p(k|\psi) = \sum_k \langle \psi | m_k \rangle \langle m_k |
\psi \rangle = \langle \psi | \left(\sum_k |m_k\rangle\langle m_k|\right)
| \psi \rangle = \langle \psi | I_N | \psi \rangle = \|\psi\|^2 = 1,
\end{equation}
where the last step uses the unit-norm property of the state.  This
derivation makes explicit the interdependence between the components of the
architecture: the Born rule produces valid probabilities \emph{because} the
Cayley discretization preserves the unit norm that the Hermitian Hamiltonian
conserves.  If the discretization introduced even a small norm drift at
each step, the output probabilities would no longer sum to one and a
corrective normalization would be needed, defeating the purpose of the
construction.

\paragraph{Algebraic structure of the output.}
The algebraic structure of the Born-rule output differs from the softmax
readout in a way that bears directly on the expressivity of the model.  In
a real-valued model with hidden state $h \in \mathbb{R}^N$ and output
weight matrix $W_o \in \mathbb{R}^{V \times N}$, the log-probability of
token $k$ is $\log p(k) = w_k^\top h - \log Z$, which depends on the
state only through the linear projection $w_k^\top h$.  The matrix of
these projections across all possible states and tokens has rank at most
$N$~\cite{yang2018breaking}.  Under the Born rule, the probability is
\begin{equation}
p(k|\psi) = \left|\sum_{j} [m_k]_j^*\, c_j\right|^2 = \sum_{j}\sum_{j'} [m_k]_j^*\, [m_k]_{j'}\, c_j^*\, c_{j'}.
\end{equation}
This is a quadratic form in the amplitudes, involving both the $N$ diagonal
terms $|[m_k]_j|^2 |c_j|^2$ (which depend only on magnitudes) and the
$\binom{N}{2}$ off-diagonal cross terms (which depend on relative phases).
Each cross term has the form
\begin{equation}
[m_k]_j^*\, [m_k]_{j'}\, c_j^*\, c_{j'} = |[m_k]_j|\,|[m_k]_{j'}|\, r_j\, r_{j'}\, e^{i(\alpha_{j'} - \alpha_j + \theta_{j'} - \theta_j)},
\end{equation}
where $\alpha_j = \arg[m_k]_j$ and $\theta_j = \arg c_j$.  The real part
of this expression contributes to $p(k|\psi)$, and its sign depends on the
combined phase $\alpha_{j'} - \alpha_j + \theta_{j'} - \theta_j$.  Two
states with identical magnitudes but different relative phases therefore
produce different token probabilities, an impossibility under a linear
readout of a real-valued state.  The quadratic structure means the
Born-rule output depends on $N^2$ features of the state (the $N$ squared
magnitudes and $N(N-1)$ real parameters in the off-diagonal cross terms),
compared to the $N$ features accessible to a linear readout.  The formal
version of this observation is Lemma~\ref{lem:quadratic_lifting} in
Section~\ref{sec:separation}, and the resulting dimensional separation is
Theorem~\ref{thm:separation}.  We note that this comparison is specific to
the affine-softmax readout used by standard architectures; a real-valued
model equipped with a polynomial readout of degree two or higher could
access a larger feature space, though such readouts are not standard in the
architectures surveyed in Section~\ref{sec:related}.

This quadratic structure is where the interference mechanism becomes
concrete at the output level.  The cross terms can contribute positively or
negatively to the probability of token $k$, depending on the phase
alignment.  When the Hamiltonian evolution causes the phases $\theta_j$ to
rotate such that previously constructive cross terms become destructive for
a given token $k$, the probability $p(k|\psi)$ decreases, not because a
gate has been driven to saturation, but because the amplitudes are now
interfering against that token.  Conversely, when the phases align
constructively for a different token $k'$, its probability increases.  This
is the Born-rule realization of the disambiguation process described in the
Introduction: the model suppresses one interpretation and reinforces another
by rotating phases in its latent space, and the quadratic output converts
these rotations into probability changes.

The connection to the quantum cognition literature
(Section~\ref{sec:quantum_cognition}) is also direct: the Born
rule~\eqref{eq:born_rule} is the same squared-amplitude measurement
postulate that Busemeyer and Bruza~\cite{busemeyer2012quantum} use to
explain interference effects in human judgment.  In their framework, the
measurement vectors correspond to possible responses in a psychological
experiment and the state encodes the subject's cognitive disposition; in
ours, the measurement vectors correspond to vocabulary tokens and the state
encodes the model's belief about the next token given the preceding context.
Whether the structural similarity between these two settings reflects a
deeper connection between language statistics and cognitive interference
patterns, or is merely a mathematical coincidence, is an empirical question
that the present theoretical framework cannot answer.

\subsection{Initialization and Training Objective}
\label{sec:init_training}

The initial state $\ket{\psi(0)}$ must be a unit-norm vector in
$\mathbb{C}^N$ defined independently of the input sequence, since no tokens
have been observed at $t = 0$.  We parameterize it using two real-valued
vectors $\mathbf{a}, \mathbf{b} \in \mathbb{R}^N$ as
\begin{equation}
\label{eq:init}
\ket{\psi(0)} = \frac{\mathbf{a} + i\mathbf{b}}{\|\mathbf{a} + i\mathbf{b}\|}.
\end{equation}
The division by the norm ensures $\|\psi(0)\| = 1$ for any values of
$\mathbf{a}$ and $\mathbf{b}$, so standard gradient-based
optimization~\cite{kingma2015adam} can update these parameters freely
without risking an invalid initial state.  The parameters are shared across
all training sequences and optimized jointly with the rest of the model.
After normalization, the magnitudes $|a_j + ib_j| / \|\mathbf{a} + i\mathbf{b}\|$
determine how much initial weight the model places on each latent dimension,
and the phases $\arg(a_j + ib_j)$ set the initial phase relationships
between dimensions before any tokens are processed.  Together, these
constitute a learned prior: the optimization discovers an initial
distribution of amplitude and an initial pattern of phase relationships
that provide a useful starting configuration for processing arbitrary input
sequences.

The training objective is the standard negative log-likelihood for
autoregressive language modeling, computed using the Born-rule
probabilities.  Given a sequence $x_0, x_1, \ldots, x_{T-1}$ with targets
$y_t = x_{t+1}$, the loss is
\begin{equation}
\label{eq:loss}
\mathcal{L} = -\sum_{t=0}^{T-1} \log p(y_t \mid \psi(t)) = -\sum_{t=0}^{T-1} \log \left|\langle m_{y_t} | \psi(t) \rangle\right|^2.
\end{equation}
This objective drives the model to adjust its parameters so that, at each
time step, the state $\ket{\psi(t)}$ has large overlap (in the Born-rule
sense) with the measurement vector $\ket{m_{y_t}}$ corresponding to the
correct next token.  Minimizing~\eqref{eq:loss} simultaneously shapes the
initial state (determining the starting point of the trajectory), the free
frequencies (determining the baseline timescale structure), the Hamiltonian
generator (determining how inputs steer the evolution), and the measurement
vectors (determining which state configurations correspond to which tokens).
The full set of learnable parameters is: the initial-state vectors
$\mathbf{a}$ and $\mathbf{b}$, the free frequencies $\{\lambda_j\}_{j=0}^{N-1}$,
the measurement matrix $M$, and the weights $\theta$ of the neural network
$g_\theta$.

Gradients of $\mathcal{L}$ with respect to all parameters are computed by
backpropagation through the sequence of Cayley integration steps.  Each
step involves solving the linear system~\eqref{eq:cayley_update}, and the
backward pass differentiates through this solve using implicit
differentiation: given a system $Ax = b$, the gradient with respect to any
parameter entering $A$ or $b$ requires solving one additional linear system
with the coefficient matrix $A^\dagger$.  This avoids forming or storing
the matrix inverse and is supported natively by modern automatic
differentiation frameworks.  The computational cost per time step is
dominated by the linear solve, which scales as $O(Nr^2)$ when the Woodbury
identity exploits the low-rank structure of $H_{\mathrm{int},I}$.

\paragraph{Gradient propagation through the state versus through the parameters.}
The unitarity of each Cayley step has a direct consequence for gradient
propagation through the hidden state.  The Jacobian of the Cayley update
$\ket{\psi_{t+1}} = W(t)\ket{\psi_t}$ with respect to the input state
$\ket{\psi_t}$ is the unitary matrix $W(t)$ itself.  The gradient of the
loss with respect to $\ket{\psi_t}$ is obtained by applying $W(t)^\dagger
= W(t)^{-1}$, which preserves the gradient's norm.  Over $T - t$ steps,
the gradient passes through $T - t$ unitary matrices, and its norm remains
unchanged.  This prevents the vanishing and exploding gradient problem for
the \emph{state gradient pathway}~\cite{vorontsov2017orthogonality},
providing the same benefit as the unitary RNNs of Arjovsky et
al.~\cite{arjovsky2016unitary} and Wisdom et al.~\cite{wisdom2016full}.

However, the gradient of the loss with respect to the \emph{parameters} of
$g_\theta$ at step $t$ involves an additional pathway: the gradient flows
from $\ket{\psi_{t+1}}$ through $W(t)$, into $H_{\mathrm{int},I}(t)$, and
then through the neural network $g_\theta$.  The network $g_\theta$ is an
unconstrained feedforward network with no norm-preservation guarantee.
Gradients passing through $g_\theta$ at each step can grow or shrink
depending on the condition of $g_\theta$'s Jacobian, and this growth or
shrinkage is not mitigated by the unitarity of the state dynamics.  The
unitary guarantee therefore addresses one source of gradient instability
(the recurrent state pathway) but not all sources (the parameter pathway
through $g_\theta$).  Standard techniques for stabilizing feedforward
networks, such as careful initialization, normalization within $g_\theta$,
and learning rate scheduling~\cite{kingma2015adam}, remain relevant for
the parameter gradients.

Therefore, the complete forward pass is as follows: at $t = 0$, the state
is initialized via~\eqref{eq:init}.  At each subsequent step $t$, the
network $g_\theta$ generates the interaction Hamiltonian
$H_{\mathrm{int}}(t)$ from the current token embedding and current
interaction-picture state; the interaction-picture Hamiltonian
$H_{\mathrm{int},I}(t)$ is computed via the
conjugation~\eqref{eq:ip_hamiltonian}; the Cayley
update~\eqref{eq:cayley_update} advances the interaction-picture state by
one step; and the Born rule~\eqref{eq:born_rule} produces a distribution
over the vocabulary from which the loss is computed.  The
loss~\eqref{eq:loss} aggregates the log-probabilities across all time
steps, and its gradients flow backward through the Cayley solves, through
the Hamiltonian generator $g_\theta$, and into all learnable parameters.
The entire pipeline is differentiable end-to-end, with the unitarity of
each Cayley step ensuring that gradients propagating along the state
pathway neither vanish nor explode, while gradients through the parameter
pathway are subject to the same considerations as in any deep feedforward
network.

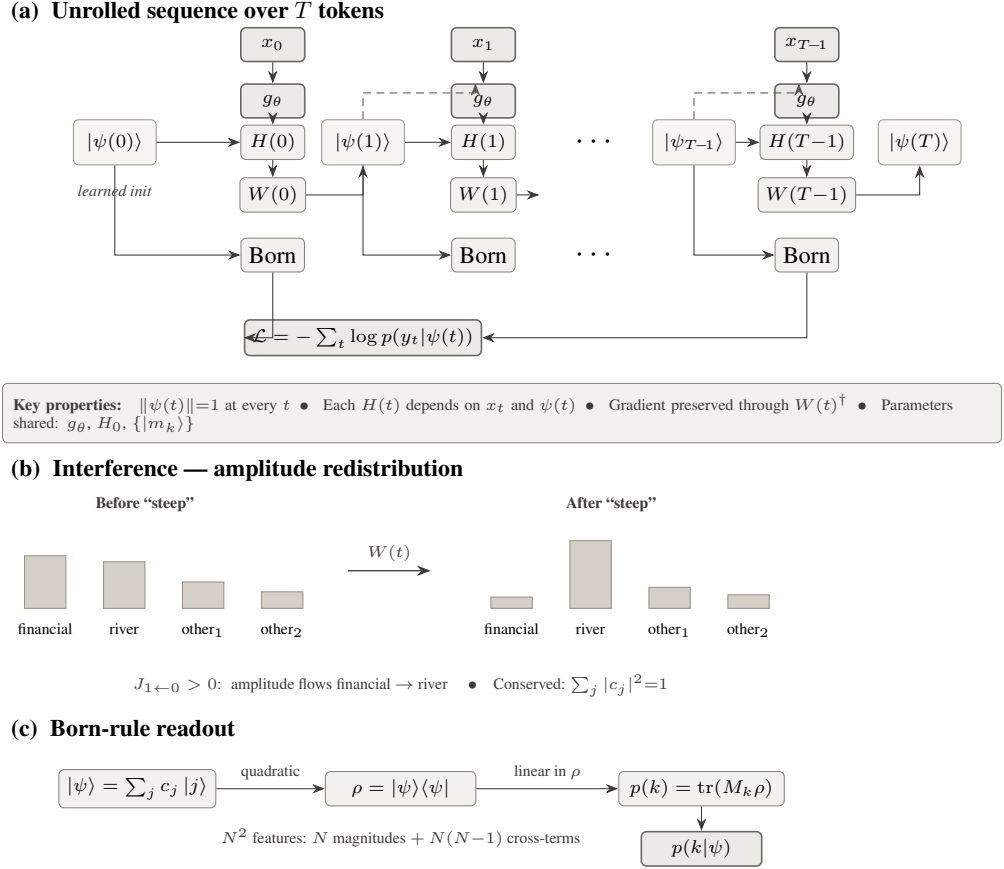
\begin{figure}[h!]
\centering
\begin{tikzpicture}[>=Stealth, font=\scriptsize,
  box/.style={draw, rectangle, rounded corners=2pt, minimum height=0.45cm,
              inner sep=3pt, align=center},
  op/.style={box, fill=figTeal, draw=figTealBorder},
  nn/.style={box, fill=figBlue, draw=figBlueBorder, line width=0.6pt},
  st/.style={box, fill=figState, draw=figStateBorder, minimum width=1.1cm, minimum height=0.6cm},
  lbl/.style={font=\tiny, text=figSlate},
  flow/.style={->, >=Stealth, line width=0.4pt, figSlate},
  gflow/.style={->, >=Stealth, densely dashed, line width=0.4pt, figAccent},
]

\node[font=\footnotesize\bfseries, anchor=north west] at (-4.8,2.0)
  {(a)\; Unrolled sequence over $T$ tokens};

\def\dx{2.8}
\def\Lshift{-1.5}

\pgfmathsetmacro{\psiZx}{-1.8+\Lshift}
\node[st] (psi0) at (\psiZx,0) {$\ket{\psi(0)}$};
\node[lbl, anchor=north] at (\psiZx,-0.45) {\emph{learned init}};

\pgfmathsetmacro{\sZx}{0.3+\Lshift}
\node[nn, minimum width=0.85cm] (x0) at (\sZx,1.3) {$x_0$};
\node[nn, minimum width=0.85cm] (g0) at (\sZx,0.55) {$g_\theta$};
\node[op, minimum width=0.85cm] (H00) at (\sZx,0) {$H(0)$};
\node[op, minimum width=0.85cm] (W0) at (\sZx,-0.7) {$W(0)$};
\draw[flow] (x0)--(g0); \draw[flow] (g0)--(H00);
\draw[flow] (H00)--(W0); \draw[flow] (psi0)--(H00);
\pgfmathsetmacro{\psiAx}{1.5+\Lshift}
\node[st] (psi1) at (\psiAx,0) {$\ket{\psi(1)}$};
\draw[flow] (W0.east) -| (psi1.south);
\node[op, minimum width=0.75cm] (b0) at (\sZx,-1.5) {\footnotesize Born};
\draw[flow] (psi0.south) |- (b0.west);

\pgfmathsetmacro{\xA}{\sZx+\dx}
\node[nn, minimum width=0.85cm] (x1) at (\xA,1.3) {$x_1$};
\node[nn, minimum width=0.85cm] (g1) at (\xA,0.55) {$g_\theta$};
\node[op, minimum width=0.85cm] (H01) at (\xA,0) {$H(1)$};
\node[op, minimum width=0.85cm] (W1) at (\xA,-0.7) {$W(1)$};
\draw[flow] (x1)--(g1); \draw[flow] (g1)--(H01);
\draw[flow] (H01)--(W1); \draw[flow] (psi1)--(H01);
\draw[flow] (W1.east) -- ++(0.3,0);
\draw[gflow] (psi1.north) -- ++(0,0.35) -| ([xshift=-3pt]g1.north);
\node[op, minimum width=0.75cm] (b1) at (\xA,-1.5) {\footnotesize Born};
\draw[flow] (psi1.south) |- (b1.west);

\pgfmathsetmacro{\xD}{\xA+1.5}
\node[font=\large] at (\xD,0) {$\cdots$};
\node[font=\large] at (\xD,-1.5) {$\cdots$};

\pgfmathsetmacro{\xB}{\xD+2.8}
\pgfmathsetmacro{\psiTx}{\xB-1.5}
\node[st] (psiTm) at (\psiTx,0) {$\ket{\psi_{T\!-\!1}}$};
\node[nn, minimum width=0.85cm] (xT) at (\xB,1.3) {$x_{T\!-\!1}$};
\node[nn, minimum width=0.85cm] (gT) at (\xB,0.55) {$g_\theta$};
\node[op, minimum width=0.85cm] (H0T) at (\xB,0) {$H(T{-}1)$};
\node[op, minimum width=0.85cm] (WT) at (\xB,-0.7) {$W(T{-}1)$};
\draw[flow] (xT)--(gT); \draw[flow] (gT)--(H0T);
\draw[flow] (H0T)--(WT); \draw[flow] (psiTm)--(H0T);
\pgfmathsetmacro{\psiFx}{\xB+1.5}
\node[st] (psiF) at (\psiFx,0) {$\ket{\psi(T)}$};
\draw[flow] (WT.east) -| (psiF.south);
\draw[gflow] (psiTm.north) -- ++(0,0.35) -| ([xshift=-3pt]gT.north);
\node[op, minimum width=0.75cm] (bT) at (\xB,-1.5) {\footnotesize Born};
\draw[flow] (psiTm.south) |- (bT.west);

\pgfmathsetmacro{\lossX}{1.5+\Lshift}
\node[nn, minimum width=2.8cm] (losstot) at (\lossX,-2.6)
  {$\mathcal{L} = -\sum_t \log p(y_t|\psi(t))$};
\draw[flow] (b0.south) |- (losstot.west);
\draw[flow] (bT.south) |- (losstot.east);

\node[draw=figTealBorder, fill=figTeal, rounded corners=2pt, inner sep=4pt,
      lbl, anchor=north west, text width=13cm, align=left] at (-4.8,-3.2) {%
  \textbf{Key properties:}\;\;
  $\|\psi(t)\|{=}1$ at every $t$\;\;$\bullet$\;\;
  Each $H(t)$ depends on $x_t$ and $\psi(t)$\;\;$\bullet$\;\;
  Gradient preserved through $W(t)^\dagger$\;\;$\bullet$\;\;
  Parameters shared: $g_\theta$, $H_0$, $\{|m_k\rangle\}$};

\node[font=\footnotesize\bfseries, anchor=north west] at (-4.8,-4.1)
  {(b)\; Interference --- amplitude redistribution};

\node[lbl, font=\tiny\bfseries, anchor=north] at (-2.9,-4.6)
  {Before ``steep''};

\def\ybar{-6.2}
\foreach \j/\h/\lab in {0/0.7/{financial}, 1/0.62/{river}, 2/0.35/{other$_1$}, 3/0.22/{other$_2$}} {
  \pgfmathsetmacro{\xb}{-4.5 + \j*1.05}
  \fill[figBar] (\xb,\ybar) rectangle (\xb+0.55,\ybar+\h);
  \draw[figBarBorder] (\xb,\ybar) rectangle (\xb+0.55,\ybar+\h);
  \node[font=\tiny, anchor=north] at (\xb+0.275,\ybar-0.08) {\lab};
}

\draw[->, >=Stealth, line width=0.5pt, figSlate]
  (-0.2,-5.7) -- node[above, font=\tiny] {$W(t)$} (0.9,-5.7);

\node[lbl, font=\tiny\bfseries, anchor=north] at (3.3,-4.6)
  {After ``steep''};

\foreach \j/\h/\lab in {0/0.15/{financial}, 1/0.9/{river}, 2/0.28/{other$_1$}, 3/0.18/{other$_2$}} {
  \pgfmathsetmacro{\xb}{1.7 + \j*1.05}
  \fill[figBar] (\xb,\ybar) rectangle (\xb+0.55,\ybar+\h);
  \draw[figBarBorder] (\xb,\ybar) rectangle (\xb+0.55,\ybar+\h);
  \node[font=\tiny, anchor=north] at (\xb+0.275,\ybar-0.08) {\lab};
}

\node[lbl, anchor=north, align=center] at (0.5,-6.95) {%
  $J_{1\leftarrow 0} > 0$:\; amplitude flows financial $\to$ river
  \quad$\bullet$\quad
  Conserved: $\sum_j |c_j|^2 {=} 1$};

\node[font=\footnotesize\bfseries, anchor=north west] at (-4.8,-7.55)
  {(c)\; Born-rule readout};

\def\yc{-8.6}
\node[op, minimum width=2.0cm] (psiR) at (-3.0,\yc) {$\ket{\psi} = \sum_j c_j\ket{j}$};

\node[op, minimum width=2.0cm] (rhoR) at (0.5,\yc) {$\rho = \ket{\psi}\!\bra{\psi}$};
\draw[flow] (psiR) -- node[above, lbl] {quadratic} (rhoR);

\node[op, minimum width=2.2cm] (prob) at (4.5,\yc) {$p(k) = \operatorname{tr}(M_k \rho)$};
\draw[flow] (rhoR) -- node[above, lbl] {linear in $\rho$} (prob);

\node[lbl, anchor=north] at (0.5,\yc-0.4) {$N^2$ features: $N$ magnitudes $+$ $N(N{-}1)$ cross-terms};

\node[nn, minimum width=1.6cm] (out) at (4.5,\yc-0.8) {$p(k|\psi)$};
\draw[flow] (prob) -- (out);

\end{tikzpicture}
\caption{Multi-scale view of the quantum sequence model.  \textbf{(a)}~The model unrolled over $T$ time steps.  At each step the same network $g_\theta$ generates a token- and state-dependent Hamiltonian; the Cayley update advances the state on the complex unit sphere; and the Born rule reads out token probabilities.  The loss aggregates log-probabilities across all steps.  \textbf{(b)}~Illustration of the interference mechanism.  Processing the disambiguating token ``steep'' after the prefix ``The bank was'' causes probability to flow from the \emph{financial} interpretation to the \emph{river} interpretation via conserved, antisymmetric probability currents $J_{j\leftarrow k}$.  \textbf{(c)}~Detail of the Born-rule readout.  The quadratic map $\psi\mapsto\psi\psi^\dagger$ lifts the $N$-dimensional complex state into the $N^2$-dimensional space of Hermitian matrices, exposing both magnitude and phase cross-terms to the linear measurement $\operatorname{tr}(M_k\rho)$.}
\label{fig:multistep}
\end{figure}

\section{Conserved Probability Currents}
\label{sec:currents}

The derivation in this section mirrors the standard quantum-mechanical
treatment of probability current in the Schr\"{o}dinger
equation~\cite{sakurai1994modern,griffiths2005introduction}.  Our
contribution is threefold: we interpret the resulting currents as a
diagnostic for \emph{latent probability flow} in a trained language model,
we show how the Hamiltonian decomposition $H = H_0 + H_{\mathrm{int}}$
isolates the input-driven component of each current, and we derive an exact
discrete-time variant (Section~\ref{sec:current_discrete}) consistent with
the Cayley update.

Section~\ref{sec:model} established that the Hermitian constraint on $H(t)$
guarantees conservation of the total probability $\sum_j |c_j(t)|^2 = 1$
at every time step.  This global conservation law ensures that the
Born-rule output (Section~\ref{sec:born_rule}) produces a valid probability
distribution, but it does not describe \emph{how} the model redistributes
amplitude among its latent dimensions as it processes each token.  Two
models that both conserve total probability may differ entirely in the
internal pattern of amplitude transfer they execute at each step.  This
section derives a finer-grained conservation law, a continuity equation
with explicit pairwise flux terms, that resolves the internal flow of
probability.  The derivation proceeds in continuous time, using the
Schr\"{o}dinger equation~\eqref{eq:schrodinger} that motivates the
architecture, and produces probability currents whose structural properties
(antisymmetry, conservation) follow algebraically from the Hermiticity of
$H(t)$.  Because the implemented model operates in discrete time via the
Cayley update (Section~\ref{sec:cayley}), the continuous-time currents are
approximations to the actual discrete probability changes;
Section~\ref{sec:current_discrete} quantifies this approximation and
defines exact discrete currents that inherit the same structural properties.

\subsection{Derivation of the Continuity Equation}
\label{sec:continuity}

Define the occupation probability of latent dimension $j$ at time $t$ as
\begin{equation}
\label{eq:occupation}
p_j(t) = |c_j(t)|^2,
\end{equation}
where $c_j(t) = \langle j | \psi(t) \rangle$ is the $j$-th amplitude of
the state vector.  The normalization constraint $\sum_j p_j(t) = 1$ ensures
that $\{p_j(t)\}_{j=0}^{N-1}$ forms a probability distribution over the
$N$ latent dimensions at each time $t$.  We now derive the equation
governing how this distribution evolves.

The Schr\"{o}dinger equation~\eqref{eq:schrodinger} in component form reads
\begin{equation}
\label{eq:component_schrodinger}
i\,\dot{c}_j(t) = \sum_{k=0}^{N-1} H_{jk}(t)\, c_k(t),
\end{equation}
so that $\dot{c}_j = -i \sum_k H_{jk}\, c_k$.  Taking the complex
conjugate and applying the Hermiticity condition $H_{jk}^* = H_{kj}$
yields
\begin{equation}
\label{eq:conjugate_schrodinger}
\dot{c}_j^*(t) = i \sum_{k=0}^{N-1} H_{kj}(t)\, c_k^*(t).
\end{equation}
Differentiating the occupation probability~\eqref{eq:occupation} via the
product rule:
\begin{align}
\dot{p}_j &= \dot{c}_j^*\, c_j + c_j^*\, \dot{c}_j \nonumber \\[4pt]
&= \left(i \sum_k H_{kj}\, c_k^*\right) c_j \;+\; c_j^*\left(-i \sum_k H_{jk}\, c_k\right) \nonumber \\[4pt]
&= \sum_{k=0}^{N-1} \Big(i\, H_{kj}\, c_k^*\, c_j \;-\; i\, H_{jk}\, c_j^*\, c_k\Big).
\label{eq:dpj_raw}
\end{align}
Each summand can be simplified by recognizing its algebraic structure.
Using $H_{kj} = H_{jk}^*$ and the commutativity of scalar multiplication,
the first term in each summand satisfies $H_{kj}\, c_k^*\, c_j =
H_{jk}^*\, c_k^*\, c_j = (H_{jk}\, c_j^*\, c_k)^*$.  Setting $w =
H_{jk}\, c_j^*\, c_k$, the summand becomes $i(w^* - w)$.  The identity
$i(w^* - w) = i(-2i\operatorname{Im}(w)) = 2\operatorname{Im}(w)$, valid
for any $w \in \mathbb{C}$, then gives the continuity equation:
\begin{equation}
\label{eq:continuity}
\boxed{\;\frac{dp_j}{dt} = \sum_{k=0}^{N-1} J_{j \leftarrow k}(t),\;}
\end{equation}
where the \textit{probability current} from dimension $k$ to dimension $j$
is
\begin{equation}
\label{eq:current}
J_{j \leftarrow k}(t) \;=\; 2\,\operatorname{Im}\!\Big(H_{jk}(t)\, c_j^*(t)\, c_k(t)\Big).
\end{equation}
The current $J_{j \leftarrow k}(t)$ is a real number at each time $t$.
When positive, probability is flowing from dimension $k$ into dimension
$j$; when negative, the flow is reversed.  The continuity
equation~\eqref{eq:continuity} states that the rate of change of
probability at any dimension is entirely accounted for by the sum of
pairwise currents into that dimension from every other dimension.

\subsection{Properties of the Probability Current}
\label{sec:current_properties}

The currents inherit structural properties from the Hermiticity of $H(t)$
that constrain them to be internally consistent: no current flows from a
dimension to itself, and every unit of probability arriving at one
dimension departs from exactly one other.

\begin{proposition}[Structural properties of the probability current]
\label{prop:current_properties}
The probability currents defined in~\eqref{eq:current} satisfy the
following three properties for all $j$, $k$, and $t$:
\begin{enumerate}
    \item \textbf{Antisymmetry:} $J_{j \leftarrow k}(t) = -J_{k \leftarrow j}(t)$.
    \item \textbf{Zero self-current:} $J_{j \leftarrow j}(t) = 0$.
    \item \textbf{Global conservation:} $\displaystyle\sum_{j=0}^{N-1} \frac{dp_j}{dt} = 0$.
\end{enumerate}
\end{proposition}

\begin{proof}
For property (1), compute $J_{k \leftarrow j} = 2\operatorname{Im}(H_{kj}\,
c_k^*\, c_j)$.  Since $H_{kj} = H_{jk}^*$ and $c_k^*\, c_j = (c_j^*\,
c_k)^*$, the argument is $H_{jk}^*\,(c_j^*\, c_k)^* = (H_{jk}\, c_j^*\,
c_k)^*$.  The identity $\operatorname{Im}(w^*) = -\operatorname{Im}(w)$
then yields $J_{k \leftarrow j} = -2\operatorname{Im}(H_{jk}\, c_j^*\,
c_k) = -J_{j \leftarrow k}$.

For property (2), set $k = j$ in~\eqref{eq:current}: $J_{j \leftarrow j}
= 2\operatorname{Im}(H_{jj}\,|c_j|^2)$.  The diagonal entry $H_{jj}$ of a
Hermitian matrix is real, and $|c_j|^2 \in \mathbb{R}$, so their product
is real and its imaginary part vanishes.

For property (3), sum the continuity equation~\eqref{eq:continuity} over
$j$: $\sum_j \dot{p}_j = \sum_j \sum_k J_{j \leftarrow k}$.  Every pair
$(j,k)$ with $j \neq k$ contributes $J_{j \leftarrow k} + J_{k \leftarrow
j} = 0$ by antisymmetry, and the diagonal terms $J_{j \leftarrow j}$
vanish by property (2).  The double sum is therefore zero.
\end{proof}

Since $J_{j \leftarrow j} = 0$, the continuity equation~\eqref{eq:continuity}
can equivalently be written as a sum restricted to $k \neq j$:
\begin{equation}
\label{eq:continuity_offdiag}
\frac{dp_j}{dt} = \sum_{k \neq j} J_{j \leftarrow k}(t).
\end{equation}
This form makes explicit that the rate of change of occupation probability
at dimension $j$ equals the net inflow from all other dimensions.  The
$\binom{N}{2}$ independent currents $\{J_{j \leftarrow k}(t)\}_{j < k}$
(antisymmetry halves the count) form a complete, internally consistent
ledger of probability transfers at each instant: every unit of probability
that arrives at one dimension is drawn from a specific other dimension,
with no residual or unaccounted-for flow.

\subsection{Structure of the Currents Under the Hamiltonian Decomposition}
\label{sec:current_decomposition}

The decomposition $H(t) = H_0 + H_{\mathrm{int}}(t)$ introduced in
Section~\ref{sec:hamiltonian} determines which component of the Hamiltonian
drives the probability currents.  Because $H_0 =
\operatorname{diag}(\lambda_0, \ldots, \lambda_{N-1})$ is diagonal, its
off-diagonal entries vanish: $[H_0]_{jk} = 0$ for $j \neq k$.  Therefore,
for any pair of distinct dimensions $j \neq k$,
\begin{equation}
H_{jk}(t) = [H_{\mathrm{int}}(t)]_{jk},
\end{equation}
and the current between them depends exclusively on the interaction
Hamiltonian:
\begin{equation}
\label{eq:current_interaction}
J_{j \leftarrow k}(t) = 2\,\operatorname{Im}\!\Big([H_{\mathrm{int}}(t)]_{jk}\, c_j^*(t)\, c_k(t)\Big) \qquad (j \neq k).
\end{equation}
This has an immediate and important consequence: when no input is driving
the interaction ($H_{\mathrm{int}}(t) = 0$), every off-diagonal current
vanishes, and the occupation probabilities $\{p_j(t)\}$ are individually
conserved.  Under $H_0$ alone, the amplitudes rotate in the complex plane
at their respective frequencies $\lambda_j$, but their squared magnitudes
do not change.  The free Hamiltonian modulates phases without transferring
probability; all probability redistribution is caused by the interaction
term $H_{\mathrm{int}}(t)$, which is generated by the neural network
$g_\theta$ from the current token embedding and state
(Section~\ref{sec:hamiltonian}).  The currents therefore isolate the causal
contribution of each token to the model's internal dynamics, separating
input-driven redistribution from the baseline phase evolution.

The low-rank-plus-diagonal construction of $H_{\mathrm{int}}(t)$ from
equation~\eqref{eq:interaction} further decomposes each current into
contributions from individual coupling channels.  Since
$[H_{\mathrm{int}}(t)]_{jk} = \sum_{a=1}^{r} \Phi_{ja}(t)\,\Phi_{ka}^*(t)$
for $j \neq k$ (the diagonal term $\operatorname{diag}(\delta(t))$
contributes only when $j = k$), the current separates as
\begin{equation}
\label{eq:current_channels}
J_{j \leftarrow k}(t) = 2 \sum_{a=1}^{r} \operatorname{Im}\!\Big(\Phi_{ja}(t)\,\Phi_{ka}^*(t)\, c_j^*(t)\, c_k(t)\Big).
\end{equation}
Each coupling channel $a \in \{1, \ldots, r\}$ contributes an independent
term to the total current between dimensions $j$ and $k$.  The $a$-th term
transfers probability at a rate determined jointly by two factors: the
coupling strengths $\Phi_{ja}(t)$ and $\Phi_{ka}(t)$, which are set by the
network $g_\theta$ in response to the current token and state, and the
current amplitudes $c_j(t)$ and $c_k(t)$, which reflect the model's
accumulated response to all previous tokens.  A current can flow between
$j$ and $k$ only if both dimensions carry nonzero amplitude and the
interaction Hamiltonian couples them; the direction of the flow (from $k$
to $j$ or from $j$ to $k$) is determined by the sign of
$\operatorname{Im}(\Phi_{ja}\,\Phi_{ka}^*\, c_j^*\, c_k)$, which depends
on the relative phases of the coupling and the amplitudes.

The state-dependence of the current formula merits emphasis.  The same
token can trigger different current patterns at different points in a
sequence, because the state $\ket{\psi(t)}$ encodes the cumulative effect
of all preceding tokens.  Two occurrences of the same word in different
contexts will produce different currents, even though the interaction
Hamiltonian $H_{\mathrm{int}}(t)$ depends on the same token embedding in
both cases, because the amplitudes $c_j(t)$ and $c_k(t)$ entering the
current formula differ.  This context-sensitivity is not an additional
feature that must be engineered; it is an algebraic consequence of the
bilinear dependence of~\eqref{eq:current} on the Hamiltonian and the
state.

\subsection{Interpretation as a Diagnostic}
\label{sec:current_diagnostic}

The probability currents provide a mechanistic account of how the model
processes each token, distinct from post-hoc interpretability methods in
both its derivation and its guarantees.  Gradient-based attribution
methods~\cite{simonyan2014deep} answer a counterfactual question (how would
the output change if this input were perturbed?) by linearizing the model
around its current operating point, an approximation whose fidelity depends
on the local curvature of the loss landscape.  Attention
visualization~\cite{vaswani2017attention} reveals which context positions a
Transformer attends to but not what computational role the attended content
plays.  Probing classifiers~\cite{belinkov2022probing} test whether specific
information is \emph{decodable} from a hidden state but do not reveal the
\emph{process} by which that information was incorporated.  The probability
currents answer a different question: at each time step, where is
probability flowing, from which dimensions, to which dimensions, and at
what rate?

The structural properties established in Proposition~\ref{prop:current_properties}
impose internal consistency constraints that post-hoc attributions are not
guaranteed to satisfy: any probability gained by one dimension is explicitly
sourced from one or more other dimensions, and the total is exactly
conserved.  These constraints follow algebraically from the Hermiticity of
$H(t)$ and cannot be violated regardless of the specific parameter values
or training procedure.

Two distinctions are necessary regarding the scope of these guarantees.
First, the continuity equation~\eqref{eq:continuity} and the current
formula~\eqref{eq:current} are derived from the continuous-time
Schr\"{o}dinger dynamics, while the implemented model advances in discrete
Cayley steps.  Section~\ref{sec:current_discrete} below shows that the
continuous-time currents approximate the actual discrete probability changes
to first order in $\Delta t$, and defines exact \emph{midpoint currents}
that reproduce the discrete changes without approximation error while
retaining antisymmetry and conservation.  For diagnostic use on the
implemented model, the midpoint currents
(equation~\eqref{eq:midpoint_current}) are the appropriate tool; the
continuous-time currents derived in this subsection provide the conceptual
framework and the structural guarantees that the midpoint currents inherit.

Second, the currents describe how probability redistributes across latent
dimensions in the model's internal state, not how probability redistributes
across vocabulary tokens in the output distribution.  The connection between
internal redistribution and output change is mediated by the Born rule
(Section~\ref{sec:born_rule}): a current flowing from dimension $k$ to
dimension $j$ increases $|c_j|^2$ and decreases $|c_k|^2$, and also
modifies the cross terms $c_j c_{j'}^*$ in the Born-rule expansion, both
of which affect the output distribution.  The currents therefore provide a
causal account of one specific aspect of the model's computation (amplitude
redistribution), from which output-level effects can be derived but are not
directly read off.

The channel decomposition~\eqref{eq:current_channels} adds a further layer
of resolution.  It identifies which of the $r$ coupling channels mediates
each part of a given probability transfer, linking the redistribution to
specific columns of the matrix $\Phi(t)$ produced by $g_\theta$.  After
training, the dominant currents triggered by a given token at a given
context constitute a token-level summary of the model's internal response:
which dimensions gain amplitude, which lose it, and which coupling pathways
carry the transfer.

This diagnostic is available at every time step and for every input
sequence.  It does not require selecting a specific output class with
respect to which attribution is computed (as gradient methods do), and it
does not require training a separate model (as probing methods do).  Its
computational cost is $O(N^2)$ per time step for the full current matrix,
or $O(Nr)$ per time step if only the channel-decomposed currents for a
preselected subset of dimension pairs are needed.  For the typical regime
in which $N$ is in the hundreds and $r \ll N$, this cost is modest relative
to the $O(Nr^2)$ cost of the Cayley update itself.

\subsection{Relation to the Discrete Update}
\label{sec:current_discrete}

The probability currents~\eqref{eq:current} are defined for the
continuous-time Schr\"{o}dinger dynamics.  In practice, the model advances
in discrete steps via the Cayley update (Section~\ref{sec:cayley}), and
the relationship between the continuous-time currents and the actual
discrete probability changes requires specification.  This subsection
establishes that relationship and defines an exact discrete alternative.

After one Cayley step with step size $\Delta t$, the discrete change in
occupation probability is
\begin{equation}
\Delta p_j(t) \;=\; |c_j(t + \Delta t)|^2 - |c_j(t)|^2.
\end{equation}
The Cayley transform implements the implicit midpoint rule, which is a
second-order integrator.  The continuous-time currents therefore
approximate the discrete changes as
\begin{equation}
\label{eq:current_discrete_approx}
\Delta p_j(t) = \Delta t \sum_{k \neq j} J_{j \leftarrow k}(t) \;+\; O(\Delta t^2).
\end{equation}
The $O(\Delta t^2)$ error arises from the difference between the
continuous-time derivative $\dot{p}_j(t)$ and the finite-difference
quotient $\Delta p_j(t)/\Delta t$.  For step sizes $\Delta t$ of order
unity (the natural choice when each discrete step corresponds to one
token), the continuous-time currents are first-order accurate estimates of
the actual discrete probability transfers.  This means the continuous-time
currents provide qualitatively correct but quantitatively approximate
descriptions of the discrete model's behavior: they correctly identify the
direction and relative magnitude of probability flow between dimensions,
but their numerical values may differ from the actual discrete changes by
an $O(\Delta t^2)$ correction.

Exact discrete probability changes $\Delta p_j(t)$ can always be computed
from the Cayley update itself, since the update is explicitly unitary and
the occupation probabilities before and after each step are available.
However, expressing $\Delta p_j$ in terms of the matrix entries of the
Cayley update operator $W(t)$ yields a sum over products of entries of
$W(t)$ and amplitudes that does not factor into a simple pairwise form
with guaranteed antisymmetry.  The continuous-time currents remain useful
because they provide the pairwise, antisymmetric, conserving decomposition
that makes the diagnostic interpretable.

\paragraph{Midpoint currents: an exact discrete decomposition.}
The tension between the appealing structural properties of the
continuous-time currents and the fact that the model operates in discrete
time can be resolved by defining currents evaluated at the implicit
midpoint that the Cayley scheme uses.  Specifically, evaluate the
continuous-time current formula~\eqref{eq:current} at the midpoint state
$\frac{1}{2}(\ket{\psi(t+\Delta t)} + \ket{\psi(t)})$ and the Hamiltonian
$H(t)$, mirroring the implicit midpoint evaluation that defines the Cayley
scheme (equation~\eqref{eq:cayley_derivation}).  The resulting midpoint
currents are
\begin{equation}
\label{eq:midpoint_current}
J_{j \leftarrow k}^{\text{mid}}(t) = 2\,\operatorname{Im}\!\Big(H_{jk}(t)\, \bar{c}_j^*(t)\, \bar{c}_k(t)\Big), \qquad \bar{c}(t) = \tfrac{1}{2}\big(c(t+\Delta t) + c(t)\big).
\end{equation}
These midpoint currents satisfy antisymmetry by the same argument as
Proposition~\ref{prop:current_properties} (the proof depends only on the
Hermiticity of $H_{jk}$ and the algebraic structure of the imaginary part,
both of which hold identically for the midpoint amplitudes $\bar{c}_j$).
They also reproduce the exact discrete probability changes without residual
error:
\begin{equation}
\label{eq:midpoint_exact}
\Delta p_j(t) = \Delta t \sum_{k \neq j} J_{j \leftarrow k}^{\text{mid}}(t).
\end{equation}
To see why, note that the Cayley update is derived precisely by evaluating
the right-hand side of the Schr\"{o}dinger equation at the midpoint
(equation~\eqref{eq:cayley_derivation}).  The discrete change $\Delta p_j$
can be written as $\Delta p_j = \Delta t\, (\dot{p}_j)_{\text{mid}}$,
where $(\dot{p}_j)_{\text{mid}}$ is the time derivative of $p_j$ evaluated
at the midpoint state.  The derivation of the continuity equation in
Section~\ref{sec:continuity} applies identically to any state vector and
any Hermitian matrix, so evaluating it at $(\bar{c}, H)$ produces
$(\dot{p}_j)_{\text{mid}} = \sum_k J_{j \leftarrow k}^{\text{mid}}$, and
multiplying by $\Delta t$ yields~\eqref{eq:midpoint_exact}.

The midpoint currents therefore provide an exact, pairwise, antisymmetric
decomposition of each discrete update, with no approximation error.  The
cost of computing them is the same as for the continuous-time currents
($O(N^2)$ for the full matrix or $O(Nr)$ for channel-decomposed currents
on selected pairs), plus the requirement of having the post-step state
$\ket{\psi(t+\Delta t)}$ available, which it is after the forward pass.
For diagnostic purposes on the implemented discrete model, the midpoint
currents are the recommended tool: they inherit the structural properties
of Proposition~\ref{prop:current_properties} (antisymmetry, zero
self-current, global conservation) while exactly matching the discrete
probability changes that the Cayley update produces.

\paragraph{Summary of the two current variants.}
The continuous-time currents $J_{j \leftarrow k}(t)$
(equation~\eqref{eq:current}) are derived from the continuous
Schr\"{o}dinger equation and evaluated at the pre-step state $c(t)$.
They provide the conceptual foundation, the structural properties
(Proposition~\ref{prop:current_properties}), the Hamiltonian decomposition
(Section~\ref{sec:current_decomposition}), and the channel factorization
(equation~\eqref{eq:current_channels}).  They approximate the discrete
probability changes to first order, with $O(\Delta t^2)$ error.  The
midpoint currents $J_{j \leftarrow k}^{\text{mid}}(t)$
(equation~\eqref{eq:midpoint_current}) are evaluated at the average of the
pre-step and post-step states.  They inherit all structural properties,
support the same Hamiltonian and channel decompositions (by replacing
$c_j, c_k$ with $\bar{c}_j, \bar{c}_k$ in all formulas), and reproduce the
exact discrete probability changes.  The choice between the two depends on
whether qualitative interpretability (continuous-time currents, available
without the post-step state) or quantitative fidelity to the discrete
dynamics (midpoint currents, requiring the post-step state) is the
priority.

Having established the structural and diagnostic properties of the model's
dynamics, we turn in the next section to the question of representational
capacity: for a concrete family of tasks, how does the state dimension
required by a complex unitary model compare to the dimension required by a
real orthogonal model?  The formal answer, a quadratic lower bound of
$\Omega(N^2)$ on the real state dimension required to match a complex model
of dimension $N$, is proved in full as Theorem~\ref{thm:separation} and
Lemma~\ref{lem:softmax_rank} of Section~\ref{sec:lower_bound}.

\section{Expressivity Separation}
\label{sec:separation}

The preceding sections constructed a sequence model whose state evolves
unitarily in $\mathbb{C}^N$ and whose output probabilities are extracted by
the Born rule.  Section~\ref{sec:born_rule} observed that the Born-rule
output is a quadratic function of the complex amplitudes, involving
$\binom{N}{2}$ pairwise interference terms that depend on relative phases.
Section~\ref{sec:currents} showed that the probability currents generated
by input tokens redistribute amplitude through these same pairwise
channels.  These observations suggest that the complex-valued model encodes
more information per dimension than a real-valued model with a linear
readout, because the quadratic output structure converts $N$ complex
coordinates into $O(N^2)$ effective features.  This section formalizes that
suggestion as a separation theorem: we construct a concrete family of
disambiguation tasks, prove that a complex unitary model of dimension $N$
solves each task exactly, and prove that any real orthogonal model with an
affine-softmax readout requires dimension $\Omega(N^2)$ for the same task.

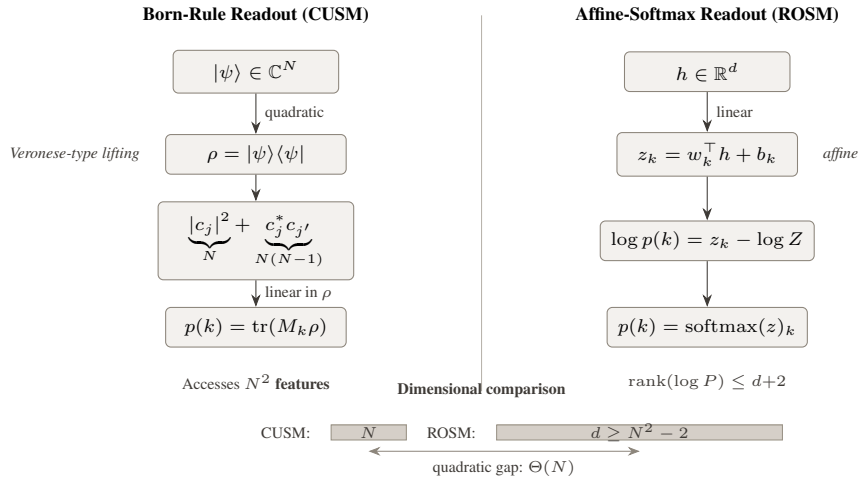
\begin{figure}[h!]
\centering
\begin{tikzpicture}[>=Stealth, font=\scriptsize,
  box/.style={draw, rectangle, rounded corners=2pt, minimum height=0.5cm,
              inner sep=4pt, align=center, draw=figTealBorder},
  lb/.style={box, fill=figTeal},
  arr/.style={->, >=Stealth, line width=0.4pt, figSlate},
  lbl/.style={font=\tiny, text=figSlate},
]

\node[font=\scriptsize\bfseries] at (-3.0,0.6) {Born-Rule Readout (CUSM)};

\node[lb, minimum width=2.2cm] (cs) at (-3.0,-0.15) {$\ket{\psi} \in \mathbb{C}^N$};

\node[lb, minimum width=2.4cm] (rho) at (-3.0,-1.25)
  {$\rho = \ket{\psi}\!\bra{\psi}$};
\node[lbl, anchor=east] at ([xshift=-6pt]rho.west) {\emph{Veronese-type lifting}};

\node[lb, minimum width=2.6cm] (feat) at (-3.0,-2.4)
  {$\underbrace{|c_j|^2}_{N} + \underbrace{c_j^* c_{j'}}_{N(N{-}1)}$};

\node[lb, minimum width=2.4cm] (bp) at (-3.0,-3.55) {$p(k) = \operatorname{tr}(M_k\rho)$};

\draw[arr] (cs) -- node[right, lbl] {quadratic} (rho);
\draw[arr] (rho) -- (feat);
\draw[arr] (feat) -- node[right, lbl] {linear in $\rho$} (bp);

\node[lbl, anchor=north, align=center] at (-3.0,-4.05) {%
  Accesses \textbf{$N^2$ features}};

\node[font=\scriptsize\bfseries] at (3.0,0.6) {Affine-Softmax Readout (ROSM)};

\node[lb, minimum width=2.2cm] (rs) at (3.0,-0.15) {$h \in \mathbb{R}^d$};

\node[lb, minimum width=2.4cm] (zk) at (3.0,-1.25)
  {$z_k = w_k^\top h + b_k$};
\node[lbl, anchor=west] at ([xshift=6pt]zk.east) {\emph{affine}};

\node[lb, minimum width=2.6cm] (logp) at (3.0,-2.4)
  {$\log p(k) = z_k - \log Z$};

\node[lb, minimum width=2.4cm] (sp) at (3.0,-3.55) {$p(k) = \mathrm{softmax}(z)_k$};

\draw[arr] (rs) -- node[right, lbl] {linear} (zk);
\draw[arr] (zk) -- (logp);
\draw[arr] (logp) -- (sp);

\node[lbl, anchor=north, align=center] at (3.0,-4.05) {%
  $\operatorname{rank}(\log P) \leq d{+}2$};

\draw[figDivider, line width=0.3pt] (0,0.6) -- (0,-4.4);

\node[lbl, font=\tiny\bfseries, anchor=south] at (0,-4.6) {Dimensional comparison};

\fill[figBar] (-2.0,-5.05) rectangle (-1.0,-4.85);
\draw[figBarBorder] (-2.0,-5.05) rectangle (-1.0,-4.85);
\node[lbl] at (-1.5,-4.95) {$N$};
\node[lbl, anchor=east] at (-2.15,-4.95) {CUSM:};

\fill[figBar] (0.2,-5.05) rectangle (4.0,-4.85);
\draw[figBarBorder] (0.2,-5.05) rectangle (4.0,-4.85);
\node[lbl] at (2.1,-4.95) {$d \geq N^2 - 2$};
\node[lbl, anchor=east] at (0.05,-4.95) {ROSM:};

\draw[<->, >=Stealth, line width=0.3pt, figTealBorder] (-1.5,-5.2) -- (2.1,-5.2)
  node[midway, below, lbl] {quadratic gap: $\Theta(N)$};

\node[draw=figAccent, rounded corners=2pt, inner sep=5pt, fill=figTeal,
      text width=11cm, align=center, font=\scriptsize] (thm) at (0,-6.05) {%
  \textbf{Theorem \ref{thm:separation} (conditional).}\;
  A CUSM of dimension $N$ solves $\mathcal{D}_N$ exactly.
  Any ROSM with affine-softmax readout that matches $p^*$ requires
  $d \geq N^2 - 2$, assuming $\operatorname{rank}(L^*) = N^2$
  (see Remark~\ref{rem:conditional}).};

\end{tikzpicture}
\caption{Source of the expressivity separation.  \textbf{Left:}~The Born-rule readout applies a quadratic (Veronese-type) lifting $\psi\mapsto\psi\psi^\dagger$, promoting the $N$-dimensional complex state to the $N^2$-dimensional space of rank-one Hermitian matrices.  The measurement $\operatorname{tr}(M_k\rho)$ is then linear in the lifted space, accessing all $N^2$ features including the $\binom{N}{2}$ pairwise phase cross-terms that encode interference.  \textbf{Right:}~The affine-softmax readout computes $z_k = w_k^\top h + b_k$, which is linear in the $d$-dimensional real state.  The rank constraint (Lemma~\ref{lem:softmax_rank}) limits the log-probability matrix to rank $d{+}2$.  \textbf{Bottom:}~Matching the $N^2$ features required by the Born-rule target forces $d \geq N^2 - 2$, a quadratic gap over the complex model's dimension $N$, under the full-rank condition on $L^*$ stated in Theorem~\ref{thm:separation}.}
\label{fig:separation}
\end{figure}

Three preliminary remarks are necessary to frame the result precisely.
First, the separation established here is between two specific \emph{output
mechanisms}: the Born rule (quadratic in the state) and the affine-softmax
readout (linear in the state).  The lower bound on the real model's
dimension (Lemma~\ref{lem:softmax_rank}) relies on the rank constraint
imposed by the affine-softmax structure~\cite{yang2018breaking}.  A
real-valued model equipped with a different readout, for instance a
polynomial readout of degree two or higher, would not be subject to the
same rank bound and could potentially narrow the gap.  The separation
should therefore be understood as a comparison between the standard
real-valued architecture (orthogonal dynamics with affine-softmax output,
as used by all unitary and orthogonal RNNs surveyed in
Section~\ref{sec:related}) and the proposed complex-valued architecture
(unitary dynamics with Born-rule output), not as a fundamental limitation
of real-valued representations independent of the output mechanism.
Second, both model classes defined below
(Definitions~\ref{def:cusm} and~\ref{def:rosm}) use state-independent
transitions: each token maps to a single fixed unitary or orthogonal
matrix, with no dependence on the current state.  This is strictly weaker
than the full architecture of Section~\ref{sec:model}, which uses
state-dependent Hamiltonians to achieve nonlinearity
(Section~\ref{sec:hamiltonian}).  The separation theorem proves an
advantage for the simplified CUSM, and since the full model subsumes the
CUSM as a special case, the advantage carries over; but the theorem does
not characterize the representational capacity of the full state-dependent
model.  Third, the formal task family $\mathcal{D}_N$ defined below uses
abstract context and query tokens, not natural-language words.  Whether
natural language contains subproblems whose structure is captured by
$\mathcal{D}_N$ is an empirical question that this section does not
address; the task family serves to establish a provable dimensional gap,
not to model a specific linguistic phenomenon.

\subsection{Model Classes for the Separation}
\label{sec:model_classes}

For the separation theorem, we compare two simplified model classes that
isolate the representational distinction between complex and real state
spaces.  Both assign one fixed transition matrix per alphabet token, with
no state-dependent dynamics.

\begin{definition}[Complex Unitary Sequence Model (CUSM)]
\label{def:cusm}
A CUSM of dimension $N$ over an alphabet $\Sigma$ and vocabulary $\mathcal{V}
= \{0, \ldots, V-1\}$ consists of:
\begin{itemize}
    \item An initial state $\ket{\psi_0} \in \mathbb{C}^N$ with $\|\psi_0\| = 1$,
    \item For each token $x \in \Sigma$, a unitary matrix $W_x \in U(N)$,
    \item A set of measurement vectors $\{\ket{m_k}\}_{k=0}^{V-1}$ in
          $\mathbb{C}^N$ satisfying the resolution of identity $\sum_{k}
          \ket{m_k}\bra{m_k} = I_N$.
\end{itemize}
Given an input sequence $x_1, x_2, \ldots, x_T$, the model updates its
state as
\begin{equation}
\ket{\psi(t)} = W_{x_t} \ket{\psi(t-1)}, \qquad t = 1, \ldots, T,
\end{equation}
and produces output probabilities at each step via the Born rule:
\begin{equation}
p(k \mid \psi(t)) = \left|\braket{m_k}{\psi(t)}\right|^2.
\end{equation}
\end{definition}

The CUSM is a special case of the full architecture of Section~\ref{sec:model},
obtained by restricting $g_\theta$ to depend only on the token embedding
(not on the state), setting $H_0 = 0$, and using $\Delta t = 1$ in the
Cayley discretization so that each token's unitary is $W_{x_t} =
\text{Cayley}(-iH_{\mathrm{int}}(x_t)/2)$.  The state-dependent
nonlinearity discussed in Section~\ref{sec:hamiltonian} is absent from the
CUSM; implementing it would strictly enlarge the class of representable
functions.

\begin{definition}[Real Orthogonal Sequence Model (ROSM)]
\label{def:rosm}
A ROSM of dimension $d$ over an alphabet $\Sigma$ and vocabulary $\mathcal{V}
= \{0, \ldots, V-1\}$ consists of:
\begin{itemize}
    \item An initial state $h_0 \in \mathbb{R}^d$ with $\|h_0\| = 1$,
    \item For each token $x \in \Sigma$, an orthogonal matrix $Q_x \in O(d)$,
    \item Output weight vectors $\{w_k\}_{k=0}^{V-1}$ in $\mathbb{R}^d$
          and a bias vector $b \in \mathbb{R}^V$.
\end{itemize}
Given an input sequence $x_1, x_2, \ldots, x_T$, the model updates its
state as
\begin{equation}
h(t) = Q_{x_t}\, h(t-1), \qquad t = 1, \ldots, T,
\end{equation}
and produces output probabilities at each step via the affine-softmax readout:
\begin{equation}
\label{eq:rosm_output}
p(k \mid h(t)) = \frac{\exp\!\big(w_k^\top h(t) + b_k\big)}{\displaystyle\sum_{k'=0}^{V-1} \exp\!\big(w_{k'}^\top h(t) + b_{k'}\big)}.
\end{equation}
\end{definition}

Both model classes use norm-preserving transitions: the CUSM preserves
$\|\psi\| = 1$, and the ROSM preserves $\|h\| = 1$.  The distinction lies
in the output mechanism.  The CUSM uses the Born rule, which is a quadratic
function of the complex state.  The ROSM uses an affine projection followed
by softmax, which depends on the state through the linear map $h \mapsto Wh
+ b$.  This difference in the algebraic degree of the output is the source
of the separation.

The affine-softmax readout in Definition~\ref{def:rosm} is the standard
output mechanism used by recurrent neural networks, including the unitary
and orthogonal RNNs of Arjovsky et al.~\cite{arjovsky2016unitary} and
Wisdom et al.~\cite{wisdom2016full}.  A real-valued model that replaced this
readout with a quadratic function of the state (e.g., $p(k) \propto
(w_k^\top h)^2$) would itself perform an implicit polynomial lifting and
would not be subject to the rank bound derived in
Lemma~\ref{lem:softmax_rank}.  The separation theorem compares the two
model classes as defined, not all possible combinations of dynamics and
readouts.

\subsection{The Disambiguation Task Family}
\label{sec:task_family}

For each integer $N \geq 2$, we define a disambiguation task $\mathcal{D}_N$
that requires the model to produce a specific probability distribution at
the final position of a sequence, where the target distribution depends
jointly on two tokens separated by an arbitrary number of filler tokens.

\paragraph{Alphabet and sequence structure.}
The input alphabet is $\Sigma_N = \{a_0, a_1, \ldots, a_{N-1}\} \cup
\{b_0, b_1, \ldots, b_{N-1}\} \cup \{\sigma\}$, containing $N$ context
tokens, $N$ query tokens, and a single filler token.  For a fixed sequence
length $T \geq 3$, the input sequences have the form
\begin{equation}
\label{eq:task_sequence}
s = (a_i,\; \underbrace{\sigma, \;\ldots,\; \sigma}_{T - 2},\; b_j), \qquad i, j \in \{0, \ldots, N-1\}.
\end{equation}
There are $N^2$ such sequences, one for each pair $(i, j)$.

\paragraph{Target distribution.}
The output vocabulary has $V \geq N^2$ tokens.  The target probability
distribution at the final position (after processing $b_j$) is defined in
terms of a specific Born-rule construction.  Choose:
\begin{enumerate}
    \item $N$ unit vectors $\ket{\psi_0}, \ldots, \ket{\psi_{N-1}} \in
          \mathbb{C}^N$ (the \emph{context states}),
    \item $N$ unitary matrices $W_0, \ldots, W_{N-1} \in U(N)$ (the
          \emph{query unitaries}),
    \item $V$ measurement vectors $\{\ket{m_k}\}_{k=0}^{V-1}$ in
          $\mathbb{C}^N$ satisfying the resolution of identity $\sum_k
          \ket{m_k}\bra{m_k} = I_N$ and the informationally complete
          condition (defined below).
\end{enumerate}
The target distribution for input pair $(i, j)$ is
\begin{equation}
\label{eq:target_distribution}
p^*(k \mid i, j) = \left|\braket{m_k}{W_j \psi_i}\right|^2 =
\operatorname{tr}\!\left(M_k\, W_j \ket{\psi_i}\bra{\psi_i} W_j^\dagger\right),
\end{equation}
where $M_k = \ket{m_k}\bra{m_k}$.

The context states, query unitaries, and measurement vectors are part of
the \emph{task specification}, not part of any model's learned parameters.
A model ``solves'' or ``computes'' the task $\mathcal{D}_N$ if, for every
input pair $(i, j)$ and every output token $k$, the model's output
probability at the final position equals $p^*(k \mid i, j)$.

\paragraph{Conditions on the task parameters.}
The context states and query unitaries must satisfy a spanning condition,
and the measurement vectors must be informationally complete.  We state
these conditions precisely.

\begin{definition}[General-position unitaries]
\label{def:general_position}
The context states $\{\ket{\psi_i}\}_{i=0}^{N-1}$ and query unitaries
$\{W_j\}_{j=0}^{N-1}$ are in \emph{general position} if the $N^2$ density
matrices
\begin{equation}
\label{eq:density_matrices}
\rho_{ij} = W_j \ket{\psi_i}\bra{\psi_i} W_j^\dagger, \qquad i, j \in \{0, \ldots, N-1\},
\end{equation}
are linearly independent as elements of the $N^2$-dimensional real vector
space of $N \times N$ Hermitian matrices.
\end{definition}

\begin{definition}[Informationally complete measurement]
\label{def:ic_povm}
The measurement vectors $\{\ket{m_k}\}_{k=0}^{V-1}$ form an
\emph{informationally complete} measurement if the map $\rho \mapsto
\big(\operatorname{tr}(M_0 \rho),\, \ldots,\, \operatorname{tr}(M_{V-1}
\rho)\big)$ is injective on the space of $N \times N$ Hermitian matrices.
Equivalently, the $V$ matrices $\{M_k = \ket{m_k}\bra{m_k}\}$ span the
space of $N \times N$ Hermitian matrices as a real vector space.
\end{definition}

Informational completeness requires $V \geq N^2$, since the space of
Hermitian matrices has real dimension $N^2$.  Such measurements exist for
every $N$: for instance, the $N^2$ outer products of a set of $N^2$ vectors
forming a weighted 2-design in $\mathbb{C}^N$ provide an informationally
complete measurement~\cite{renes2004symmetric}.  For the separation theorem,
we fix any informationally complete measurement with $V \geq N^2$, and
choose context states and query unitaries in general position; the next
subsection verifies that general-position configurations exist for every
$N$.

\subsection{Upper Bound: Complex Unitary Construction}
\label{sec:upper_bound}

\begin{proposition}[CUSM upper bound]
\label{prop:upper_bound}
For every $N \geq 2$ and every general-position configuration of context
states and query unitaries, a CUSM of dimension $N$ computes $\mathcal{D}_N$
exactly.
\end{proposition}

\begin{proof}
Set the initial state of the CUSM to $\ket{\psi_0^{\mathrm{init}}} =
\ket{\psi_0}$ (the first context state).  For each context token $a_i$,
define the transition unitary as $W_{a_i} = V_i$, where $V_i \in U(N)$ is
any unitary satisfying $V_i \ket{\psi_0} = \ket{\psi_i}$.  Such a unitary
exists for every pair of unit vectors in $\mathbb{C}^N$.  For the filler
token, set $W_\sigma = I_N$.  For each query token $b_j$, set $W_{b_j} =
W_j$.  Equip the model with the measurement vectors $\{\ket{m_k}\}$ from
the task specification.

After processing the input sequence $(a_i, \sigma^{T-2}, b_j)$, the state
is
\begin{equation}
\ket{\psi(T)} = W_j \cdot I_N^{T-2} \cdot V_i \cdot \ket{\psi_0} = W_j \ket{\psi_i}.
\end{equation}
The Born-rule output at position $T$ is
\begin{equation}
p(k \mid \psi(T)) = \left|\braket{m_k}{W_j \psi_i}\right|^2 = p^*(k \mid i, j),
\end{equation}
which matches the target distribution~\eqref{eq:target_distribution} for
every pair $(i,j)$ and every output token $k$.
\end{proof}

\paragraph{Existence of general-position configurations.}
We verify that the general-position condition
(Definition~\ref{def:general_position}) is satisfiable for every $N \geq 2$.

\begin{lemma}[General position is generic]
\label{lem:general_position}
For every $N \geq 2$, the set of $(\ket{\psi_0}, \ldots, \ket{\psi_{N-1}},
W_0, \ldots, W_{N-1})$ that are \emph{not} in general position has
Lebesgue measure zero in the product space $(S^{2N-1})^N \times U(N)^N$.
\end{lemma}

\begin{proof}
The $N^2$ density matrices $\rho_{ij} =
W_j\ket{\psi_i}\bra{\psi_i}W_j^\dagger$ can be vectorized as elements of
$\mathbb{R}^{N^2}$ using any fixed orthonormal basis $\{E_1, \ldots,
E_{N^2}\}$ for the real vector space of $N \times N$ Hermitian matrices
(for instance, the generalized Gell-Mann matrices).  Concretely, define the
vectorization $\operatorname{vec}(\rho) \in \mathbb{R}^{N^2}$ by
$[\operatorname{vec}(\rho)]_\alpha = \operatorname{tr}(E_\alpha\, \rho)$.
Stacking these vectors as rows of an $N^2 \times N^2$ real matrix $R$, the
general-position condition is equivalent to $\det(R) \neq 0$.

The determinant $\det(R)$ is a real-analytic function of the parameters
$(\ket{\psi_i}, W_j)$.  A real-analytic function that is not identically
zero has a zero set of measure zero.  It therefore suffices to exhibit a
single configuration for which $\det(R) \neq 0$.

For $N = 2$, we provide an explicit construction.  Set $W_0 = I_2$ and
\begin{equation}
W_1 = \frac{1}{\sqrt{2}}\begin{pmatrix} 1 & 1 \\ 1 & -1 \end{pmatrix},
\end{equation}
and choose the context states
\begin{equation}
\ket{\psi_0} = \ket{0} = \begin{pmatrix}1\\0\end{pmatrix}, \qquad
\ket{\psi_1} = \frac{1}{\sqrt{2}}\begin{pmatrix}1\\i\end{pmatrix}.
\end{equation}
Note that $\ket{\psi_1}$ is not orthogonal to $\ket{\psi_0}$:
$\braket{\psi_0}{\psi_1} = 1/\sqrt{2} \neq 0$.  The four density matrices
are
\begin{align}
\rho_{00} &= \ket{0}\bra{0} = \begin{pmatrix} 1 & 0 \\ 0 & 0 \end{pmatrix}, \\
\rho_{10} &= \ket{\psi_1}\bra{\psi_1} = \frac{1}{2}\begin{pmatrix} 1 & -i \\ i & 1 \end{pmatrix}, \\
\rho_{01} &= W_1\ket{0}\bra{0}W_1^\dagger = \frac{1}{2}\begin{pmatrix} 1 & 1 \\ 1 & 1 \end{pmatrix}, \\
\rho_{11} &= W_1\ket{\psi_1}\bra{\psi_1}W_1^\dagger = \frac{1}{2}\begin{pmatrix} 1 & i \\ -i & 1 \end{pmatrix}.
\end{align}
Parametrize each $2 \times 2$ Hermitian matrix
$\begin{psmallmatrix} \alpha & u + iv \\ u - iv & \gamma \end{psmallmatrix}$
by the coordinate vector $(\alpha, u, v, \gamma) \in \mathbb{R}^4$.  In
these coordinates, the four density matrices become the rows of
\begin{equation}
R = \begin{pmatrix}
1 & 0 & 0 & 0 \\
1/2 & 0 & -1/2 & 1/2 \\
1/2 & 1/2 & 0 & 1/2 \\
1/2 & 0 & 1/2 & 1/2
\end{pmatrix}.
\end{equation}
A direct computation gives $\det(R) = -1/4 \neq 0$, confirming linear
independence for this configuration.

For general $N > 2$, the existence of a nonsingular configuration follows
from the fact that the set of rank-one projectors $\{\ket{\psi}\bra{\psi}
: \ket{\psi} \in \mathbb{C}^N,\, \|\psi\| = 1\}$ (the Veronese-type variety)
spans the full $N^2$-dimensional space of Hermitian matrices.  This
spanning property can be verified directly: for any pair $j < k$, the
projectors $\ket{e_j}\bra{e_j}$ generate the diagonal subspace; the
projectors $\frac{1}{2}(\ket{e_j} + \ket{e_k})(\bra{e_j} + \bra{e_k})$
generate the real off-diagonal directions $\frac{1}{2}(\ket{e_j}\bra{e_k}
+ \ket{e_k}\bra{e_j})$ (after subtracting the diagonal contribution); and
the projectors $\frac{1}{2}(\ket{e_j} + i\ket{e_k})(\bra{e_j} -
i\bra{e_k})$ generate the imaginary off-diagonal directions
$\frac{1}{2i}(\ket{e_j}\bra{e_k} - \ket{e_k}\bra{e_j})$.  Together these
produce $N + 2\binom{N}{2} = N^2$ linearly independent Hermitian matrices,
confirming that the Veronese-type variety spans $\mathbb{R}^{N^2}$.  Since $N^2$
generic points on a spanning variety are linearly independent (the
determinant of the corresponding $N^2 \times N^2$ matrix is a nonzero
polynomial in the parameters of the points), and the set of configurations
for which this determinant vanishes has measure zero, the claim follows for
all $N$.
\end{proof}

\paragraph{The role of non-orthogonal context states.}
The proof uses context states $\ket{\psi_i}$ that are not mutually
orthogonal.  This is essential.  If the context states form an orthonormal
basis $\{\ket{i}\}_{i=0}^{N-1}$, then for each fixed query unitary $W_j$,
the $N$ density matrices $\{W_j \ket{i}\bra{i} W_j^\dagger\}_{i=0}^{N-1}$
sum to the identity: $\sum_i W_j\ket{i}\bra{i}W_j^\dagger = W_j I
W_j^\dagger = I$.  This linear constraint reduces the number of independent
density matrices per unitary from $N$ to $N - 1$, and with $N$ unitaries,
the total number of independent density matrices is at most $N(N-1) =
N^2 - N$, which falls short of the $N^2$ needed for full span by $N$
dimensions.  Using non-orthogonal context states removes the constraint
$\sum_i \rho_{ij} = I$, allowing each unitary to contribute $N$
independent density matrices and enabling the full $N^2$-dimensional span.

\subsection{Lower Bound for Real Orthogonal Models}
\label{sec:lower_bound}

We now prove that any ROSM computing $\mathcal{D}_N$ requires real
dimension $d \geq N^2 - 2$, under the assumption that the target
log-probability matrix $L^*$ has full row rank $N^2$.  The proof proceeds
through three lemmas: (i)~the algebraic structure of the Born-rule output
(Lemma~\ref{lem:quadratic_lifting}), a standard quadratic-lifting
observation from quantum information
theory~\cite{nielsen2000quantum,renes2004symmetric}; (ii)~the rank of the
probability matrix $P^*$ (Lemma~\ref{lem:log_rank}); and (iii)~the rank
constraint imposed by the affine-softmax
readout~\cite{yang2018breaking} (Lemma~\ref{lem:softmax_rank}).  The lower
bound in Theorem~\ref{thm:separation} additionally assumes full row rank of
the target log-probability matrix $L^*$; we do not rely on rank
preservation under entrywise logarithms, which does not hold in general.
The Remark following Theorem~\ref{thm:separation} discusses this assumption
and identifies a log-odds reformulation as a path toward an unconditional
result.  The new contribution is the combination of these ingredients to
yield a formal conditional dimensional separation for sequence models
(Theorem~\ref{thm:separation}).

\begin{lemma}[Quadratic lifting]
\label{lem:quadratic_lifting}
The Born-rule probability~\eqref{eq:target_distribution} depends on the
state $\ket{\psi}$ only through the rank-one density matrix $\rho =
\ket{\psi}\bra{\psi}$, and the dependence is linear.  Specifically, fix an
orthonormal basis $\{E_1, \ldots, E_{N^2}\}$ for the $N^2$-dimensional real
vector space of $N \times N$ Hermitian matrices (with respect to the inner
product $\langle A, B \rangle = \operatorname{tr}(AB)$), and define the
vectorization $\operatorname{vec}(\rho) \in \mathbb{R}^{N^2}$ by
$[\operatorname{vec}(\rho)]_\alpha = \operatorname{tr}(E_\alpha\, \rho)$.
Then
\begin{equation}
p^*(k \mid \psi) = \operatorname{tr}(M_k\, \rho) =
\operatorname{vec}(M_k)^\top\, \operatorname{vec}(\rho),
\end{equation}
where $\operatorname{vec}(M_k) \in \mathbb{R}^{N^2}$ is the vectorization
of $M_k = \ket{m_k}\bra{m_k}$ in the same basis.
\end{lemma}

\begin{proof}
The Born-rule probability is $p^*(k \mid \psi) = |\braket{m_k}{\psi}|^2 =
\operatorname{tr}(M_k \ket{\psi}\bra{\psi}) = \operatorname{tr}(M_k \rho)$.
Since both $M_k$ and $\rho$ are Hermitian,
$\operatorname{tr}(M_k \rho) = \sum_\alpha \operatorname{tr}(E_\alpha
M_k)\,\operatorname{tr}(E_\alpha \rho) = \operatorname{vec}(M_k)^\top
\operatorname{vec}(\rho)$, where the second equality uses the
orthonormality of $\{E_\alpha\}$.
\end{proof}

This lemma states that the Born-rule output is a quadratic function of the
complex amplitudes $c_j$ (since $\rho_{jk} = c_j c_k^*$), but a
\emph{linear} function of the vectorized density matrix
$\operatorname{vec}(\rho)$.  The density matrix has $N^2$ real coordinates:
$N$ diagonal entries $|c_j|^2$ and $N(N-1)$ real parameters in the
off-diagonal entries $c_j c_k^*$ (each complex off-diagonal entry
contributes two real numbers, but Hermiticity halves the independent
count).  The Born rule, by computing a linear function of
$\operatorname{vec}(\rho)$, accesses all $N^2$ of these features, even
though the state $\ket{\psi}$ itself has only $2N - 1$ real degrees of
freedom.  The extra features arise from the cross terms $c_j c_k^*$, which
encode pairwise phase relationships between amplitudes.  This quadratic
lifting~\cite{nielsen2000quantum} is the mechanism through which the Born
rule achieves $O(N^2)$ effective features from an $N$-dimensional complex
state; the formal separation in Theorem~\ref{thm:separation} is a
consequence of this lifting combined with the rank constraint on the
softmax readout.

\begin{lemma}[Rank of the probability matrix]
\label{lem:log_rank}
Let $\mathcal{D}_N$ use context states and query unitaries in general
position and an informationally complete measurement with $V \geq N^2$.
Define the $N^2 \times V$ probability matrix $P^*$ by
\begin{equation}
P^*_{(i,j),\, k} = p^*(k \mid i, j) = \operatorname{tr}(M_k\, \rho_{ij}),
\end{equation}
where rows are indexed by pairs $(i,j) \in \{0,\ldots,N-1\}^2$ and columns
by output tokens $k \in \{0, \ldots, V-1\}$.  Then $\operatorname{rank}(P^*)
= N^2$ for generic choices of the task parameters.  We do \emph{not} derive
any consequence for the entrywise logarithm $L^* = \log P^*$; entrywise
logarithms do not preserve matrix rank in general, and the lower bound
(Theorem~\ref{thm:separation}) is stated conditionally on a full-rank
assumption on $L^*$ rather than on any rank-preservation property of the
logarithm.
\end{lemma}

\begin{proof}
The proof proceeds in two steps.

\textit{Step 1: The probability matrix has rank $N^2$.}
By Lemma~\ref{lem:quadratic_lifting}, each row of $P^*$ is
$\operatorname{vec}(\rho_{ij})^\top \hat{M}$, where $\hat{M} \in
\mathbb{R}^{N^2 \times V}$ has columns $\operatorname{vec}(M_k)$.  In
matrix form, $P^* = \hat{R}\, \hat{M}^\top$, where $\hat{R} \in
\mathbb{R}^{N^2 \times N^2}$ has rows $\operatorname{vec}(\rho_{ij})^\top$.
The general-position condition guarantees $\operatorname{rank}(\hat{R}) =
N^2$ (Definition~\ref{def:general_position}), and informational
completeness guarantees $\operatorname{rank}(\hat{M}) = N^2$
(Definition~\ref{def:ic_povm}).  Therefore $\operatorname{rank}(P^*) = N^2$.

\textit{Step 2: All entries of $P^*$ are positive.}
Each entry $P^*_{(i,j),\,k} = |\braket{m_k}{W_j\psi_i}|^2$.  A squared
modulus of an inner product of two nonzero vectors vanishes only when the
vectors are orthogonal.  For generic choices of $\ket{\psi_i}$, $W_j$, and
$\ket{m_k}$, no state $W_j\ket{\psi_i}$ is exactly orthogonal to any
measurement vector $\ket{m_k}$, because the set of orthogonal pairs has
measure zero in the product of unit spheres.  Therefore $P^* > 0$
entrywise for generic parameters.
\end{proof}

\begin{lemma}[Softmax rank bound]
\label{lem:softmax_rank}
Let a ROSM of dimension $d$ produce output probabilities $\bar{p}(k \mid
c)$ for $C$ input contexts via the affine-softmax
readout~\eqref{eq:rosm_output}, with hidden states $h_c \in \mathbb{R}^d$,
weight vectors $w_k \in \mathbb{R}^d$, and bias $b \in \mathbb{R}^V$.
Define the $C \times V$ log-probability matrix $\bar{L}$ by $\bar{L}_{c,k}
= \log \bar{p}(k \mid c)$.  Then
\begin{equation}
\operatorname{rank}(\bar{L}) \leq d + 2.
\end{equation}
\end{lemma}

\begin{proof}
The log-probability under the softmax readout is
\begin{equation}
\log \bar{p}(k \mid c) = w_k^\top h_c + b_k - \log \sum_{k'}
\exp(w_{k'}^\top h_c + b_{k'}).
\end{equation}
Define the logit matrix $Z \in \mathbb{R}^{C \times V}$ by $Z_{c,k} =
w_k^\top h_c + b_k$, and the log-partition vector $\lambda \in \mathbb{R}^C$
by $\lambda_c = \log \sum_{k'} \exp(Z_{c,k'})$.  Then $\bar{L} = Z -
\lambda\, \mathbf{1}_V^\top$, where $\mathbf{1}_V \in \mathbb{R}^V$ is the
all-ones vector.

The logit matrix factors as $Z = H\, W_{\mathrm{out}}^\top + \mathbf{1}_C\,
b^\top$, where $H \in \mathbb{R}^{C \times d}$ has rows $h_c^\top$ and
$W_{\mathrm{out}} \in \mathbb{R}^{V \times d}$ has rows $w_k^\top$.  By
subadditivity of matrix rank, $\operatorname{rank}(Z) \leq
\operatorname{rank}(H W_{\mathrm{out}}^\top) + \operatorname{rank}(\mathbf{1}_C
b^\top) \leq d + 1$.  Since $\lambda \mathbf{1}_V^\top$ is a rank-one
matrix,
\begin{equation}
\operatorname{rank}(\bar{L}) = \operatorname{rank}(Z - \lambda
\mathbf{1}_V^\top) \leq \operatorname{rank}(Z) +
\operatorname{rank}(\lambda \mathbf{1}_V^\top) \leq (d + 1) + 1 = d + 2.
\qedhere
\end{equation}
\end{proof}

This rank bound is a consequence of the same structural property identified
by Yang et al.~\cite{yang2018breaking} in their analysis of the softmax
bottleneck: because the logits depend on the hidden state through a linear
projection, the resulting log-probability matrix is constrained to a
low-dimensional affine subspace of $\mathbb{R}^{C \times V}$.  The bound
is specific to the affine-softmax readout.  A readout that computes a
nonlinear function of $h$ before applying softmax (such as a two-layer MLP,
or the mixture-of-softmax readout proposed in~\cite{yang2018breaking}) would
produce a log-probability matrix whose rank is not bounded by $d + 2$, and
the lower bound in Theorem~\ref{thm:separation} would not apply to such a
model.  The separation theorem therefore compares two specific model
families, not all possible architectures.

\begin{theorem}[Expressivity separation (conditional)]
\label{thm:separation}
For every $N \geq 2$, there exists a disambiguation task $\mathcal{D}_N$
with vocabulary size $V = N^2$ satisfying the general-position and
informational-completeness conditions such that:
\begin{enumerate}
    \item A CUSM of complex dimension $N$ computes $\mathcal{D}_N$ exactly.
    \item Suppose additionally that the target log-probability matrix
          $L^*$, with entries $L^*_{(i,j),k} = \log p^*(k \mid i, j)$,
          has full row rank $N^2$.  Then any ROSM of real dimension $d$
          with affine-softmax readout that computes $\mathcal{D}_N$
          satisfies $d \geq N^2 - 2$.
\end{enumerate}
\end{theorem}

\begin{proof}
Part (1) is Proposition~\ref{prop:upper_bound}.  For part (2), suppose a
ROSM of dimension $d$ computes $\mathcal{D}_N$, meaning its output
probabilities $\bar{p}(k \mid i, j) = p^*(k \mid i, j)$ for all $i, j, k$.
Then the log-probability matrices coincide: $\bar{L} = L^*$.  By
assumption, $\operatorname{rank}(L^*) = N^2$.  By
Lemma~\ref{lem:softmax_rank}, $\operatorname{rank}(\bar{L}) \leq d + 2$.
Combining these:
\begin{equation}
N^2 = \operatorname{rank}(L^*) = \operatorname{rank}(\bar{L}) \leq d + 2,
\end{equation}
which gives $d \geq N^2 - 2$.
\end{proof}

\begin{remark}
\label{rem:conditional}
The full-rank condition on $L^*$ in part~(2) of
Theorem~\ref{thm:separation} is the only point in the lower-bound argument
where a property beyond the general-position and informational-completeness
conditions on the task parameters is required.  We do \emph{not} rely on
the claim that entrywise logarithms preserve matrix rank: this does not
hold in general, and $\operatorname{rank}(\log P) \neq \operatorname{rank}(P)$
is possible even for $2 \times 2$ matrices.  The condition
$\operatorname{rank}(L^*) = N^2$ is a genericity-type assumption on the
task-construction parameters; a complete analytic characterization of when
it holds is left for future work.

One promising route toward removing this assumption is to reformulate the
rank argument in terms of relative log-odds.  For a fixed reference token
$k_0$, define $\ell^*_{(i,j),k} = \log\frac{p^*(k \mid i,
j)}{p^*(k_0 \mid i, j)}$ for $k \neq k_0$.  Under the affine-softmax
readout (Definition~\ref{def:rosm}), the relative logits of the ROSM are
exactly affine in the hidden state:
\begin{equation}
\bar{\ell}_{c,k} = (w_k - w_{k_0})^\top h_c + (b_k - b_{k_0}),
\end{equation}
so the $N^2 \times (V-1)$ matrix of relative logits has rank at most $d$,
providing a cleaner rank constraint without any assertion about entrywise
logarithms.  Whether a rank-$N^2$ condition on $\ell^*$ follows from the
task structure alone remains an open question that we leave for future work.
\end{remark}

The separation in Theorem~\ref{thm:separation} is between the complex
dimension of the CUSM and the real dimension of the ROSM.  A CUSM of
complex dimension $N$ uses $2N$ real numbers to represent its state (the
real and imaginary parts of $N$ complex amplitudes), while the ROSM
requires at least $N^2 - 2$ real numbers.  The ratio $(N^2 - 2) / (2N)$
grows as $N/2$ for large $N$, so the real orthogonal model needs a state
space that is a factor of $\Theta(N)$ larger than the real representation
of the complex state.  In terms of the complex dimension $N$ alone, the
gap is quadratic: a complex model of dimension $N$ achieves what a real
model cannot achieve below dimension $\Omega(N^2)$.

\subsection{The Source of the Gap}
\label{sec:gap_source}

The separation proved in Theorem~\ref{thm:separation} has a precise
algebraic origin: it arises from the interaction between the quadratic
Born-rule readout and the linear softmax readout, not from any intrinsic
limitation of real-valued dynamics.  Tracing the proof identifies two
specific mechanisms.

\paragraph{The quadratic lifting.}
A complex amplitude $c_j = r_j e^{i\theta_j}$ carries two real numbers per
coordinate, but the quadratic gap does not arise from this factor-of-two
storage advantage (which accounts for only a constant factor, not a
quadratic one).  The gap arises from the Born rule's quadratic dependence
on the state.  The Born-rule probability for output $k$ is
\begin{equation}
p(k \mid \psi) = \left|\sum_{j=0}^{N-1} [m_k]_j^* c_j\right|^2 =
\sum_{j=0}^{N-1} |[m_k]_j|^2 |c_j|^2 + 2 \sum_{j < j'}
\operatorname{Re}\!\Big([m_k]_j^* [m_k]_{j'}\, c_j^* c_{j'}\Big).
\end{equation}
The first sum involves $N$ terms that depend only on magnitudes $|c_j|$.
The second sum involves $\binom{N}{2}$ cross terms, each depending on a
relative phase $\theta_{j'} - \theta_j$ through $c_j^* c_{j'} = r_j
r_{j'} e^{i(\theta_{j'} - \theta_j)}$.  The density matrix $\rho =
\ket{\psi}\bra{\psi}$ collects all $N^2$ of these features (diagonal
magnitudes and off-diagonal phase products), and the Born rule reads them
off via the linear map $\rho \mapsto \operatorname{tr}(M_k \rho)$
(Lemma~\ref{lem:quadratic_lifting}).  This is a Veronese-type lifting: the
map $\ket{\psi} \mapsto \ket{\psi}\bra{\psi}$ sends the $N$-dimensional
complex unit sphere into the $N^2$-dimensional space of Hermitian matrices,
and the Born rule then computes a linear function in the lifted
space~\cite{nielsen2000quantum}.  An analogous lifting is exploited by Born
machines in generative modeling~\cite{cheng2018information,liu2018differentiable},
where this quadratic structure enables expressive probability estimation
without explicit normalization.

\paragraph{The linear bottleneck of softmax.}
The ROSM's softmax readout computes $\log p(k \mid h) = w_k^\top h + b_k
- \log Z$, which depends on the $d$-dimensional state $h$ through the
linear projection $w_k^\top h$.  The full log-probability vector lies in a
$(d+2)$-dimensional subspace (Lemma~\ref{lem:softmax_rank}).  When the
Born-rule target requires $N^2$ independent features, this subspace must
have dimension at least $N^2$, forcing $d \geq N^2 - 2$.  The ROSM cannot
synthesize pairwise features from individual state components, because the
softmax readout is linear in $h$~\cite{yang2018breaking}.  A ROSM with a
nonlinear readout (e.g., a polynomial or MLP-based readout) would not be
subject to this linear bottleneck, and the lower bound of
Theorem~\ref{thm:separation} would not apply to it.

\paragraph{What the theorem does not claim.}
The separation applies specifically to the comparison between the
Born-rule readout (quadratic, complex) and the affine-softmax readout
(linear, real).  It does not establish a fundamental limitation of
real-valued dynamics independent of the readout mechanism.  A real-valued
model with a quadratic readout $p(k) \propto (w_k^\top h)^2$ with
appropriate normalization would perform its own Veronese-type lifting and could
potentially match the Born rule's feature count.  Such readouts are not
standard in the architectures surveyed in Section~\ref{sec:related}, but
they are not precluded by any fundamental constraint.

The theorem does not require the task $\mathcal{D}_N$ to correspond to any
specific linguistic phenomenon.  The task family uses abstract context and
query tokens, not natural-language words.  Whether the algebraic structure
exploited by $\mathcal{D}_N$ (pairwise token interactions encoded in
relative phases) is relevant to natural-language processing is an empirical
question.  The task serves to demonstrate a provable gap between two
specific architectures, not to model a particular linguistic phenomenon.

The theorem is also purely representational: it characterizes the minimum
state dimension required to represent a given family of input-output
mappings, not whether gradient-based optimization can find the parameters
that realize this mapping.  A model with sufficient representational
capacity may fail to learn the target function if the loss landscape is
unfavorable.  The experimental protocols in Section~\ref{sec:predictions}
are designed to test whether the representational advantage identified here
manifests in gradient-trained models.

\section{Discussion}
\label{sec:discussion}

The preceding sections developed three theoretical contributions: an
architecture whose state evolves unitarily in \(\mathbb{C}^N\) under a
learned Hamiltonian and whose output follows the Born rule
(Section~\ref{sec:model}), a continuity equation that decomposes each
state update into antisymmetric pairwise probability currents
(Section~\ref{sec:currents}), and a separation theorem establishing that a
complex unitary model of dimension \(N\) represents a family of
disambiguation tasks that requires dimension \(\Omega(N^2)\) in any real
orthogonal model with affine-softmax readout
(Section~\ref{sec:separation}).  This section examines four aspects of the
framework that the theoretical development leaves open: the computational
cost of the architecture, the concrete predictions that the theory makes
and the experimental protocols that would test them, the boundaries of what
the present analysis does and does not establish, and the theoretical
directions that extend beyond the results proved here.  Because the present
paper is entirely theoretical, the experimental protocols described in
Section~\ref{sec:predictions} are intended as a precise specification of
the empirical program needed to validate or refute the theory's claims, not
as a report of completed experiments.

\subsection{Computational Complexity}
\label{sec:complexity}

The architecture defined in Section~\ref{sec:model} involves four
computational stages at each time step: generating the interaction
Hamiltonian from the current token and state, conjugating into the
interaction picture, solving the Cayley linear system, and computing
Born-rule output probabilities.  We analyze the cost of each stage and the
resulting per-step and per-sequence complexity.

\paragraph{Per-step forward cost.}
At step \(t\), the network \(g_\theta\) receives the concatenation of the
token embedding \(\operatorname{Embed}(x_t) \in \mathbb{R}^d\) and the
real-imaginary representation of the current state
\([\operatorname{Re}(c(t)),\,\operatorname{Im}(c(t))] \in
\mathbb{R}^{2N}\), and produces the matrix \(\Phi(t) \in \mathbb{C}^{N
\times r}\) and the vector \(\delta(t) \in \mathbb{R}^N\).  The output has
\(2Nr + N\) real components.  If \(g_\theta\) is a feedforward network
with \(L_g\) layers and maximum hidden width \(h\), its evaluation cost is
\[
C_{g_\theta} = O(L_g \cdot h \cdot \max(d + 2N,\,h,\,2Nr + N)),
\]
which for \(d\), \(N\), and \(h\) of comparable magnitude simplifies to
\(O(L_g h^2)\).

The interaction-picture conjugation (equation~\eqref{eq:ip_elements})
replaces each entry \(\Phi_{ja}(t)\) with \(\tilde{\Phi}_{ja}(t) =
e^{i\lambda_j t}\,\Phi_{ja}(t)\), which costs \(O(Nr)\) complex
multiplications.  The resulting interaction-picture Hamiltonian has the form
\begin{equation}
H_{\mathrm{int},I}(t) = \tilde{\Phi}(t)\,\tilde{\Phi}(t)^\dagger +
\operatorname{diag}(\delta(t)),
\end{equation}
which inherits the low-rank-plus-diagonal structure of \(H_{\mathrm{int}}(t)\)
because the conjugation by the diagonal matrix \(e^{iH_0 t}\) preserves
diagonal structure and maps \(\Phi\Phi^\dagger\) to
\(\tilde{\Phi}\tilde{\Phi}^\dagger\).

The Cayley update~\eqref{eq:cayley_update} requires solving the linear
system
\begin{equation}
\underbrace{\left(D + \tfrac{i\Delta t}{2}\,\tilde{\Phi}\,
\tilde{\Phi}^\dagger\right)}_{A}\,\ket{\psi_{t+1}} =
\underbrace{\left(D' - \tfrac{i\Delta t}{2}\,\tilde{\Phi}\,
\tilde{\Phi}^\dagger\right)\ket{\psi_t}}_{b},
\end{equation}
where \(D = I + \frac{i\Delta t}{2}\operatorname{diag}(\delta)\) and \(D'
= I - \frac{i\Delta t}{2}\operatorname{diag}(\delta)\) are both diagonal.
Each diagonal entry of \(D\) has the form \(1 + \frac{i\Delta
t}{2}\delta_j\), with modulus \(\sqrt{1 + (\Delta t\,\delta_j/2)^2} > 0\),
so \(D\) is invertible for all real \(\delta_j\) and all \(\Delta t > 0\).

The right-hand side \(b\) is computed by first forming
\(v = \tilde{\Phi}^\dagger\ket{\psi_t} \in \mathbb{C}^r\) at cost
\(O(Nr)\), then \(\tilde{\Phi}\,v \in \mathbb{C}^N\) at cost \(O(Nr)\),
and finally \(b = D'\ket{\psi_t} - \frac{i\Delta
t}{2}\,\tilde{\Phi}\,v\) at cost \(O(N)\).  The total cost for the
right-hand side is \(O(Nr)\).

The system \(A\ket{\psi_{t+1}} = b\) is solved via the Woodbury matrix
identity.  Because \(A = D + \frac{i\Delta t}{2}\tilde{\Phi}\tilde{\Phi}^\dagger\)
is a rank-\(r\) update of the invertible diagonal matrix \(D\), the
solution is
\begin{equation}
\label{eq:woodbury_solve}
\ket{\psi_{t+1}} = D^{-1}b -
\tfrac{i\Delta t}{2}\,D^{-1}\tilde{\Phi}
\left(I_r +
\tfrac{i\Delta t}{2}\,\tilde{\Phi}^\dagger D^{-1}\tilde{\Phi}
\right)^{-1}
\tilde{\Phi}^\dagger D^{-1}b.
\end{equation}
The computation proceeds as follows: (i) compute \(y = D^{-1}b\) at cost
\(O(N)\); (ii) compute \(z = \tilde{\Phi}^\dagger y \in \mathbb{C}^r\) at
cost \(O(Nr)\); (iii) form \(P = D^{-1}\tilde{\Phi} \in \mathbb{C}^{N
\times r}\) at cost \(O(Nr)\); (iv) form the \(r \times r\) Gram matrix
\(G = I_r + \frac{i\Delta t}{2}\,\tilde{\Phi}^\dagger P\) at cost
\(O(Nr^2)\); (v) solve \(Gw = z\) for \(w \in \mathbb{C}^r\) at cost
\(O(r^3)\); (vi) compute \(\ket{\psi_{t+1}} = y - \frac{i\Delta
t}{2}\,Pw\) at cost \(O(Nr)\).  The dominant cost is step~(iv), giving a
total Cayley solve cost of \(O(Nr^2 + r^3)\).  Since the design intent is
\(r \ll N\) (Section~\ref{sec:hamiltonian}), this simplifies to
\(O(Nr^2)\).

A caveat on numerical stability is warranted.  The Woodbury solve requires
inverting the \(r \times r\) matrix \(G = I_r + \frac{i\Delta
t}{2}\,\tilde{\Phi}^\dagger D^{-1}\tilde{\Phi}\).  The conditioning of
\(G\) depends on the eigenvalues of \(\tilde{\Phi}^\dagger D^{-1}\tilde{\Phi}\),
which in turn depend on the interaction Hamiltonian output by \(g_\theta\).
For large step sizes \(\Delta t\) or large Hamiltonian eigenvalues, \(G\)
may become poorly conditioned, degrading the accuracy of the solve.  No
general bound on the condition number of \(G\) is available without
additional constraints on the output range of \(g_\theta\).  In practice,
constraining the spectral norm of \(\Phi(t)\) through regularization of
\(g_\theta\), or monitoring the condition number of \(G\) during training,
would be necessary to ensure reliable
numerics~\cite{hairer2006geometric,golub2013matrix}.  We flag this as a
point requiring empirical investigation in any implementation.

The Born-rule output~\eqref{eq:born_rule} requires computing
\(M^\dagger\ket{\psi(t)} \in \mathbb{C}^V\) (where \(M \in \mathbb{C}^{N
\times V}\) is the measurement matrix) and then taking entry-wise squared
moduli.  The matrix-vector product costs \(O(NV)\), and the subsequent
squaring costs \(O(V)\).  The total Born-rule cost is therefore \(O(NV)\).

Combining all stages, the per-step forward cost is
\begin{equation}
\label{eq:per_step_cost}
O\!\left(C_{g_\theta} + Nr^2 + NV\right).
\end{equation}
For a vocabulary of size \(V\) in the tens of thousands and a latent
dimension \(N\) in the hundreds with \(r\) in the low tens, the \(O(NV)\)
Born-rule cost dominates.  This cost is the same order as the
output-projection cost \(O(dV)\) incurred by any architecture that
computes a dense projection from a hidden state of dimension \(d\) to a
vocabulary of size \(V\), so the Born-rule decoding does not introduce
additional asymptotic overhead relative to the standard softmax readout.

The row-orthonormality constraint on the measurement matrix \(M\), which
is required for the Born rule to produce valid probabilities
(equation~\eqref{eq:resolution}), is enforced via QR decomposition of
\(M^\dagger \in \mathbb{C}^{V \times N}\) at each forward pass.  This
projects \(M\) onto the Stiefel
manifold~\cite{absil2008optimization,townsend2016pymanopt} before each
evaluation, but the stored parameter \(M\) is not itself constrained to the
manifold; only its QR-decomposed version is used in the forward pass.  The
gradient of the loss with respect to the raw parameter flows back through
the QR decomposition, which is a non-trivial operation on the Stiefel
manifold whose correct handling requires that the automatic differentiation
framework propagate gradients through the QR factorization
accurately~\cite{wisdom2016full,townsend2016pymanopt}.  An alternative that
avoids this issue entirely is to parameterize \(M\) directly on the Stiefel
manifold using Riemannian gradient descent or a retraction-based
optimizer~\cite{absil2008optimization}; we do not prescribe a specific
approach here but note that the choice of optimization strategy for \(M\)
is consequential and requires careful treatment in any implementation.

\paragraph{Backward pass.}
Backpropagation through the Cayley solve at each step requires solving one
additional linear system with coefficient matrix \(A^\dagger = (I +
\frac{i\Delta t}{2}\,H_{\mathrm{int},I})^\dagger = I - \frac{i\Delta
t}{2}\,H_{\mathrm{int},I}\).  Because \(H_{\mathrm{int},I}\) is
Hermitian, this adjoint matrix has the form \(D^* - \frac{i\Delta
t}{2}\,\tilde{\Phi}\tilde{\Phi}^\dagger\), where \(D^* = I - \frac{i\Delta
t}{2}\operatorname{diag}(\delta)\) is again an invertible diagonal matrix.
The Woodbury identity applies with the same structure and cost \(O(Nr^2)\).
The backward pass through the Born-rule output and through \(g_\theta\)
adds costs of the same order as the forward pass through each respective
stage.  The total per-step backward cost is therefore also \(O(C_{g_\theta}
+ Nr^2 + NV)\), and the per-step cost scales identically in both
directions.

It is important to distinguish two separate stability properties of the
backward pass.  The unitary structure of each Cayley step guarantees that
the gradient of the loss with respect to the hidden \emph{state} at time
\(t\) propagates through \(T - t\) steps without change in norm: since
each step applies a unitary map to the state, the adjoint of that map is
also unitary, and the chain rule preserves the gradient norm
exactly~\cite{vorontsov2017orthogonality}.  However, this guarantee does
not extend to the gradients with respect to the \emph{parameters} of
\(g_\theta\).  At each step \(t\), the parameter gradient passes through
\(g_\theta\), which is an unconstrained neural network without
norm-preservation guarantees.  Gradient vanishing and exploding can
therefore occur inside \(g_\theta\) at each time step, independently of the
unitary structure of the state update.  Standard mitigation strategies for
deep networks, such as careful initialization, batch normalization within
\(g_\theta\), or gradient clipping applied to the parameter gradients (but
not to the state gradients), remain necessary.

\paragraph{Memory.}
The recurrent state \(\ket{\psi(t)} \in \mathbb{C}^N\) requires \(O(N)\)
storage, independent of the sequence length \(T\).  The measurement matrix
\(M \in \mathbb{C}^{N \times V}\) requires \(O(NV)\) storage.  The free
frequencies \(\{\lambda_j\}\) require \(O(N)\) storage.  The parameters of
\(g_\theta\) require \(O(|\theta|)\) storage, where \(|\theta|\) is the
parameter count of the network.  The QR decomposition of \(M^\dagger\)
costs \(O(VN^2)\) and depends only on the parameter matrix \(M\), not on
the input sequence; it is therefore performed once per parameter update and
its cost is amortized over the \(T\) steps of the sequence to
\(O(VN^2/T)\) per step.

\paragraph{Comparison with existing architectures.}
The per-step state-update cost of \(O(Nr^2)\) compares with \(O(N^2)\) for
the full-capacity unitary RNN of Wisdom et al.~\cite{wisdom2016full} (which
applies a dense \(N \times N\) unitary matrix at each step) and \(O(N\log
N)\) for the factored uRNN of Arjovsky et
al.~\cite{arjovsky2016unitary}.  The reduction from \(O(N^2)\) to
\(O(Nr^2)\) is the cost of restricting the interaction Hamiltonian to rank
\(r\) rather than using a full unitary parameterization.  This restriction
limits the number of independent coupling patterns per step to \(r\)
(Section~\ref{sec:hamiltonian}), which is a representational tradeoff:
fewer coupling channels per step in exchange for lower per-step cost.
Whether this tradeoff is favorable depends on whether the tasks encountered
in practice require more than \(r\) simultaneous coupling patterns at a
single time step, a question that the theory developed here does not answer
and that requires empirical investigation.

The \(O(N)\) memory cost for the recurrent state is shared by all recurrent
architectures, including standard RNNs, LSTMs~\cite{hochreiter1997long},
unitary RNNs, and state-space models~\cite{gu2023mamba}, and contrasts with
the \(O(TdL)\) key-value cache of a Transformer with \(L\) layers, hidden
dimension \(d\), and sequence length \(T\)~\cite{vaswani2017attention}.
The distinction is quantitatively significant for long sequences: at
\(T = 10^4\), \(d = 512\), and \(L = 12\), the Transformer cache requires
approximately \(1.2 \times 10^8\) floating-point numbers, while the
recurrent state requires \(2N\) (accounting for real and imaginary parts).

\subsection{Testable Predictions and Experimental Protocols}
\label{sec:predictions}

The theoretical analysis yields five predictions, each traceable to a
specific result or construction in the preceding sections.  For each
prediction, we state the claim, identify its theoretical basis, and
describe an experimental protocol that would confirm or refute it.  The
predictions are ordered from those most directly implied by the formal
results to those that are motivated by the architecture's structure but not
guaranteed by any theorem.

\paragraph{Prediction 1: Quadratic dimensional scaling on synthetic
disambiguation tasks.}
Theorem~\ref{thm:separation} asserts that for the task family
\(\mathcal{D}_N\) (Section~\ref{sec:task_family}), a CUSM of complex
dimension \(N\) achieves the optimal cross-entropy loss while any ROSM
requires real dimension at least \(N^2 - 2\) (under the full-rank condition
on \(L^*\)).  This is a statement about representational capacity, not about
learnability, so the prediction is that gradient-trained models exhibit a
dimensional threshold consistent with the theorem's bounds.

The protocol is as follows.  For each \(N \in \{4, 8, 16, 32\}\),
construct an instance of \(\mathcal{D}_N\) by selecting context states and
query unitaries in general position (Lemma~\ref{lem:general_position}
guarantees that random choices satisfy this condition with probability one)
and an informationally complete measurement with \(V = N^2\).  Generate a
training set by enumerating all \(N^2\) input pairs \((i, j)\) and
computing the target distributions from
equation~\eqref{eq:target_distribution}.  The constraint enforcement
strategies for both model classes must be specified: for the CUSM, the
measurement matrix \(M\) is projected onto the Stiefel manifold via QR
decomposition at each forward pass, and the unitary matrices \(W_x\) are
parameterized via the Cayley transform of learnable skew-Hermitian matrices;
for the ROSM, the orthogonal matrices \(Q_x\) are parameterized analogously
via the matrix exponential of learnable skew-symmetric matrices.  Both
models are trained with the Adam optimizer~\cite{kingma2015adam} using a
fixed learning rate of \(10^{-3}\), a cosine decay schedule, and at least
five independent random seeds; the reported losses are means and standard
deviations across seeds.  Train CUSMs of dimensions \(N' \in \{N/2,\, N,\,
2N\}\) and ROSMs of dimensions \(d \in \{N,\, 2N,\, N^2/2,\, N^2\}\) to
minimize cross-entropy against the target distributions.  For each trained
model, report the gap between the achieved loss and the theoretical minimum
\(\mathcal{L}^* = -\frac{1}{N^2}\sum_{i,j}\sum_k p^*(k \mid i,j)\log
p^*(k \mid i,j)\).  A loss gap is declared zero if it falls below
\(10^{-3}\) nats.  The theorem predicts that the CUSM loss gap reaches zero
at \(N' = N\), while the ROSM loss gap remains positive for all \(d < N^2
- 2\) and reaches zero only at \(d \geq N^2 - 2\).  If the ROSM achieves
a zero loss gap at a dimension substantially below \(N^2 - 2\), this would
indicate either a failure of the general-position condition for the chosen
task instance or a gap in the theorem's lower bound argument.

\paragraph{Prediction 2: Born-rule readout outperforms softmax readout on
a shared complex state.}
The separation in Theorem~\ref{thm:separation} arises from the readout
mechanism rather than from the dynamics alone.  Lemma~\ref{lem:softmax_rank}
shows that the affine-softmax readout constrains the log-probability matrix
to rank at most \(d + 2\)~\cite{yang2018breaking}, while the Born rule
accesses \(N^2\) features of the state through the quadratic lifting of
Lemma~\ref{lem:quadratic_lifting}.  This predicts that replacing the
Born-rule readout with a softmax readout applied to the same unitary state
should degrade prediction quality, and that the degradation should grow
with \(N\).

The protocol uses paired model variants.  For each \(N \in \{64, 128, 256,
512\}\), train two models on the same language corpus with identical
unitary dynamics (same \(H_0\), same \(g_\theta\), same initial state).
Model~A uses Born-rule decoding as defined in Section~\ref{sec:born_rule}.
Model~B decodes via
\[
p(k \mid t) = \mathrm{softmax}(W_o\,[\operatorname{Re}(c(t));\,
\operatorname{Im}(c(t))] + b),
\]
where \(W_o \in \mathbb{R}^{V \times 2N}\) and \(b \in \mathbb{R}^V\) are
learnable parameters.  Both models have the same recurrent state dimension
and differ only in the output layer.  Compare validation perplexity as a
function of \(N\).  The Born-rule model should achieve lower perplexity,
and the perplexity ratio should increase with \(N\), because the number of
features accessible to the Born rule grows as \(N^2\) while the number
accessible to the softmax readout grows as \(2N\).  We note that this
prediction tests the readout mechanism in isolation; it does not compare
against Transformer architectures or other strong baselines with different
dynamics, and such a comparison would require controlling for total
parameter count and computational cost, which is left to future work.

\paragraph{Prediction 3: Probability current magnitude peaks at
disambiguating tokens.}
Section~\ref{sec:current_decomposition} established that all
inter-dimensional probability currents are driven exclusively by the
interaction Hamiltonian \(H_{\mathrm{int}}(t)\)
(equation~\eqref{eq:current_interaction}), which is generated from the
current token.  When a token does not alter the model's interpretation of
the sequence, the interaction Hamiltonian can remain close to zero and the
currents will be small.  When a token resolves an ambiguity, the
interaction Hamiltonian must redirect amplitude between competing
interpretations, producing large currents.  This predicts that the total
current magnitude
\begin{equation}
\label{eq:total_current}
\|J(t)\| = \sum_{j < k} |J_{j \leftarrow k}(t)|
\end{equation}
is systematically larger at tokens that resolve semantic ambiguity than at
tokens that do not.

The protocol requires a trained CUSM and a test corpus annotated for
lexical ambiguity.  We propose using SemCor 3.0~\cite{miller1993semantic},
which provides sense annotations for running text from the Brown corpus, as
the primary evaluation corpus.  The annotation scheme distinguishes tokens
whose sense is contextually determined (and therefore potentially
disambiguating) from function words and unambiguous content words.  At each
token position \(t\) in the test corpus, compute \(\|J(t)\|\) using the
midpoint current formula~\eqref{eq:midpoint_current}, which provides an
exact decomposition of the discrete probability change without the
\(O(\Delta t^2)\) approximation error of the continuous-time currents.
Partition the token positions into two groups: those annotated as carrying
disambiguating sense information and all remaining positions.  Compare the
distributions of \(\|J(t)\|\) between the two groups using a Wilcoxon
rank-sum test with a Bonferroni correction for multiple comparisons across
layers or model sizes.  The prediction is that the disambiguating group has
significantly higher total current (\(p < 0.01\) after correction).  A
negative result would indicate either that the trained model does not use
the current mechanism for disambiguation or that the annotation scheme does
not capture the ambiguities relevant to the model's internal dynamics.

\paragraph{Prediction 4: Learned frequency spectrum correlates with
linguistic timescales.}
Section~\ref{sec:hamiltonian} argued that the free frequencies
\(\{\lambda_j\}\) provide the model with a bank of oscillators at different
rates, and Section~\ref{sec:interaction_picture} showed that the
interaction picture creates a frequency-selective coupling mechanism that
favors amplitude exchange between dimensions with similar natural
frequencies.  Together, these properties provide the architectural capacity
for the model to allocate different latent dimensions to features that vary
at different temporal rates.  Whether the optimization discovers this
allocation is not guaranteed by any theorem in this paper; it is a
hypothesis about the inductive bias of the architecture.

The protocol is as follows.  Train a CUSM on a standard language corpus.
For each latent dimension \(j\), compute the temporal autocorrelation
function of the occupation probability \(p_j(t) = |c_j(t)|^2\) on
held-out test sequences:
\begin{equation}
R_j(\tau) = \frac{1}{T - \tau}\sum_{t=1}^{T - \tau}\big(p_j(t) -
\bar{p}_j\big)\big(p_j(t + \tau) - \bar{p}_j\big),
\end{equation}
where \(\bar{p}_j = \frac{1}{T}\sum_t p_j(t)\) is the time-averaged
occupation.  Define the decorrelation time \(\tau_j^*\) as the smallest
\(\tau\) for which \(R_j(\tau) < R_j(0)/e\).  Plot \(\tau_j^*\) against
\(|\lambda_j|\).  The hypothesis predicts a negative correlation: dimensions
with large \(|\lambda_j|\) (fast oscillators) should track rapidly changing
features and therefore have short decorrelation times, while dimensions with
small \(|\lambda_j|\) (slow oscillators) should track gradually varying
features and have long decorrelation times.  A null result (no correlation
between \(\tau_j^*\) and \(|\lambda_j|\)) would indicate that the
optimization does not exploit the timescale structure afforded by the free
Hamiltonian, and would motivate investigating alternative initialization or
regularization strategies for the frequency spectrum.

\paragraph{Prediction 5: Phase interference contributes to output quality.}
Section~\ref{sec:gap_source} decomposed the Born-rule probability into \(N\)
diagonal terms that depend only on magnitudes \(|c_j|^2\) and
\(\binom{N}{2}\) cross terms that depend on relative phases \(\theta_{j'} -
\theta_j\).  If the trained model encodes useful information in the
relative phases of its state (as the separation theorem requires for the
disambiguation tasks), then discarding the cross terms should degrade
prediction quality.  Define the diagonal-only readout
\begin{equation}
\label{eq:diag_readout}
\tilde{p}(k \mid \psi) = \frac{\sum_{j=0}^{N-1} |[m_k]_j|^2\,
|c_j|^2}{\sum_{k'}\sum_{j=0}^{N-1} |[m_{k'}]_j|^2\, |c_j|^2},
\end{equation}
which retains only the magnitude-dependent terms and renormalizes.  The
prediction is that the full Born rule~\eqref{eq:born_rule} achieves lower
perplexity than the diagonal-only readout~\eqref{eq:diag_readout} on
held-out data, and that the gap quantifies the contribution of phase
interference to the model's predictions.

The protocol uses a single trained CUSM evaluated under two readout rules.
At test time, compute the per-token log-probability under both the Born rule
and the diagonal-only readout, using the same state trajectory (no
retraining).  Report the perplexity difference.  Because the diagonal-only
readout accesses \(N\) features (the squared magnitudes) while the full
Born rule accesses \(N^2\) features (including all cross terms), the gap
should be substantial when the model has learned to encode pairwise
relational information in phase.

\subsection{Limitations of the Present Analysis}
\label{sec:limitations}

The theoretical results established in Sections~\ref{sec:model}--\ref{sec:separation}
are subject to several scope restrictions that we state explicitly, including
some that affect the interpretation of the central results.

\paragraph{The separation theorem applies to a simplified model class.}
The most significant structural limitation of the present work is that the
CUSM used in the separation theorem (Definition~\ref{def:cusm}) assigns a
fixed, state-independent unitary to each token.  The full architecture of
Section~\ref{sec:model} uses state-dependent Hamiltonians
(Section~\ref{sec:hamiltonian}), in which the network \(g_\theta\) receives
the current state as input and produces a Hamiltonian that depends on the
full history of previous tokens.  The full model therefore belongs to a
strictly richer expressivity class than the CUSM.  This means the
separation theorem, which was motivated throughout
Sections~\ref{sec:introduction}--\ref{sec:model} as establishing an
advantage for the full architecture, actually establishes an advantage only
for the simplified state-independent special case.  Any task solvable by
the CUSM is also solvable by the full model, so the \(\Omega(N^2)\) lower
bound on real orthogonal models carries over, but the theorem does not
characterize the additional expressivity gained by state-dependence.
Proving a separation for the full state-dependent architecture is an open
problem that requires different proof techniques, since the state-dependence
breaks the algebraic structure exploited in Lemmas~\ref{lem:quadratic_lifting}
and~\ref{lem:log_rank}.

\paragraph{The separation is between output mechanisms, not between complex
and real dynamics.}
Theorem~\ref{thm:separation} compares the Born-rule readout with the
affine-softmax readout.  The lower bound in Lemma~\ref{lem:softmax_rank}
holds for the specific form of the softmax readout defined in
Definition~\ref{def:rosm}: a linear projection of the state followed by
softmax normalization~\cite{yang2018breaking}.  A real-valued model equipped
with a quadratic readout of the form \(p(k) \propto (w_k^\top h)^2\), or
with a mixture-of-softmax readout~\cite{yang2018breaking}, would itself
perform a lifting from \(\mathbb{R}^d\) to a higher-dimensional feature
space and could substantially narrow the gap.  Such readouts are not
standard in the architectures surveyed in Section~\ref{sec:related}, but
they are not precluded by any fundamental constraint.  The separation should
therefore be understood as a statement about the advantage of Born-rule
decoding over affine-softmax decoding, not as a claim that complex-valued
dynamics are inherently superior to real-valued dynamics for all output
mechanisms.  This qualification applies wherever the paper describes the
separation as being between ``complex'' and ``real'' models without further
specification.

\paragraph{The lower bound is conditional on a full-rank assumption.}
The lower bound in part~(2) of Theorem~\ref{thm:separation} is conditioned
on the assumption that the target log-probability matrix $L^*$ has full row
rank $N^2$.  This condition is stated as an explicit hypothesis in the
theorem rather than derived from the task-construction parameters, and a
complete analytic characterization of when it holds is deferred to future
work.  We do not claim that the full rank of the probability matrix $P^*$
(established by Lemma~\ref{lem:log_rank}) implies the full rank of
$L^* = \log P^*$: entrywise logarithms do not preserve matrix rank in
general, and this step is not valid as a general mathematical statement
(rank can increase or decrease under entrywise logarithm even in $2 \times
2$ examples).  The Remark following Theorem~\ref{thm:separation} identifies
a log-odds reformulation that would yield a cleaner, potentially
unconditional lower bound; developing this into a complete proof is left for
future work.  Additionally, Lemma~\ref{lem:general_position} provides a
complete explicit construction only for $N = 2$; for $N > 2$, the proof
invokes a spanning property of the Veronese-type variety that is stated but not
fully executed.

\paragraph{The interaction picture does not reduce computation in the
discrete-step regime.}
Section~\ref{sec:interaction_picture} motivates the interaction picture by
arguing that it allows the integrator to ``resolve only the timescales
introduced by the input-driven interaction, not the potentially much faster
free oscillations.''  This argument is valid for adaptive continuous-time
integrators that adjust their step size to resolve fast oscillations.  In
the discrete-step regime used by the model, the Cayley transform is applied
once per token with a fixed step size \(\Delta t\), regardless of the
spectrum of \(H_0\).  The interaction-picture Hamiltonian has off-diagonal
entries \([H_{\mathrm{int},I}(t)]_{jk} =
[H_{\mathrm{int}}(t)]_{jk}\, e^{i(\lambda_j - \lambda_k)t}\), which
oscillate rapidly when \(|\lambda_j - \lambda_k|\) is large.  The Cayley
transform does not adaptively resolve these oscillations; it applies a
single midpoint step regardless.  The claimed computational advantage of
the interaction picture, in the sense of reduced integrator burden,
therefore does not materialize in the discrete-step regime.  The interaction
picture remains a useful change of variables that separates the free and
interaction dynamics conceptually, and the frequency-selective coupling
argument of Section~\ref{sec:interaction_picture} applies in an averaged
sense over many steps, but the computational motivation as stated does not
hold for the model as implemented.  This does not affect the correctness of
the architecture, only the justification for one of its components.

\paragraph{The quantum cognition motivation is a hypothesis, not an
inference.}
Sections~\ref{sec:introduction} and~\ref{sec:quantum_cognition} cite
Busemeyer and Bruza~\cite{busemeyer2012quantum} and Pothos and
Busemeyer~\cite{pothos2013quantum} to motivate the use of complex amplitudes
and interference in a language model.  The cited works document violations
of classical probability in human judgment experiments, and argue that these
violations are naturally explained by quantum probability.  However, the
inference that a language model would benefit from the same mathematical
structure does not follow directly.  Quantum cognition models describe
experimental data from specific psychological paradigms; they are not
general models of linguistic computation.  Whether the statistical structure
of text corpora exhibits the kind of interference effects documented in
those experiments is an empirical question that requires corpus-level
evidence and has not been established.  The connection to quantum cognition
should be read as a motivating analogy and a hypothesis to be empirically
tested, not as a logical justification for the architecture.  The
architecture's formal properties---norm preservation, quadratic output,
conserved probability currents---stand on their own algebraic foundations
and do not require the quantum cognition literature for their justification.

\paragraph{No optimization theory.}
The analysis is entirely representational: it characterizes the set of
functions that each model class can represent, not whether gradient-based
optimization can find the parameters that realize a given function.  This
is a significant practical limitation.  The training of the proposed
architecture involves several optimization challenges that are not addressed
here: constrained optimization of the measurement matrix \(M\) on the
Stiefel manifold~\cite{absil2008optimization}; differentiation through
complex-valued Cayley solves at each step; state-dependent Hamiltonian
generation through a neural network \(g_\theta\) that receives the current
state as input, creating a recurrent dependency in the gradient computation;
and complex-valued parameter gradients throughout.  The combination of these
factors may produce optimization landscapes with unfavorable local minima,
poor conditioning, or gradient interference between parameter groups that
negate the representational advantage identified by the separation theorem.
The gradient norm preservation guaranteed by the unitary Cayley steps
applies only to the state gradient, as noted above, not to the parameter
gradients of \(g_\theta\).  The absence of any optimization analysis, even
heuristic, is a gap that significantly limits the practical relevance of
the paper's theoretical contributions.  Addressing this gap, whether
through theoretical analysis of the training dynamics or through empirical
characterization of the loss landscape on the synthetic tasks of
Section~\ref{sec:predictions}, is a necessary step before the architecture
can be recommended for practical use.

\paragraph{The comparison with real-valued models does not engage with the
full literature on complex networks.}
Section~\ref{sec:related} cites Trabelsi et al.~\cite{trabelsi2018deep} and
Hirose and Yoshida~\cite{hirose2012generalization} as primary references for
the advantages of complex-valued networks.  The empirical literature is not
uniformly positive on this point.  Several studies have found that
real-valued models with increased width, or with gating mechanisms such as
LSTM \cite{hochreiter1997long} gates or Mamba's~\cite{gu2023mamba} selective
state-space mechanism, match or outperform complex-valued models of the same
parameter count on standard benchmarks.  The representation-theoretic
advantage identified by Theorem~\ref{thm:separation} applies to a specific
comparison between Born-rule and affine-softmax readouts; it does not imply
that complex-valued models outperform real-valued models with richer
nonlinear readouts or gating mechanisms in all settings.  A fair empirical
evaluation of the proposed architecture would need to control for parameter
count, training cost, and the choice of competing architectures, and would
need to acknowledge the mixed existing evidence on complex-valued networks.

\section{Conclusion}
\label{sec:conclusion}

This paper introduced a sequence modeling framework in which the latent
state is a unit-norm vector in \(\mathbb{C}^N\) evolving under a learned,
time-dependent Hermitian Hamiltonian, and in which output probabilities are
computed via the Born rule.  The framework rests on a single structural
constraint: the Hamiltonian \(H(t) = H_0 + H_{\mathrm{int}}(t)\) is
Hermitian by construction (Section~\ref{sec:hamiltonian}), which guarantees
that the evolution is unitary, which guarantees that the state norm is
exactly preserved, which guarantees that the Born rule produces a valid
probability distribution.  Each subsequent component of the architecture
follows from this chain: the interaction picture
(Section~\ref{sec:interaction_picture}) factors out the analytically known
free oscillations to isolate the input-driven dynamics; the Cayley
(Crank--Nicolson) discretization (Section~\ref{sec:cayley}) transfers the
continuous-time unitarity guarantee to discrete hardware without
approximation~\cite{hairer2006geometric,crank1947practical}; and the
Born-rule decoding (Section~\ref{sec:born_rule}) converts the preserved
complex state into token probabilities through a quadratic function that is
sensitive to all \(\binom{N}{2}\) pairwise phase relationships among latent
dimensions.  It should be noted that norm preservation and integration
accuracy are distinct properties: the Cayley transform guarantees that the
state norm is exactly \(1\) at every step regardless of step size, but it
does not eliminate phase errors that accumulate over long sequences.  The
model can process arbitrarily long sequences without norm drift, but the
accuracy of the state trajectory as an approximation to the continuous-time
solution degrades with sequence length at a rate determined by the step size
and the spectrum of the interaction Hamiltonian \(H_{\mathrm{int}}\).

The paper established three results from this construction.  First, the
Hermitian constraint produces a continuity equation
(Section~\ref{sec:currents}), mirroring the standard quantum-mechanical
treatment of probability current~\cite{sakurai1994modern}, whose flux terms
are pairwise probability currents \(J_{j \leftarrow k}(t) =
2\operatorname{Im}(H_{jk}\, c_j^* c_k)\) between latent dimensions.  These
currents are antisymmetric, vanish on the diagonal, and sum to zero at each
dimension (Proposition~\ref{prop:current_properties}), providing an exact,
conservation-respecting decomposition of each state update into directed
pairwise transfers.  The currents are driven entirely by the interaction
Hamiltonian \(H_{\mathrm{int}}(t)\)
(Section~\ref{sec:current_decomposition}), which links each token's
contribution to a measurable redistribution of amplitude across the model's
internal state.  For diagnostic purposes, the midpoint current
formula~\eqref{eq:midpoint_current}, evaluated at the average of the
pre-step and post-step states, provides an exact pairwise decomposition of
the discrete probability change without the \(O(\Delta t^2)\) approximation
error of the continuous-time formula.

Second, the Born rule's quadratic dependence on the complex amplitudes
enables a representational advantage over the affine-softmax readout used
by standard real-valued models.  Theorem~\ref{thm:separation} constructs a
family of disambiguation tasks \(\mathcal{D}_N\) for which a complex
unitary sequence model of dimension \(N\) computes the correct output
distribution exactly, while any real orthogonal sequence model with
affine-softmax readout~\cite{yang2018breaking} requires dimension at least
\(N^2 - 2\).
This separation is between two specific output mechanisms (Born-rule
decoding and affine-softmax decoding) applied to norm-preserving dynamics
of the respective types.  It should not be read as a general claim that
complex-valued dynamics are superior to real-valued dynamics for all
architectures and readout mechanisms; a real-valued model with a quadratic
or mixture-of-softmax readout could substantially narrow the gap.  The
proof traces the separation to a specific algebraic mechanism: the Born
rule implicitly lifts the \(N\)-dimensional complex state to the
\(N^2\)-dimensional space of Hermitian matrices via the map
\(\ket{\psi} \mapsto \ket{\psi}\bra{\psi}\)
(Lemma~\ref{lem:quadratic_lifting}; this is the same Veronese-type lifting
exploited by Born machines in generative
modeling~\cite{cheng2018information}), accessing \(N^2\) linearly
independent features through the cross terms \(c_j c_{j'}^*\) that encode
relative phases.  The affine-softmax readout, which depends linearly on the
state, cannot synthesize these pairwise features from a state of dimension
less than \(N^2 - 2\) (Lemma~\ref{lem:softmax_rank}).  The separation
theorem is established for the simplified state-independent CUSM; extending
it to the full state-dependent architecture of Section~\ref{sec:model} is
an open problem.
Additionally, Theorem~\ref{thm:separation} is stated conditionally on the
assumption that the target log-probability matrix \(L^*\) has full row rank
\(N^2\); this assumption is not derived from the task construction, and the
Remark following the theorem identifies a log-odds reformulation as a path
toward unconditionalizing the result.  Lemma~\ref{lem:general_position}
provides a full explicit construction only for \(N = 2\); the general-\(N\)
argument invokes a spanning property of the Veronese-type variety that is stated
but not fully executed.

Third, the per-step computational cost of the architecture is
\(O(C_{g_\theta} + Nr^2 + NV)\) (Section~\ref{sec:complexity}), where the
Born-rule output cost \(O(NV)\) matches the output-projection cost of
standard architectures with the same vocabulary size, and the Cayley solve
cost \(O(Nr^2)\) is controlled by the rank parameter \(r\) of the
interaction Hamiltonian.  The numerical stability of the Woodbury solve at
each step depends on the conditioning of an \(r \times r\) matrix whose
eigenvalues are determined by the current Hamiltonian; this conditioning is
not guaranteed by the architecture and requires monitoring or regularization
in practice~\cite{golub2013matrix}.

These three results are interdependent in a way that reflects the coherence
of the framework rather than being independent contributions.  The
probability currents exist because the dynamics are unitary, which is the
same property that makes the Born rule produce valid probabilities.  The
separation theorem quantifies the advantage of Born-rule decoding, which
depends on the phase information that the unitary dynamics preserve.  The
computational cost of the Cayley discretization is what makes the
unitarity guarantee practical, and the low-rank structure of the
interaction Hamiltonian that controls this cost is the same structure that
decomposes the probability currents into interpretable coupling channels
(equation~\eqref{eq:current_channels}).

The framework is theoretical in its present form: no experiments on
language data have been conducted, and the separation theorem applies to a
constructed task family rather than to natural language directly.  The claim
that language modeling may benefit from interference mechanisms is motivated
by an analogy with the quantum cognition
literature~\cite{busemeyer2012quantum,pothos2013quantum} and by emerging
empirical evidence of Born-rule effectiveness in NLP
tasks~\cite{guidotti2022text}, which document interference effects in
human judgment and quantum-inspired text classification respectively; these
analogies are hypotheses to be empirically tested, not logical derivations
from those results.  Whether the interference mechanism that drives the
separation is relevant to the structure of natural language is the central
empirical question that this work leaves open.  The experimental protocols
specified in Section~\ref{sec:predictions} are designed to resolve this
question: they test whether the quadratic dimensional scaling predicted by
Theorem~\ref{thm:separation} manifests in gradient-trained models
(Prediction~1), whether Born-rule decoding outperforms softmax decoding on
a shared complex state (Prediction~2), whether probability currents
concentrate at semantically disambiguating tokens (Prediction~3), whether
the learned frequency spectrum correlates with linguistic timescales
(Prediction~4), and whether phase cross terms contribute measurably to
output quality (Prediction~5).  A positive outcome on Predictions~1 and~2
would confirm the representational advantage identified by the separation
theorem in the context of trained models.  A positive outcome on
Predictions~3 through~5 would provide evidence that the interference
mechanism is not merely a theoretical possibility but is exploited by the
trained model on natural text.

The paper's contribution is a formal argument that the mathematical
structure characterized by complex-valued unit-norm states, Hermitian
generators, and squared-amplitude measurement provides specific and
quantifiable algebraic properties that distinguish it from standard
real-valued sequence models with linear readouts: norm-preservation by
construction, a quadratic output mechanism that accesses \(O(N^2)\)
pairwise features from an \(N\)-dimensional state, and a conserved
continuity equation that decomposes each state update into antisymmetric
pairwise probability flows.  Whether these algebraic properties translate
to practical gains on language modeling benchmarks is a question that the
theory motivates but cannot answer alone, and that the specified
experimental program is designed to address.

\newpage

\bibliographystyle{unsrt}  

\bibliography{references}

\end{document}